\documentclass[11pt,letterpaper]{article}

\usepackage{booktabs}

\usepackage{latexsym}
\usepackage{amssymb,amsmath, bm,pgfplots,tikz,bbm}
\usepackage{graphicx}
\usepackage{marvosym}
\usepackage{multirow,float}
\usepackage{caption}
\usepackage{mathtools}

\usepackage{subcaption}
\usepackage{comment}

\usepackage{natbib}
\usepackage{color}
\usepackage[bookmarksopen=true, bookmarksnumbered=true,
pdfstartview=FitH, breaklinks=true, urlbordercolor={0 1 0}, citebordercolor={0 0 1}]{hyperref}

\usepackage{dcolumn}
\newcolumntype{.}{D{.}{.}{-1}}
\newcolumntype{d}[1]{D{.}{.}{#1}}

\usepackage{theorem}
\theoremstyle{plain}
\theoremheaderfont{\scshape}
\newtheorem{theorem}{Theorem}

\newtheorem{assumption}{Assumption}
\newtheorem{corollary}{Corollary}
\newtheorem{lemma}{Lemma}

\newcommand{\qed}{\hfill \ensuremath{\Box}}

\providecommand{\norm}[1]{\lVert#1\rVert}
\newenvironment{proof}{\vspace{1ex}\noindent{\bf Proof}\hspace{0.5em}}
{\hfill\qed\vspace{1ex}}
\usetikzlibrary{decorations.markings}
\usetikzlibrary{decorations.pathmorphing}
\usetikzlibrary{shapes.geometric, arrows}
\usetikzlibrary{arrows,decorations.pathmorphing,backgrounds,positioning,fit,matrix}
\usetikzlibrary{shapes,decorations,arrows,calc,arrows.meta,fit,positioning}
\tikzset{
	-Latex,auto,node distance =1 cm and 1 cm,semithick,
	state/.style ={circle, draw, minimum width = 0.7 cm},
	point/.style = {circle, draw, inner sep=0.04cm,fill,node contents={}},
	bidirected/.style={Latex-Latex,dashed},
	el/.style = {inner sep=2pt, align=left, sloped}
}
\usetikzlibrary{positioning}
\usetikzlibrary{fadings}
\usetikzlibrary{intersections}
\usepackage{kantlipsum}
\allowdisplaybreaks

\usepackage{rotating}

\usepackage{arydshln}
\usepackage{threeparttable}

\usepackage[compact]{titlesec}

\newcommand{\blind}{0}

\usepackage[ruled,linesnumbered,vlined]{algorithm2e}
\usepackage{pifont}

\newtheorem{definition}{Definition}

\newcommand\E{\mathbb{E}}
\newcommand\V{\mathbb{V}}

\usetikzlibrary{decorations.markings}
\usetikzlibrary{decorations.pathmorphing}
\usetikzlibrary{shapes.geometric, arrows}
\usetikzlibrary{arrows,decorations.pathmorphing,backgrounds,positioning,fit,matrix}
\usetikzlibrary{shapes,decorations,arrows,calc,arrows.meta,fit,positioning}
\tikzset{auto,node distance =1 cm and 1 cm,semithick,
	state/.style ={circle, draw, minimum width = 0.7 cm},
	point/.style = {circle, draw, inner sep=0.04cm,fill,node contents={}},
	bidirected/.style={Latex-Latex,dashed},
	el/.style = {inner sep=2pt, align=left, sloped}
}

\begin{document}

% === new commands ===
\newcommand\ud{\mathrm{d}}
\newcommand\dist{\buildrel\rm d\over\sim}
\newcommand\ind{\stackrel{\rm indep.}{\sim}}
\newcommand\iid{\stackrel{\rm i.i.d.}{\sim}}
\newcommand\logit{{\rm logit}}
\renewcommand\r{\right}
\renewcommand\l{\left}
\newcommand\cO{\mathcal{O}}
\newcommand\cY{\mathcal{Y}}
\newcommand\cL{\mathcal{L}}
\newcommand\cA{\mathcal{A}}
\newcommand\cB{\mathcal{B}}
\newcommand\cD{\mathcal{D}}
\newcommand\cE{\mathcal{E}}
\newcommand\cH{\mathcal{H}}
\newcommand\cU{\mathcal{U}}
\newcommand\cN{\mathcal{N}}
\newcommand\cT{\mathcal{T}}
\newcommand\cX{\mathcal{X}}
\newcommand\bA{\bm{A}}
\newcommand\bH{\bm{H}}
\newcommand\bB{\bm{B}}
\newcommand\bP{\bm{P}}
\newcommand\bQ{\bm{Q}}
\newcommand\bU{\bm{U}}
\newcommand\bD{\bm{D}}
\newcommand\bS{\bm{S}}
\newcommand\bx{\bm{x}}
\newcommand\bX{\bm{X}}
\newcommand\bTheta{\bm{\Theta}}
\newcommand\btheta{\bm{\theta}}
\newcommand\bV{\bm{V}}
\newcommand\bW{\bm{W}}
\newcommand\bM{\bm{M}}
\newcommand\bZ{\bm{Z}}
\newcommand\bY{\bm{Y}}
\newcommand\bt{\bm{t}}
\newcommand\bbeta{\bm{\beta}}
\newcommand\bpi{\bm{\pi}}
\newcommand\bdelta{\bm{\delta}}
\newcommand\bgamma{\bm{\gamma}}
\newcommand\balpha{\bm{\alpha}}
\newcommand\bmu{\bm{\mu}}
\newcommand\bone{\mathbf{1}}
\newcommand\bzero{\mathbf{0}}
\newcommand\tomega{\tilde\omega}
\newcommand{\argmax}{\operatornamewithlimits{argmax}} 

\DeclarePairedDelimiter\floor{\lfloor}{\rfloor}

\newcommand{\R}{\textsf{R}}
\newcommand{\Python}{\textsf{Python}}
\newcommand\spacingset[1]{\renewcommand{\baselinestretch}%
  {#1}\small\normalsize}

\spacingset{1}

\newcommand{\tit}{\bf Cramming Contextual Bandits for On-policy
  Statistical Evaluation} %with Generic Machine Learning Algorithms}

%%%%%%%%%%%%%%%%%%%%%%%%%%%%%%%%%%%%%%%%%%%%%%%%%%%%%%%%%%%%%%%%%%%%%%%%%%%%%%%%

\if0\blind

{\title{\tit\thanks{
      The proposed methodology will be implemented
      through an open-source \Python{} and \R{} packages, {\sf CRAM}
      % {\bf KI: Add a proper citation here once it's up}.
    }
  }

  \author{Zeyang Jia\thanks{PhD Student, Department of Statistics,
      Harvard University, Cambridge, MA, 02138. Email:
      \href{mailto:zeyangjia@g.harvard.edu}{zeyangjia@g.harvard.edu}}\hspace{.75in}
    Kosuke Imai\thanks{Professor, Department of Government and
      Department of Statistics, Harvard University, Cambridge, MA
      02138. Phone: 617--384--6778, Email:
      \href{mailto:Imai@Harvard.Edu}{Imai@Harvard.Edu}, URL:
      \href{https://imai.fas.harvard.edu}{https://imai.fas.harvard.edu}}
    \hspace{.75in} Michael Lingzhi Li\thanks{Assistant Professor,
      Technology and Operations Management, Harvard Business School,
      Boston, MA 02163. Email:
      \href{mailto:mili@hbs.edu}{mili@hbs.edu}, URL:
      \href{https://www.michaellz.com}{https://www.michaellz.com}} } }
 
\fi 

\if1\blind
\title{\tit}
\fi

\date{%First version: \today \\ This version:
  \today}

\maketitle

\pdfbookmark[1]{Title Page}{Title Page}

\thispagestyle{empty}
\setcounter{page}{0}

\begin{abstract}
  We introduce the `cram' method as a general statistical framework
  for evaluating the final learned policy from a multi-armed
  contextual bandit algorithm, using the dataset generated by the same
  bandit algorithm.  The proposed {\it on-policy} evaluation
  methodology differs from most existing methods that focus on
  off-policy performance evaluation of contextual bandit algorithms.
  Cramming utilizes an entire bandit sequence through a single pass of
  data, leading to both statistically and computationally efficient
  evaluation.  We prove that if a bandit algorithm satisfies a certain
  stability condition, the resulting crammed evaluation estimator is
  consistent and asymptotically normal under mild regularity
  conditions. Furthermore, we show that this stability condition holds
  for commonly used linear contextual bandit algorithms, including
  $\epsilon$-greedy, Thompson Sampling, and Upper Confidence Bound
  algorithms. Using both synthetic and publicly available datasets, we
  compare the empirical performance of cramming with the
  state-of-the-art methods. The results demonstrate that the proposed
  cram method reduces the evaluation standard error by approximately
  40\% relative to off-policy evaluation methods while preserving
  unbiasedness and valid confidence interval coverage. % 146 words
  
\bigskip
\noindent {\bf Key Words:} adaptive experimentation, policy learning,
policy evaluation, sample splitting

\end{abstract}
\clearpage
\spacingset{1.5}

\section{Introduction}
\label{sec:introduction}

Recent years have witnessed an increasing use of bandit algorithms for
various sequential decision-making problems, including online
advertising \citep[e.g.,][]{li2010contextual}, mobile health
\citep[e.g.,][]{yom2017encouraging, nahum2018just}, and online
education \citep[e.g.,][]{rafferty2019statistical}. The deployment of
these adaptive algorithms is attractive because they are likely to
discover an optimal policy more quickly.  As a result, there now
exists a large literature that explores the conditions under which
different bandit algorithms achieve optimality \citep[e.g.,][]{rusmevichientong2010linearly,abbasi2011improved,bubeck2012regret,li2017provably}. 

Beyond these theoretical guarantees, however, it is essential for
practitioners to be able to efficiently \emph{evaluate} the real-world
empirical performance of bandit algorithms that are being deployed.
In particular, we would like to use the same data to both train a
bandit algorithm and evaluate its performance without the need to
collect a separate data set for evaluation.  In this paper, we
introduce the `cram' method as a general methodological framework for
such {\it on-policy} statistical evaluation of multi-armed contextual
bandit algorithms.  Specifically, the proposed method leverages data
adaptively collected by a bandit algorithm to empirically evaluate the
performance of its final learned policy.

Many scholars have developed {\it off-policy} evaluation methods of
bandit algorithms with adaptively collected data, which requires a
separate dataset for evaluation
\citep[e.g.][]{luedtke2016statistical,zhang2020inference,hadad2021confidence,bibaut2021post,bibaut2021risk,zhan2021off,zhang2021statistical,zhan2023policy}.
The key difference between cram and these existing methods is that the
former evaluates the performance of the final learned policy, which is
a data-dependent parameter, whereas the latter focuses on fixed
population parameters such as coefficients in regression models, the
average treatment effect, and the value of a fixed
policy. \citep[e.g.][]{luedtke2016statistical,zhang2020inference,hadad2021confidence,bibaut2021post,bibaut2021risk,zhan2021off,zhang2021statistical,zhan2023policy}.
Since a learned policy may be changing at every iteration, this
presents a technical challenge in establishing the asymptotic
properties of the cram method.

The name of the proposed methodology is inspired by an intensive
learning approach often used at cram schools, where students iterate a
process of learning new materials and taking practice tests to prepare
for the final exam that covers all materials.  Similarly, the cram
method updates a learned policy by applying a contextual bandit
algorithm to a new batch of data, and then evaluates the performance
of \emph{all} previously learned policies using this new batch of
data. Repeating this process through a single pass of data leads to
the evaluation of the final learned policy, thereby yielding both
computationally and statistically efficient policy evaluation.

\begin{figure}[t!] \spacingset{1}
  % \begin{subfigure}{\textwidth}
  \includegraphics[width=\textwidth,clip]{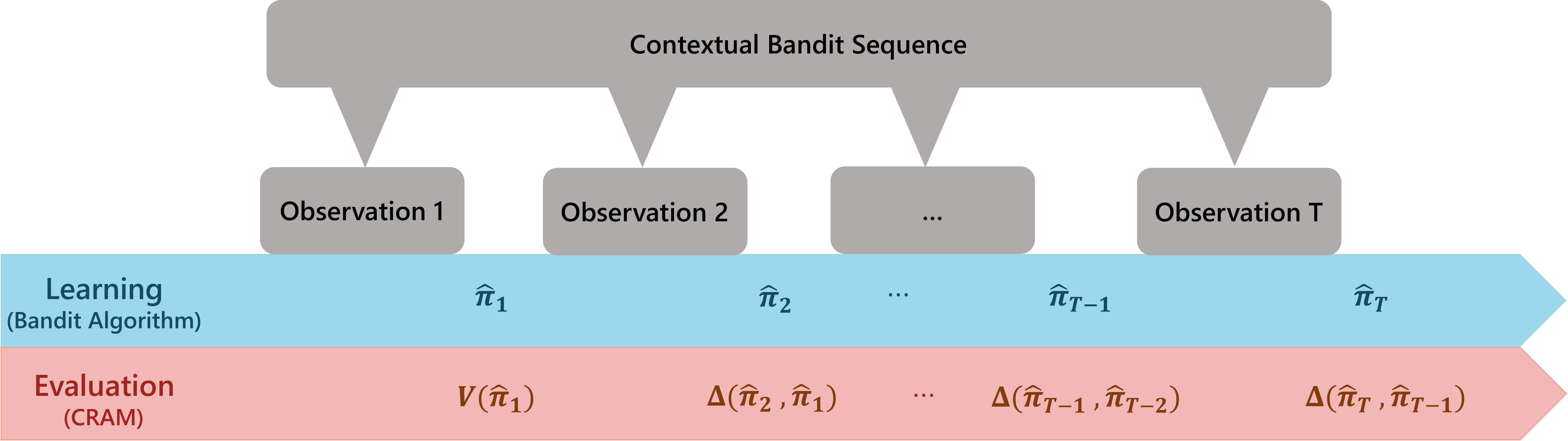}
  \caption{A schematic illustration of the cram method. A contextual
    bandit algorithm uses the first batch of data to obtain the first
    learned policy $\hat\pi_1$, and we estimate the value of this
    policy using the remaining $T-1$ observations, denoted by
    $V(\hat\pi_1)$. Next, the bandit algorithm updates this learned
    policy with the second batch, yielding the updated learned policy
    $\hat\pi_2$.  We then estimate the value difference between these
    two policies,
    $\Delta(\hat\pi_2, \hat\pi_1)=V(\hat\pi_2)-V(\hat\pi_1)$, using
    the remaining $T-2$ batches of data. Repeating this
    update-and-test process leads to the final learned policy
    $\hat\pi_{T-1}$, and the evaluation of final performance
    improvement $\Delta(\hat\pi_{T-1}, \hat\pi_{T-2})$ based on the
    last $T$th batch.  Finally, summing these performance difference
    estimates yields the estimated value of the final learned
    policy.}
  \label{fig:cram}
\end{figure}

Figure~\ref{fig:cram} presents a schematic illustration of the cram
method, which we formally introduce in
Section~\ref{sec:cram_method}. Suppose that we have a total of $T$
batches of data adaptively collected by a contextual bandit algorithm.
We require that the entire sequence of learned policies is stored for
evaluation, i.e., $\{\hat\pi_1, \hat\pi_2,\ldots,\hat\pi_T\}$.  Our
goal is to use the same data and evaluate the empirical performance of
the final learned policy produced by this bandit algorithm using the
$T$ batches of bandit data.  We begin by estimating the value of the
first learned policy $V(\hat\pi_1)$ using all the data other than the
first batch.  In the second iteration, we use the last $T-2$ batches
to estimate the value improvement due to this update
$\Delta(\hat\pi_2,\hat\pi_1)=V(\hat\pi_2)-V(\hat\pi_1)$. Repeating
this procedure additional $T-2$ times yields the value difference
$\Delta(\hat\pi_{T}, \hat\pi_{T-1})$. Summing these value differences
and the value of the first learned policy yield the value of the final
learned policy, i.e.,
$V(\hat\pi_T)=V(\hat\pi_1)+\sum_{t=2}^{T-1}
\Delta(\hat\pi_t,\hat\pi_{t-1})$.

In essence, for estimating every policy value difference
$\Delta(\hat\pi_t,\hat\pi_{t-1})$ at the $t$th iteration, cram
utilizes all future batches of data that are not yet used to train the
learned policy $\hat\pi_t$.  In this way, we turn an on-policy
evaluation problem into a telescoping sum over a sequence of
off-policy evaluation problems. Although there is no bandit data left
for estimating the final value difference
$\Delta(\hat\pi_T, \hat\pi_{T-1})$, we show in
Section~\ref{sec:estimation_inference} that under a certain stability
condition, this term is negligible, leading to the consistency and
asymptotic normality of the crammed evaluation estimator
$\widehat{V}(\hat\pi_T)$.  We further demonstrate that this stability
condition holds for commonly used contextual bandit algorithms such as
$\epsilon$-greedy, Thompson Sampling, and Upper Confidence Bound (UCB)
algorithms.

In Section~\ref{sec:synthetic}, we conduct extensive numerical
experiments based on both synthetic and real-world data and compare
the performance of cram with that of sample splitting (a training set
is used to learn a policy and a test set is used to estimate its
value).  We find that the cram method is more statistically efficient
than sample splitting across a wide range of settings with various
sample sizes, signal strengths, and bandit algorithms.  The policy
evaluation based on cram is also unbiased and yields confidence
intervals whose empirical coverage rates are close to their nominal
rates.

\subsection*{Related literature}

Many scholars have proposed the methods for {\it off-policy}
statistical evaluation with adaptively collected historical data
\citep[see][for a comprehensive review]{bibaut2025demystifying}. It
has been shown that when the data are adaptively collected, the direct
estimation methods such as linear regression for estimating the value
of policy are biased and asymptotically non-normal
\citep[e.g.,][]{nie2018adaptively,shin2019sample,wang2023subgroup}. To
address this problem, many use adaptive weighted Inverse Probability
Weighting (IPW) estimators for off-policy evaluation
\citep[e.g.,][]{luedtke2016statistical,hadad2021confidence,bibaut2021post,bibaut2021risk,zhang2020inference,zhang2021statistical,zhan2021off,zhan2023policy,shen2024doubly}.  

Building on this body of work, we develop a methodology for {\it
  on-policy} statistical evaluation methodology that simultaneously
trains and evaluates a bandit algorithm.  Unlike off-policy
evaluation, the goal of on-policy statistical evaluation is to
estimate the value of the learned policy at any point in time, which is a
data-dependent parameter.  As a result, on-policy evaluation directly
use the same data adaptively collected by a bandit algorithm to
evaluate its performance.  In contrast, the aforementioned off-policy
evaluation literature focuses on the estimation of fixed population
parameters. For example, \cite{zhang2020inference} and
\cite{hadad2021confidence} use the adaptive weighted IPW estimators to
estimate the average treatment effect.  \cite{zhang2021statistical}
considers statistical inference for M-estimators of parametric models
with adaptively collected data. \cite{bibaut2021post} and
\cite{zhan2021off} evaluate a pre-specified known policy using bandit
data.  All of these methods are designed for off-policy evaluation
whose estimands are fixed population quantities.

Importantly, our methodology estimates the value of the learned
policy at every point in time without assuming that different bandit runs converge to an
optimal policy or even an identical policy.  This contrasts with
existing methods that estimate the value of (unknown but fixed)
optimal policy using either i.i.d. data \citep{luedtke2016statistical}
or bandit data \citep{shen2024doubly}.  In particular,
\cite{shen2024doubly} estimates the value of optimal policy using
linear contextual bandit data while assuming the model is correctly
specified.  The cram method does not assume correct model
specification.  

Another related literature is a collection of recent papers that
develop anytime valid inference using adaptively collected data
\citep[e.g.,
][]{kaufmann2021mixture,howard2021time,bibaut2021sequential,bibaut2022near,waudby2022anytime,woong2023design,cook2024semiparametric}.
These methods use finite-sample probability bounds to establish
confidence sequences that are valid at any point during adaptive
experimentation.  While we leave the development of such anytime valid
inference with cram to future work, one important difference is that
similar to off-policy evaluation, most of existing methods focus on
the inference for a fixed population parameter rather than a
data-dependent parameter. One exception is \cite{waudby2022anytime}
which considers confidence sequence and intervals for the average
value of all the policies in the bandit sequence. In contrast, the
cram method estimates the value of the final learned policy in the
bandit sequence.

The cram method uses the same bandit data for both training and
evaluation to improve the efficiency of statistical inference.  This
idea is similar to the resampling methods including bootstrap
\citep[e.g.,][]{efron1992bootstrap,efron1997improvements} and
cross-validation
\citep[e.g.,][]{stone1974,blum1999,zhang1993,austern2020asymptotics,
  NEURIPS2020_bce9abf2,imai2023,bates2023}. However, these methods
typically require the exchangeability of data and are not directly
applicable to adaptively collected data.  In a recent article,
\cite{ghosh2024did} considers using bootstrap to generate many bandit
sequences for testing the performance of the RL algorithm.  However,
this method is computational expensive and its theoretical properties
are yet unknown.  In contrast, the cram method directly leverages the
sequential nature of bandit data and does not require an
exchangeability assumption.

\section{Problem formulation} 
\label{sec:problem_formulation}

Though our methodology and its theoretical properties readily
generalize to batched bandits, for notational simplicity, we focus on
$K$-arm contextual bandits with a total of $T$ time steps.  Let
$\bX_t\in\mathcal{X}$ denote the context variables at time
$t=1,2,\ldots,T$ where $\cX$ is the support of the context variables.
We use $A_t\in \cA=\{1,2,...,K\}$ to denote action taken at time $t$,
and $\{R_t(1), R_t(2),\ldots,R_t(K)\}$ to represent $K$ potential
rewards under different actions.  Note that we only observe one of
these potential rewards, and this realized reward is denoted by
$R_t = R_t(A_t)$.  Throughout this paper, we assume that the
underlying data generating process is stationary.
\begin{assumption}[Stationary Bandit] \spacingset{1}
  \label{ass:stationarity}
  The context and potential rewards
  $\{\bX_t,R_t(1),R_t(2),\ldots,R_t(K)\}$ are independently and
  identically distributed across time steps.
\end{assumption}

Let $\bH_t:= \{\bX_i,A_i,R_i\}_{i=1}^t$ be the realized history up to
time $t$.  We use $\hat\pi_t(\bx,a)$ to denote the learned policy
based on $\bH_{t}$, which represents the probability of selecting
action $a$ under context $\bx$ at the next time step, i.e.,
$\hat\pi_t(\bx, a) = \mathbb{P}(A_{t+1}=a\mid \bX_{t+1}=\bx,\bH_{t})$.
We use the ``hat'' notation to emphasize that the learned policy
$\hat\pi_t$ is dependent on the observed history $\bH_t$.  At each time
step $t$, we first observe a context $\bX_t$, followed by action $A_t$
that is selected according to the learned policy $\hat\pi_{t-1}$ based
on the previous history $\bH_{t-1}$. Finally, the reward under under
the selected action, i.e., $R_t=R_t(A_t)$ is observed, and the bandit
algorithm updates the learned policy to $\hat\pi_t$ based on the
current history $\bH_t$.

Under this stationarity assumption, we consider the {\it on-policy}
statistical evaluation of the learned policy $\hat\pi_{T}$ at time
$T$, using the entire data $\{(\bX_t,A_t,R_t)\}_{t=1}^T$ adaptively
collected by the bandit algorithm itself.  Although
Assumption~\ref{ass:stationarity} requires that the potential rewards
and contexts are i.i.d, the observed data
$\{(\bX_t,A_t,R_t)\}_{t=1}^T$ are not i.i.d. given that the arm
selection is done adaptively.  Our goal is to use this data for
estimating the value of the final learned policy denoted by
$V(\hat\pi_{T})$.  The general definition of the value of policy $\pi$
is given by,
\begin{align}
V(\pi) \ := \  \mathbb{E}_{R,\bX,A\sim \pi}\left[R(\pi(\bX,A))\right]
 \ = \ \sum_{a=1}^K \mathbb{E}_{\bX}\left[\pi(\bX,a)\mu_a(\bX)\right],\label{eq:value}
\end{align}
where $\mu_a(\bx) := \mathbb{E}\left[R_t(a)\mid \bX_t=\bx\right]$
represents the conditional mean of potential reward under action $a$
and context $\bx$.

\section{The cram method}
\label{sec:cram_method}

\begin{algorithm}[t!] \spacingset{1}
  \caption{The cram method for on-policy evaluation of a contextual bandit algorithm}\label{alg:0}
  \KwIn{Data $\{(\bX_t,A_t,R_t)\}_{t=1}^T$ collected by a contextual
    bandit algorithm; Sequence of learned policies
    $\{\hat\pi_1,\ldots,\hat\pi_T\}$ obtained via the bandit
    algorithm.}  \KwOut{Estimated value of the final learned policy,
    $\widehat{V}(\hat\pi_{T})$} Estimate the value of the initial
  learned policy $\hat\pi_1$ using the remaining bandit data
  $\{\bX_t, A_t, R_t\}_{t=2}^T$ and store the estimate as
  $\widehat{V}(\hat\pi_1)$\;
%  Set $\hat\pi_0 = \pi_0$\; 
  \For{$t=2$ \KwTo $T-1$}{ Estimate the policy value difference
    between $\hat\pi_t$ and $\hat\pi_{t-1}$, i.e.,
    $\Delta(\hat\pi_t, \hat\pi_{t-1})=V(\hat\pi_t)-V(\hat\pi_{t-1})$,
    using the data $\{\bX_i,A_i,R_i\}_{i=t+1}^T$ and store the
    resulting estimate as $\widehat{\Delta}(\hat\pi_t,\hat\pi_{t-1})$;
  } Estimate the value of the final learned policy
  $\hat\pi_{T}$ as:
  \begin{align*}
    \widehat{V}(\hat\pi_{T}) \approx \widehat{V}(\hat\pi_{T-1}) := \widehat{V}(\hat\pi_1)+\sum_{t=2}^{T-1}\widehat{\Delta}(\hat\pi_t,\hat\pi_{t-1}).
  \end{align*}
\end{algorithm}

%\begin{figure}[t]
%    \centering  \spacingset{1} \vspace{-0.25in} 
%    \includegraphics[width=\textwidth,clip,trim={0 2cm 0 0}]{Figures/CRAM22.pdf}
%    \caption{A schematic illustration of cramming for on-policy evaluation with bandit data as defined in Algorithm~\ref{alg:0}. At the $t$-th iteration, we use
%      the remaining last $T-t$ samples in the bandit sequence (red boxes) to estimate the policy
%      value difference between $\hat\pi_t$ and $\hat\pi_{t-1}$, i.e.,
%      $\widehat{\Delta}({\hat\pi}_t-\hat\pi_{t-1})$. After repeating this
%      train-and-test process $T-1$ times, the estimated policy value
%      differences from all iterations are summed together to obtain an
%      estimate of the value of $\hat\pi_{T-1}$, which is an approximation of the value of the final learned policy $\hat\pi_{T}$.} 
%    \label{fig:cram-alg} 
%\end{figure}

Algorithm~\ref{alg:0} outlines the cram method for on-policy
evaluation with contextual bandit data.  The key idea of cramming is
to evaluate the policy value difference after every update of bandit
algorithm, using all the remaining bandit data.  Specifically, at each
time $t$, we consider the policy value difference between the
previously learned policy $\hat\pi_{t-1}$ and the updated policy
$\hat\pi_t$, denoted by $\Delta(\hat\pi_t,\hat\pi_{t-1})$, using the
remaining $T-t$ observations.  Then, summing these policy value
improvements lead to the value of the final learned policy,
\begin{align}
  V(\hat\pi_{T}) %& \ = \ V(\pi_T) - V(\pi_0)%\ = \ \sum_{t=1}^T V(\pi_t) - V(\hat\pi_{t-1})\\
\ = \ V(\hat\pi_1)+\sum_{t=2}^{T} \Delta(\hat\pi_t,\hat\pi_{t-1}).
\end{align}

Let $\widehat\Delta(\hat\pi_t, \hat\pi_{t-1})$ denote the estimated
policy value difference at time $t$.  At the final time step $t=T$,
there is no bandit data left for estimating this policy value
difference.  Thus, the crammed policy value estimator uses the
following approximation,
\begin{align}
  \widehat{V}(\hat\pi_{T})\ \approx\ \widehat{V}(\hat\pi_{T-1}) \ = \ \widehat{V}(\hat\pi_1)+\sum_{t=2}^{T-1} \widehat{\Delta}(\hat\pi_t;\hat\pi_{t-1}).
\end{align}
In the next section, we impose a relatively weak stability condition
on the bandit algorithm so that this final policy difference term
$\Delta(\hat\pi_t, \hat\pi_{t-1})$ is indeed theoretically
negligible. 

\section{Statistical inference after cramming}
\label{sec:estimation_inference}

In this section, we introduce the crammed policy value estimator
$\widehat{V}(\hat\pi_{T-1})$ based on the inverse probability
weighting (IPW) and derive its asymptotic properties. We establish
that under mild regularity conditions, the proposed estimator is
$L_1$-consistent and asymptotically normal so long as a bandit
algorithm satisfies a certain stability condition. We then show that
this stability condition holds for commonly used bandit algorithms,
including $\epsilon$-greedy, Thompson Sampling, and UCB algorithms.

\subsection{Crammed policy evaluation estimator}

We consider the following IPW-based crammed policy evaluation
estimator.
\begin{definition}[The crammed IPW policy evaluation estimator]
  \label{def:crammed_estimator} \spacingset{1} The crammed IPW policy
  evaluation estimator of $V(\hat\pi_{T-1})$ is given by,
  $$
  \widehat{V}(\hat\pi_{T-1}) \ = \ \widehat{V}(\hat\pi_1) + \sum_{t=2}^{T-1} \widehat{\Delta}(\hat\pi_t, \hat\pi_{t-1})$$
where for $t = 2,\ldots,T-1$ and $j = t+1, t+2,\ldots,T$,
$$
  \begin{aligned}
\widehat V(\hat\pi_1) & \ = \ \frac{1}{T-1}\sum_{j=2}^T
\widehat{\Gamma}_{1j}, \ \quad \widehat{\Gamma}_{1j} \ = \
\sum_{a=1}^K \frac{
  \mathbb{I}(A_j=a)R_j}{\hat\pi_{j-1}(\bX_j,a)}\cdot\hat\pi_1(\bX_j,a), \\
\widehat{\Delta}(\pi_t, \pi_{t-1}) & \ = \ \frac{1}{T-t}
\sum_{j=t+1}^T \widehat{\Gamma}_{tj}, \ \text{and} \quad
\widehat{\Gamma}_{tj} \ = \ \sum_{a=1}^K \frac{
    \mathbb{I}(A_j=a)R_j}{\hat\pi_{j-1}(\bX_j,a)}\cdot(\hat\pi_t(\bX_j,a)-\hat\pi_{t-1}(\bX_j,a)).
\end{aligned}
$$
\end{definition}
In the above definition, $\widehat\Gamma_{tj}$ represents the IPW
estimator of $\Delta(\hat\pi_t,\hat\pi_{t-1})$ based on the $j$th
observation.  We then average $\widehat\Gamma_{tj}$ over all remaining
$T-t$ observations for $t+1\le j\le T$ to obtain the estimator
$\widehat{\Delta}(\hat\pi_t, \hat\pi_{t-1})$ at time $t$.  Summing
these policy value difference estimators for $t=2,\dots,T-1$ and
adding $\widehat V(\hat\pi_1)$ result in the crammed policy value
estimator $\widehat{V}(\hat\pi_{T-1})$.  While we consider the IPW
estimator for simplicity, we can also use the standard doubly robust
augmented IPW (AIPW) estimator by using,
$$
\widehat\Gamma_{tj}= \sum_{a=1}^K\left\{
  \frac{\mathbb{I}(A_j=a)}{\pi_{j-1}(\bX_j,a)} (R_j-\hat\mu_a(\bX_j))
  + \hat\mu_a(\bx)\right\}(\hat\pi_t(\bX_j,a)-\hat\pi_{t-1}(\bX_j,a)),
$$
where $\hat\mu_a(\bX_j)$ is the estimated
conditional mean function of the potential reward $R_j(a)$ given the
context $\bx$.

Definition~\ref{def:crammed_estimator} has the intuitive
interpretation mentioned earlier; the crammed policy value estimator
evaluates each incremental policy update of the bandit algorithm at
every time step with an IPW estimator based on the remaining data.
Unfortunately, the estimated policy value difference
$\widehat\Delta(\hat\pi_t;\hat\pi_{t-1})$ is not independent of each
other, complicating the asymptotic analysis of the crammed policy
value estimator.

To address this problem, we rewrite the crammed policy evaluation
estimator in the following equivalent but alternative form,
\begin{equation}
\widehat{V}(\hat\pi_{T-1}) \ = \ \sum_{j=2}^T  \widehat\Gamma_j(T), \label{eq:Gamma_j_definition} 
\end{equation}
where 
\begin{align}
  \widehat\Gamma_{j}(T)&\ = \ \
                         \sum_{t=1}^{j-1}\frac{1}{T-t}\widehat\Gamma_{tj}
  \nonumber \\
  & \ = \ \sum_{a=1}^K \frac{ \mathbb{I}(A_j=a)R_j}{\hat\pi_{j-1}(\bX_j,a)}\cdot\frac{\hat\pi_1(\bX_j,a) }{T-1}+ \sum_{t=2}^{j-1} \sum_{a=1}^K \frac{ \mathbb{I}(A_j=a)R_j}{\hat\pi_{j-1}(\bX_j,a)}\cdot\frac{(\hat\pi_t(\bX_j,a)-\hat\pi_{t-1}(\bX_j,a))}{T-t} \nonumber\\
  &\ = \ \sum_{a=1}^K \frac{ \mathbb{I}(A_j=a)R_j}{\hat\pi_{j-1}(\bX_j,a)}\cdot\left(\frac{\hat\pi_1(\bX_j,a)}{T-1}+\sum_{t=2}^{j-1}\frac{\hat\pi_t(\bX_j,a)-\hat\pi_{t-1}(\bX_j,a)}{T-t}\right) \label{eq:Gamma_j_definition2}
\end{align}
This alternative expression shows that at each time step $j$, we
estimate a weighted sum of all previous incremental policy value
differences up to time $j-1$.  The estimators of these policy value
differences, i.e.,
$\widehat{\Gamma}_{1j}, \widehat{\Gamma}_{2j}, \ldots,
\widehat{\Gamma}_{j-1,j}$, are weighted such that an earlier
incremental policy value difference has a smaller weight.  The advantage of
this alternative formulation is its clear martingale structure of
$\{\widehat\Gamma_j(T)\}_{j=2}^T$.  We now leverage this structure to
conduct a theoretical analysis of the crammed policy evaluation
estimator.

\subsection{Assumptions on a bandit algorithm}

To derive the asymptotic properties of the crammed policy evaluation
estimator introduced above, we impose two assumptions on the policy
learning algorithm.  The first is the key assumption of the cram
method that requires a bandit algorithm to stabilize at a certain rate
over time.
\begin{assumption}[Stability] \spacingset{1}
  \label{ass:learning_rate} For a given bandit algorithm and time step
  $t$, define the maximal average $L_1$ distance between $\hat\pi_t$ and
  $\hat\pi_{t-1}$ as,
  \begin{align*}
    Q_t :=
    %\max_{a\in\{1,2,...,K\}}\E_{\mathbf{X}}\left[|\hat\pi_t(\mathbf{X},a)-\hat\pi_{t-1}(\mathbf{X},a)|\right]
    %=
    \max_{a\in\{1,2,...,K\}}\int_{\mathbf{x}\in\mathcal{X}} |\hat\pi_t(\mathbf{x},a)-\hat\pi_{t-1}(\mathbf{x},a)| d F_{\mathbf{X}}(\mathbf{x})
  \end{align*}
  where $F_{\bX}$ is the CDF of the context $\bX$.  The bandit
  algorithm satisfies the following stabilization rate condition,
      \begin{equation*}
        \lim_{t\rightarrow\infty}\mathbb{E}\left[t^{1+\delta}Q_t\right] = 0, \label{ass:learning_rate2}
      \end{equation*}
      for any $\delta > 0$.
\end{assumption} 
We emphasize that $Q_t$ defined above is a random variable because the
learned policies, i.e., $\hat\pi_t$ and $\hat\pi_{t-1}$, depend on the
bandit data. 

The stabilization rate $t^{1+\delta}$ in
Assumption~\ref{ass:learning_rate} naturally arises from the fact that
for every iteration of cramming, we lose one observation used for
evaluation.  Therefore, the policy value difference must decrease at a
faster rate to guarantee the convergence of the crammed policy
evaluation estimator.  If the learned policy sequence
$\{\hat\pi_t\}_{t=1}^\infty$ generated by a bandit algorithm satisfies
the above stabilization condition, it has a limit in the $L_1$ metric
almost surely.  We emphasize that the learned policy sequence does not
have to converge to an optimal policy or any particular policy, which
is constant across different realizations of bandit data.  In other
words, different bandit runs can have different limit policies so long
as the stability condition holds. This could happen, for example, if
different bandit runs converge to different locally optimally
policies.

Similar stability conditions appear in the off-policy evaluation
literature mentioned in Section~\ref{sec:introduction}.  For example,
\cite{zhan2021off} requires a stability condition that can be
expressed in our notation as,
$$
\sup_{\bx,a} \left| \frac{\hat\pi_{t}^{-1}(\bx,a)}{\mathbb{E}[\hat\pi_{t}^{-1}(\bx,w)]} - 1 \right| \rightarrow 0, \quad a.s.
$$
This essentially requires the learned policy to converge to its
expectation over different bandit runs, uniformly over all context
$\bx$ and action $a$.  In contrast, we neither require uniform
convergence nor assume different bandit runs converge to the same
policy. \cite{hadad2021confidence} uses weights similar to those of
\cite{zhan2021off} without a stability assumption, but this is because
they consider non-contextual bandits.  While \cite{bibaut2021post} and
\cite{zhang2021statistical} do not require a stability condition, they
assume that the reward model can either be parametrically specified or
be well estimated nonparametrically. \cite{zhang2020inference} assumes
the bandit algorithm eventually select a single arm non-randomly,
which can be seen as a type of stability condition. Lastly,
\cite{luedtke2016statistical} consider an estimator of optimal policy
value in an i.i.d. setting.  They do not explicitly make a stability
assumption, but require the conditional variance of the estimator in
each sequential step to be estimated well.

Finally, we impose the following clipping condition on the learned
policy sequence generated by a bandit algorithm.
\begin{assumption}[Clipping] \spacingset{1}
  \label{ass:clip_rate} There exists a constant
  $\eta\in \left[0,\
    \min\left(\frac{1}{2},\frac{\delta}{1+\delta}\right)\right]$ with
  $\delta > 0$ such that for any given time step $t$, the action
  selection probability
  $\hat\pi_t(\bx, a) := \mathbb{P}\left(A_{t+1}=a\mid \bX_{t+1}=\bx,
    \bH_t\right)$ satisfies
  $$\hat\pi_t(\bx, a) \ge  \frac{c}{t^{\eta}}$$
  for all $a \in \cA$ and $\bx \in \bX$  
 where $c > 0$ is a constant.
\end{assumption}

The clipping condition deals with the fact that the variance of IPW
estimator explodes as the action selection probabilities approaches
zero.  A similar condition is commonly assumed in the off-policy
evaluation literature. For example, \cite{zhang2020inference} presents
the following uniform clipping rate as a sufficient condition for their
results, $c\le\hat\pi_t(\bx,a) \le 1-c $ for some $0<c<\frac{1}{2}$.
Similarly, \cite{zhang2021statistical} assumes the existence of a
constant stationary policy sequence $\{\pi_t^{sta}\}_{t=0}^{T-1}$ such
that for all context $\bx$ and action $a$,
$$
0<\rho_{\min} \le \frac{\pi_{t}^{sta}(\bx,a)}{\hat\pi_{t}(\bx,a)} \le \rho_{\max}<\infty \ \ a.s.
$$
If the stationary policy is unknown, this is equivalent to assuming
that the action selection probabilities are bounded away from zero. In
contrast, our clipping rate assumption allows action selection
probabilities to approach zero.

Furthermore, for off-policy evaluation of a fixed policy with bandit
data, \cite{zhan2021off} and \cite{bibaut2021post} make an important
progress by only assuming the action selection probabilities are
bounded by $t^{-\alpha}$ where $0\le\alpha<\frac{1}{2}$.
\cite{shen2024doubly} use the same $t^{-1/2}$ clipping rate condition,
but assumes the correctly specified linear contextual bandit
algorithms such as $\epsilon$-greedy, UCB, and Thompson Sampling.
Lastly, unlike other off-policy evaluation methods discussed here,
\cite{luedtke2016statistical} considers an i.i.d setting and hence
does not require a clipping rate assumption.

We have a slightly stronger clipping rate than the existing rate of
$t^{-1/2}$ in the off-policy evaluation literature.  This is primarily
because cram evaluates the incremental policy value improvements and
we must control their evaluation. Indeed, the off-policy evaluation
can be thought of a special case of cram, in which a policy never
updates, i.e., $\delta=\infty$. In this case, the cram method will
have the same clipping rate condition as the best available result for
off-policy evaluation.  This explains why the clipping rate directly
relates to the stability condition.

\subsection{Consistency and asymptotic normality} 

We now establish the consistency and asymptotic normality of the
crammed policy evaluation estimator introduced in
Definition~\ref{def:crammed_estimator}.  We require the following mild
regularity conditions about potential rewards.
\begin{assumption}[Bounded conditional expectation and conditional variance]
  \label{ass:bounded} \spacingset{1} Both the conditional expectation 
  and conditional variance of the potential reward, i.e.,
  $\mu_a(\bx):=\E[R(a) \mid \bX = \bx]$ and
  $\sigma_a^2(\bx):=\V(R(a) \mid \bX = \bx)$ for $a=1,2,...,K$,
  respectively, are uniformly bounded over the context space $\cX$:
    $$\sup_{\bx\in\cX} |\mu_a(\bx)| < \mu_U<\infty, \quad
    0<\sigma_L^2<\inf_{\bx\in\cX}\sigma_a^2(\bx)\le\sup_{\bx\in\cX}
    \sigma_a^2(\bx) < \sigma_U^2< \infty.$$
\end{assumption}
\begin{assumption}[Moment condition]
  \label{ass:fourth_moment} \spacingset{1} Each potential reward has
  a finite fourth moment:
  $$\E[R(a)^4]\le K_4 < \infty,  \quad \text{for} \ a = 1,2,...,K.$$
\end{assumption}
These regularity conditions are expected to hold in many bandit
settings and are commonly made in the off-policy evaluation
literature.  For example, \cite{hadad2021confidence} and
\cite{zhan2021off} assume the identical conditions.
% \cite{zhang2021statistical} assumes the fourth order moment
% condition for the $M$-estimator, but does not require
% Assumption~\ref{ass:bounded} because their goal is the estimation of
% parameters in a parametric model rather than policy evaluation.

We now present the main theorems that establish the asymptotic
properties of the crammed policy evaluation estimator, as the number
of time steps $T$ goes to infinity.  We first focus on the asymptotic
behavior of $\widehat{V}(\hat\pi_{T-1})$ and then later show that the
difference between $V(\hat\pi_{T})$ and $V(\hat\pi_{T-1})$ is
asymptotically negligible.
\begin{theorem}[consistency]
  \label{thm:L1Consistency} \spacingset{1} Suppose that a sequence of
  learned policies $\{\hat\pi_t\}_{t=1}^T$ satisfies
  Assumption~\ref{ass:learning_rate}.  Then, under
  Assumptions~\ref{ass:stationarity},~\ref{ass:clip_rate},~\ref{ass:bounded},~and~\ref{ass:fourth_moment},
  we have,
        $$\bigr|\widehat V(\hat\pi_{T-1}) - V(\hat\pi_{T-1})\bigr|\ \xrightarrow{P}\ 0 \ \ \text{as}\ \  T\rightarrow\infty.$$
\end{theorem}
Proof is given in Appendix~\ref{thm:L1Consistency_proof}. 
\begin{theorem}[Asymptotic normality] \spacingset{1}
  \label{thm:Asymptotic_Normality} Suppose that a sequence of policies
  $\{\hat\pi_t\}_{t=1}^T$ satisfies
  Assumption~\ref{ass:learning_rate}.  Then, under
  Assumptions~\ref{ass:stationarity},~\ref{ass:clip_rate},~\ref{ass:bounded},~and~\ref{ass:fourth_moment},
  we have,
    $$\sqrt{T-1}\cdot\frac{\widehat{V}(\hat\pi_{T-1}) - V(\hat\pi_{T-1})}{v_T}\ \stackrel{d}{\longrightarrow}\ \mathcal{N}(0,1).$$
    The asymptotic variance is given by,
    $$v_T^2 \ := \ (T-1)\sum_{j=2}^T \V(\widehat\Gamma_j(T)\mid \bH_{j-1}).$$
\end{theorem}
Proof is given in Appendix~\ref{thm:Asymptotic_Normality_proof}.
Unlike the standard central limit theorem, both the estimand
$\widehat{V}(\hat\pi_{T-1})$ and the asymptotic variance $v_T^2$ are
random variables that are functions of the observed bandit data. Since
the goal of cramming is the on-policy evaluation of the final learned
policy, the asymptotic variance also depends on the realized data
sequences and policies.

It may be a surprise to find that the asymptotic normality only
requires the stability of a bandit algorithm in terms of the $L_1$
distance (Assumption~\ref{ass:learning_rate}) rather than $L_2$
distance.
%\textbf{ML: Zeyang, we should explain a little bit more about the
%significance of this. What would we lose if we only had $L_2$,
%etc..}\zj{I'm not sure how to present this as it's a bit
%technical. If we assume the $L_2$ norm is of the order
%$t^{-1-\delta}$, then we are not able to show the asymptotic
%normality. If we want the same proof to proceed, we will need a
%$t^{-1-\delta-\delta'}$ $L_2$ stability condition where $\delta'$ is
%some positive constant. I should be able to find out the $\delta'$
%but I'm not certain what will be the value. }
This is mainly due to the cancellation of the policy differences that
occur in cramming as we sum over the $T-1$ iterations to arrive at the
final crammed evaluation estimator.

% The use of $L_1$ condition is important because in the policy learning setup, the $L_1$ condition can imply the $L_2$ condition:
% \begin{align}
% {\mathbb{E}_{\bX}\left[|\hat\pi_t(\bX,a)-\hat\pi_{t-1}(\bX,a)|^2\right]} \le \mathbb{E}_{\bX}\left[|\hat\pi_t(\bX,a)-\hat\pi_{t-1}(\bX,a)|\right] \le \sqrt{\mathbb{E}_{\bX}\left[|\hat\pi_t(\bX,a)-\hat\pi_{t-1}(\bX,a)|^2\right]}
% \end{align} 
% If we assume a $L_1$ condition of $\mathbb{E}_{\bX}\left[|\hat\pi_t(\bX,a)-\hat\pi_{t-1}(\bX,a)|\right] \sim t^{-1-\delta}$ rate

Next, we introduce the crammed variance estimator, which is consistent
for the asymptotic variance $v_T^2$ of the crammed policy evaluation
estimator.  
\begin{definition}[The crammed variance estimator]
  \spacingset{1} \label{def:crammed_variance} Define the following
  estimator of $\V(\widehat\Gamma_j(T)\mid H_{j-1})$ for $j<T$:
\begin{align*}
    \hat v^2_{Tj} \ : = \frac{1}{T-j+1}\sum_{k=j}^T(G_{Tjk} - \bar G_{Tj\cdot})^2 ,
\end{align*}%\footnote{For finite sample adjustment, replace $T-j+1$ with $T-j$.}
where
$$
\bar G_{Tj\cdot} \ = \ \frac{1}{T-j+1}\sum_{k=j}^T G_{Tjk},  \quad
 G_{Tjk} \ : =  \ \sum_{a=1}^K\frac{R_k\mathbb{I}(A_k=a)}{\pi_{k-1}(\bX_k, a)}\left(\frac{\hat\pi_1(\bX_k,a)}{T-1}+\sum_{t=2}^{j-1}\frac{\pi_t(\bX_k,a)-\pi_{t-1}(\bX_k,a)}{T-t}\right).
$$
Then, the crammed variance estimator is
defined as:
$$\hat v^2_{T}\ := \ (T-1)\sum_{j=2}^{T} \hat v_{Tj}^2,$$
where $\hat v_{TT}^2=0$.
\end{definition}
Note that since we do not have enough data to estimate the last
variance term $\V(\widehat\Gamma_T(T)\mid \bH_{T-1})$, we set the
corresponding variance estimator to zero, i.e., $\hat v_{TT}^2=0$. In
a batched bandit, however, this variance parameter can be estimated.

The next theorem shows that this crammed variance estimator is
consistent. 
\begin{theorem}[Consistency of the crammed variance estimator.]
  \label{thm:variance_estimator}\spacingset{1} Suppose that the
  conditions of Theorem~\ref{thm:Asymptotic_Normality} hold.  Then, as
  $T\rightarrow\infty$, we have:
  $$| \hat v_{T}^2 - v_T^2|\ \stackrel{p}{\longrightarrow}\ 0.$$
\end{theorem}
Proof is given in Appendix~\ref{thm:variance_estimator_proof}.
Theorem~\ref{thm:variance_estimator} permits the construction of an
asymptotically valid confidence interval.  This result is stated as
the following corollary.
\begin{corollary}[Asymptotic confidence intervals]
  \label{cor:ci} \spacingset{1} Suppose that the conditions of
  Theorem~\ref{thm:Asymptotic_Normality} hold. Then, as
  $T\rightarrow \infty$, we have:
  $$  \sqrt{T-1}\cdot\frac{\widehat{V}(\hat\pi_{T-1}) - V(\hat\pi_{T-1})}{\hat v_T}\ \stackrel{d}{\longrightarrow}\ \mathcal{N}(0,1)$$
\end{corollary}
Proof is given in Appendix~\ref{cor:ci_proof}.

Lastly, we show that the final policy value difference,
$\Delta(\hat\pi_T, \hat\pi_{T-1})$ is asymptotically negligible. In
this way, the consistency and asymptotic normality results in
Theorems~\ref{thm:L1Consistency}~and~\ref{thm:Asymptotic_Normality}
still hold if we replace the estimand $V(\hat\pi_{T-1})$ with
$V(\hat\pi_T)$.
\begin{corollary}[Asymptotically negligible final policy difference]
  \label{cor:final_policy} \spacingset{1} Suppose that the conditions of
  Theorem~\ref{thm:Asymptotic_Normality} hold. Then, as
  $T\rightarrow \infty$, 
\begin{align}
  (T-1)\mathbb{E}\left[\left|\Delta(\hat\pi_T;\hat\pi_{T-1})\right|\right]\ &\rightarrow\ 0 \nonumber
\end{align}
Therefore, we have,
$$\left|\widehat V(\hat\pi_{T-1}) - V(\hat\pi_{T})\right|\
 \xrightarrow{P}\ 0, \quad
  \sqrt{T-1}\cdot\frac{\widehat{V}(\hat\pi_{T-1}) - V(\hat\pi_{T})}{\hat v_T}\ \stackrel{d}{\longrightarrow}\ \mathcal{N}(0,1).$$
\end{corollary}
Proof is given in Appendix~\ref{cor:final_policy_proof}. 

\subsection{Common linear bandit algorithms}
\label{subsec:stability_bandit}

The above asymptotic results critically depend on the stabilization
rate condition (Assumption~\ref{ass:learning_rate}).  This assumption
essentially requires a bandit algorithm to learn fast enough so that
the learned policy sequence converges at a faster-than-linear
rate. Here, we show that this stability assumption holds for commonly
used linear contextual bandit algorithms such as $\epsilon$-greedy,
Thompson Sampling, and UCB.  

We consider the following standard linear contextual bandit
formulation that is common in the literature.  For example,
\cite{shen2024doubly} maintain an essentially identical set of
assumptions stated below.\footnote{The only difference is that
  \cite{shen2024doubly} assumes a separation condition in probability.
  Although we could also assume that the gap between optimal arm and
  other arms is bounded in probability, we maintain uniform separation
  for simplicity.}  We emphasize that these assumptions are not
required by the cram method itself.
\begin{align*}
R(\bx, a) \ = \ \bx^\top\btheta(a) + \epsilon \ = \ \Phi(\bx,a)^\top \bTheta + \epsilon,
\end{align*}
for each arm $a\in \cA=\{1,2,...,K\}$ and
$\bx \in \mathcal{X}\subset \mathbb{R}^d$ where
$\Phi(\bx,a)=[\bx I(a=1)\ \bx I(a=2)\ \cdots\ \bx I(a=K)]^\top$,
$\Theta = [\btheta(1)\ \btheta(2)\ \cdots\ \btheta(K)]^\top$, and the
error term follows a normal distribution with a constant variance
$\epsilon \iid \mathcal{N}(0, \sigma^2)$.  
\begin{assumption}[Bounded context]
  \label{ass:xbounded_bandit} \spacingset{1}
  There exists a constant $L > 0$ such that for any context
  $\bx\in\mathcal{X}$, we have $\norm{\bx} \le L$.
\end{assumption}
\begin{assumption}[Positive definite] 
  \label{ass:positivedefinite} \spacingset{1}
  The matrix $\mathbb{E}\left[\bX\bX^\top \right]$ is positive definite.
%  \begin{align}
%  \lambda_{\min} \left(\mathbb{E}\left[\bX\bX^T\right]\right) > 0
%  \end{align}
\end{assumption}
\begin{assumption}[Separation] \spacingset{1}
  \label{ass:uniform_gap} The policy value difference between the optimal arm and the second
  best arm is uniformly bounded away from zero,
  $$
  \Delta(\bx) = \inf_{\bx\in\mathcal{X}} \left[\Phi(\bx,a^*(\bx))^\top \bTheta - \max_{a\not =
  a^*(\bx)}\Phi(\bx,a)^\top \bTheta\right] > 0,
$$
where $a^*(\bx) = \argmax_{a=1,2,...,K}\Phi(\bx,a)^\top \bTheta$
denotes the optimal arm.
\end{assumption}
\begin{assumption}[Clipping] \spacingset{1}
  \label{ass:clip_bandit} There exists a deterministic sequence $c_t$ such that for any $t, \bx, a$,
  \begin{align}
  \hat\pi_t(\bx,a) \in [c_t, 1-(K-1)c_t].
  \end{align}
  where $c_t \ge \frac{c}{t^\zeta}$ and $0\le\zeta<\frac{1}{2}$
\end{assumption}

We require
Assumptions~\ref{ass:xbounded_bandit}~and~\ref{ass:positivedefinite}
to avoid a degenerate case, in which different context variables are
linearly dependent.  This guarantees that a linear bandit algorithm
converges to the optimal policy.  These assumptions are also made by
\cite{zhang2020inference} and are common when deriving a regret bound
on linear contextual bandits
\citep[e.g.,][]{rusmevichientong2010linearly,abbasi2011improved}.
Assumption~\ref{ass:uniform_gap} implies that there is a non-zero gap
between the optimal arm and the second best arm in any context.  This
ensures that a bandit algorithm stabilizes without continuing to
switch between arms.  Lastly, Assumption~\ref{ass:clip_bandit} is
implies by Assumption~\ref{ass:clip_rate} but is stated for
completeness. 

Now, we are ready to show that these assumptions of linear contextual
bandit algorithms imply our stability condition.  We emphasize that
these are sufficient conditions, meaning that the violation of these
assumptions does not necessarily imply the stabilization condition
(Assumption~\ref{ass:learning_rate}) fails to hold.  For example, even if
Assumption~\ref{ass:uniform_gap} is violated, a contextual bandit
algorithm that converges to constant probabilities of selecting multiple
optimal arms still satisfies our stability condition.
\begin{theorem}[Stability condition under linear contextual bandit algorithms]
  \label{thm:linear_contextual_bandit} \spacingset{1} Suppose
  $\{(\bX_t,A_t,R_t)\}_{t=1}^T$ is a sequence of bandit data collected
  from $\epsilon$-greedy, UCB, or Thompson Sampling algorithm.  If
  Assumptions~\ref{ass:xbounded_bandit},~\ref{ass:positivedefinite},~\ref{ass:uniform_gap}
  and~\ref{ass:clip_bandit} hold, the resulting policy sequence
  $\{\hat\pi_t\}_{t=0}^\infty$ satisfies
  Assumption~\ref{ass:learning_rate}.
\end{theorem}
Proof of the theorem is given in
Appendix~\ref{thm:linear_contextual_bandit_proof}.

\section{Numerical experiments}
\label{sec:synthetic}

In this section, we conduct numerical experiments; one based on
synthetic data sets and the other based on real-world data
sets. Across our experiments, the cram method achieves roughly 40\%
reduction in the evaluation root mean square error (RMSE) on average,
when compared to the 80--20\% sample split off-policy evaluation. It
also has comparable bias to the sample splitting methods, and valid
empirical coverage of 95\% confidence interval.

\subsection{Synthetic data}
\label{sec:ss_method}

We consider a contextual bandit setting with four arms.  At each time
step, we generate a three-dimensional vector of context variables
$\bX_t = [X_{t1}\ X_{t2}\ X_{t3}]^\top$ where
$X_{ti} \iid \text{Unif}(-1,1)$ for all $t$ and $i\in\{2,3\}$ and the
first element corresponds to the intercept term, i.e., $X_{t1}=1$ for
all $t$. We specify the outcome model as a linear contextual bandit
model, $R_t(a)=\bX_t^\top\bbeta_{a}+\epsilon$, where
$\epsilon\iid N(0,1)$, $\bbeta_1 = [ \beta\ 0\ 0]^\top$,
$\bbeta_2 = [ 0\ \beta\ 0]^\top$, and
$\bbeta_3 = [ 0\ 0\ \beta]^\top$.  We consider three different signal
strength: no signal ($\beta = 0$), weak signal ($\beta = 0.5$), strong
signal ($\beta = 2$). We apply linear $\epsilon$-greedy, UCB, and
Thompson Sampling algorithms to these synthetic data sets, and clip
the action selection probability at the rate of $0.025t^{-\eta}$. We
fix a batch size to $20$ throughout these experiments, but vary the
sample size of $T$ from $20$ to $160$, with clipping rate
$\eta=0, 0.5$, and the signal strength $\beta=0,0.5,2$.

\begin{algorithm}[t!] \spacingset{1}
  \caption{Sample splitting method with off-policy evaluation using bandit data}\label{alg:1}
  \KwIn{Data $\{(\bX_t,A_t,R_t)\}_{t=1}^T$ collected by a contextual
    bandit algorithm; Sequence of learned policies
    $\{\hat\pi_1,\ldots,\hat\pi_T\}$ obtained via the bandit
    algorithm. Proportion of the training set $p_{\text{train}}$}
  \KwOut{Estimated value of the policy learned using the first
    $[p_{\text{\text{train}}}T]$ observations,
    $\widehat{V}^{SS}(\hat\pi_{[p_{\text{train}}T]})$}
%  Set $\hat\pi_0 = \pi_0$\; 
  \For{$t=[p_{\text{train}}T]+1$ \KwTo $T$}{ 
    Estimate the value of policy $\hat\pi_{[p_{\text{train}}T]}$ as:
    $$
    \widehat V_t^{SS} = \sum_{a=1}^K\left(\frac{\mathbb{I}(A_t=a)(R_t - \hat\mu_{t-1}(\bX_t,a))}{\hat\pi_{t-1}(\bX_t,a)} + \hat\mu_{t-1}(\bX_t,A_t)\right)\hat\pi_{[p_{\text{train}}T]}(\bX_t,A_t),
    $$
    where the $\hat\mu_{t-1}(\bx,a)$ is the estimated conditional
    expectation of potential reward $R_{t-1}(a)$ given the context
    $\bX_{t-1}=\bx$ } Estimate the value of the final learned policy
  $\hat\pi_{T}$ as:
  \begin{align*}
    \widehat V^{SS} = \sum_{t=[Tp_{\text{train}}]+1}^T \frac{W_t\widehat V_t^{SS}}{\sum_{t=[p_{\text{train}}T]+1}^T W_t},
  \end{align*}
  where $W_t$ is the adaptive weight defined in~\cite{zhan2021off}.
\end{algorithm}

We compare the cram method with the standard sample splitting, which
is commonly used for policy evaluation with adaptively collected
data. For cram, we evaluate the value of policy $\hat\pi_{T-1}$ using
the entire data sequence (see Algorithm~\ref{alg:0}). For sample
splitting, we consider both 60--40\% and 80--20\% train-test splits
and use the test data to conduct off-policy evaluation of the policy
learned from the training data (see Algorithm~\ref{alg:1}). We apply
the state-of-the-art off-policy evaluation method developed by
\cite{zhan2021off} with four different adaptive weights;
non-contextual variance stabilization (NS), non-contextual variance
minimization (NO), contextual variance stabilization (CS), and
contextual variance minimization (CO) weights.  In the simulation, we
find that these different weights achieve almost identical results.
This is mainly due to the fact that in our settings bandit algorithms
are relatively stable when the 80\% of the data are used for training.
Thus, we only present the results based on NS weights. For fair
comparison, we use the AIPW estimator for both cram and
sample-splitting where we use the bandit estimates of the conditional
reward function at time $t-1$ to compute the augmentation term at time
$t$.  We run 2000 Monte Carlo replications under each simulation
setup, and compare the two methods in terms of bias, RMSE, and the
empirical coverage of the 95\% confidence intervals.

\begin{figure}[t!]
  \centering \spacingset{1}
  \begin{subfigure}[b]{0.49\textwidth}
      \centering
      \includegraphics[width=\textwidth]{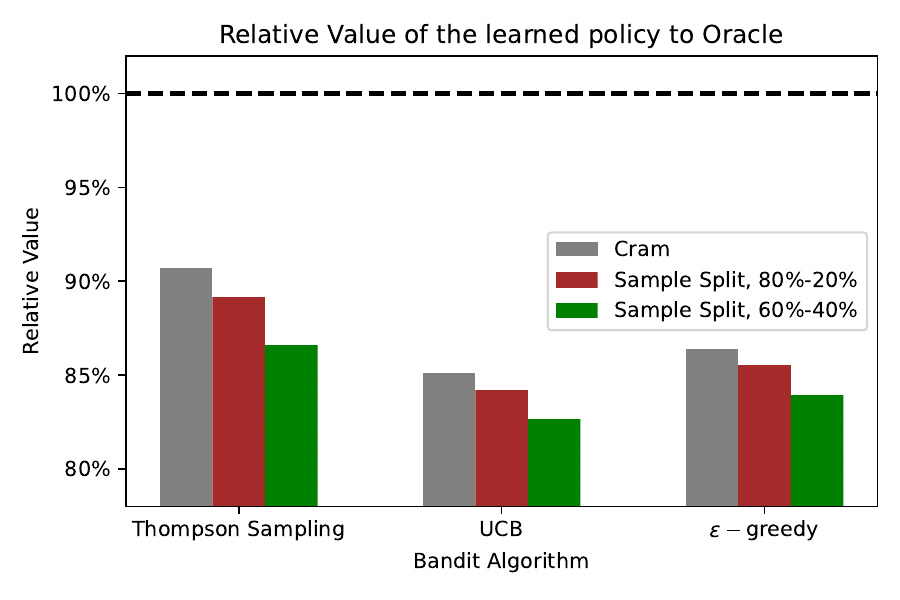}
      \caption{Relative Value}
      \label{fig:Method_Value}
  \end{subfigure}
  \begin{subfigure}[b]{0.49\textwidth}
      \centering
      \includegraphics[width=\textwidth]{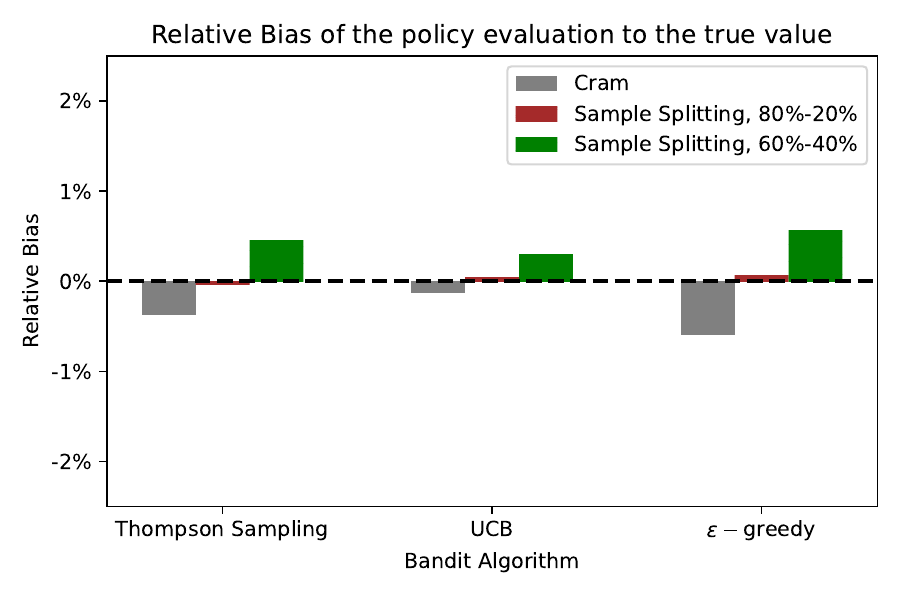}
      \caption{Relative Bias}
      \label{fig:Method_Bias}
  \end{subfigure}
  \begin{subfigure}[b]{0.49\textwidth}
      \centering
      \includegraphics[width=\textwidth]{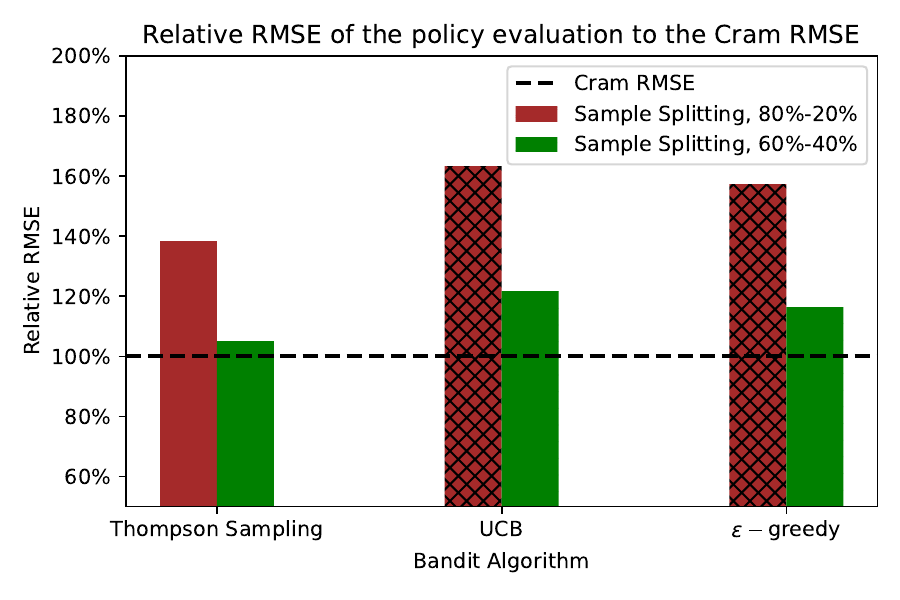}
      \caption{Relative RMSE}
      \label{fig:Method_RMSE}
  \end{subfigure}
  \begin{subfigure}[b]{0.49\textwidth}
    \centering
    \includegraphics[width=\textwidth]{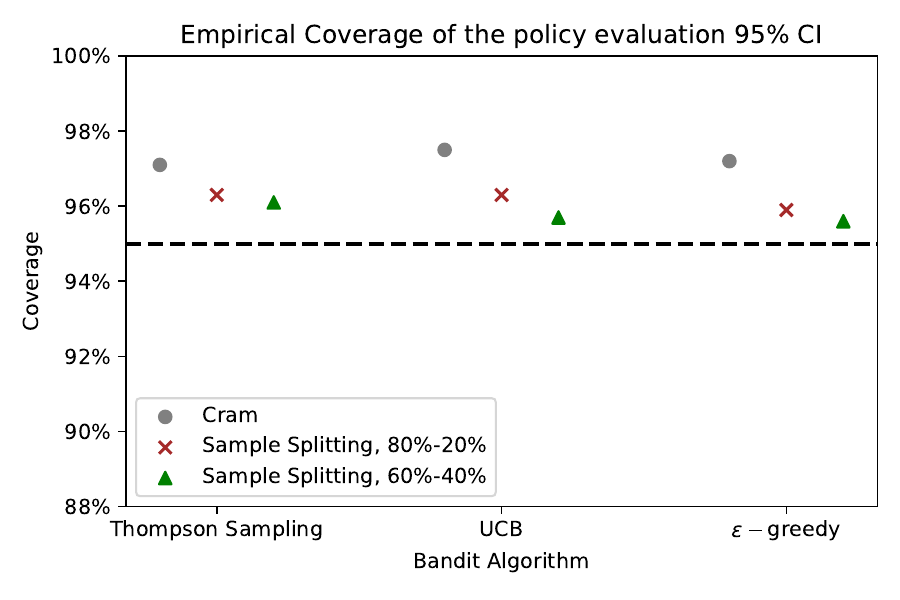}
    \caption{Coverage of 95\% CI, $\beta=2$}
    \label{fig:Method_Coverage}
\end{subfigure}
\caption{Learning and evaluation performance of cram and sample
  splitting for synthetic data.  For sample splitting, we use
  non-contextual variance stabilization weights of \cite{zhan2021off}.
  The parameters are set to $T=100$ (sample size), $\eta=0.5$
  (clipping rate), and $\beta=0.5$ (signal strength). See
  Appendix~\ref{app:additional_simulation} for additional parameter
  settings. For sample splitting, we consider 80--20\% and 60--40\%
  train-and-test splits.  The figure shows policy value relative to
  the oracle (a. top left), bias relative to true policy value (b. top
  right), RMSE relative to cram (c. bottom left), and empirical
  coverage of 95\% confidence intervals (d. bottom right). }
  \label{fig:Method-Synthetic}
\end{figure} 
 
Figure~\ref{fig:Method-Synthetic} shows the results that compare the
performance of cram with sample splitting based on NS weights and
$T=100,\beta=0.5,\eta=0.5$. Appendix~\ref{app:additional_simulation}
presents additional experimental results with different parameter
values.  These results are consistent with those presented here.  For
all three bandit algorithms, the final learned policy based on cram
has a higher value than those based on the sample splitting methods.
This makes sense because cram uses the entire data to learn a policy
while sample splitting reserves a subset of the data for evaluation.
For both cram and sample splitting methods, the estimated values of
learned policies are unbiased (Figure~\ref{fig:Method_Bias}).  In
addition, the cram method has consistently smaller RMSE than sample
splitting though the improvement is much greater for the 80--20\%
split because the latter has a small evaluation data set
(Figure~\ref{fig:Method_RMSE}).  Finally, the coverage of the 95\%
confidence intervals based on the cram method is somewhat greater than
the nominal level though sample splitting also exhibits some
overcoverage. This is mainly due to the fact that our variance
estimator tends to overestimate the true variance in finite samples,
as the consistency relies on the policy sequence to stabilize over
time. With a stronger signal size, the bandit algorithm stabilizes
faster so our variance estimator is more accurate, making the coverage
closer to the 95\% nominal level.  (see
Appendix~\ref{app:additional_simulation} for details).
 
\subsection{Real-world data}

Next, we follow the prior work
\citep[e.g.,][]{dudik2011doubly,dimakopoulou2017estimation,zhan2021off}
and use multi-class classification datasets to generate bandit
data. The idea is to transform a sequential classification problem
into a bandit problem, by taking features as contexts and potential
classes as potential arms. At each time step $t$, we randomly sample,
with replacement, an observation consisting of features and class
label from a classification dataset. Then, we use the set of sampled
features as the context $\bX_t$, and a bandit algorithm is used to
select an arm $A_t$. The reward of and the potential outcomes under
different arms are set as $R_t = 1 + \epsilon_t$ if the bandit
algorithm chooses the same class as the observed class label, and
$R_t = \epsilon_t$ otherwise, where $\epsilon\iid \mathcal{N}(0,1)$.
We repeat this process to obtain the bandit data
$\{\bX_t,A_t,R_t\}_{t=1}^T$.

We consider 86 public datasets from OpenML
\citep{vanschoren2014openml}, which vary in the number of classes, the
number of features, and the number of observations.  We use the linear
UCB algorithm to collect data, and fix batch size $B=20$ and length
$T=100$. We compare the evaluation bias and MSE between the cram and
sample splitting (80--20\% train-test split) methods.

\begin{figure}[t!]
  \centering \spacingset{1}
  \begin{subfigure}[b]{0.49\textwidth}
    \centering
    \includegraphics[width=\textwidth]{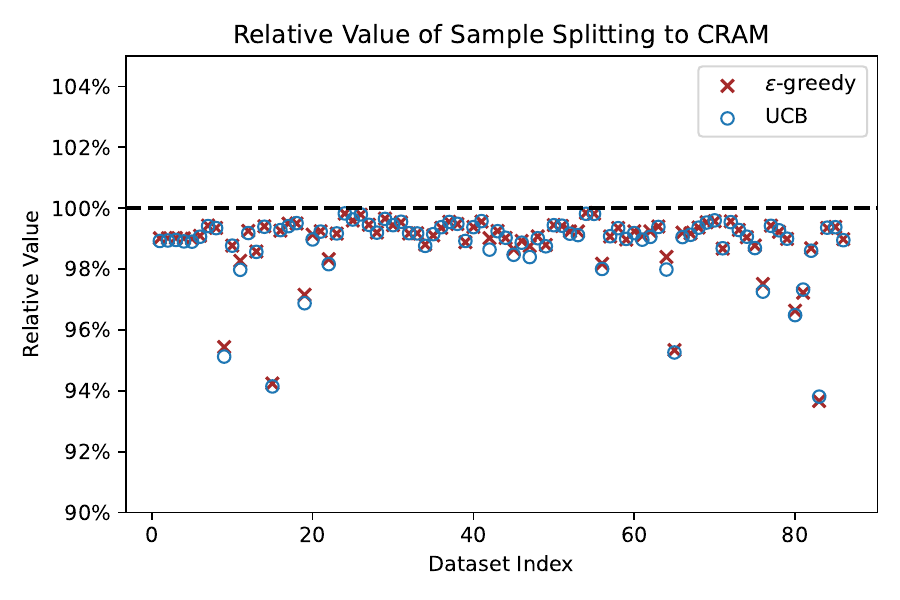}
    \caption{Value relative to cram}
    \label{fig:Epsilon_value_dataset}
\end{subfigure}
  \begin{subfigure}[b]{0.49\textwidth}
      \centering
      \includegraphics[width=\textwidth]{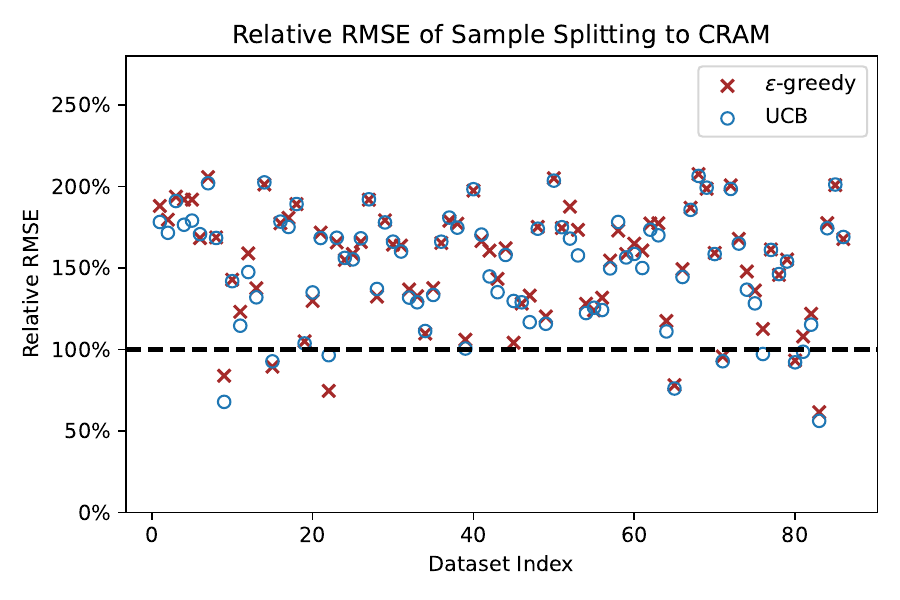}
      \caption{RMSE relative to cram}
      \label{fig:Epsilon_rmse_dataset}
  \end{subfigure}
  \begin{subfigure}[b]{0.49\textwidth}
      \centering
      \includegraphics[width=\textwidth]{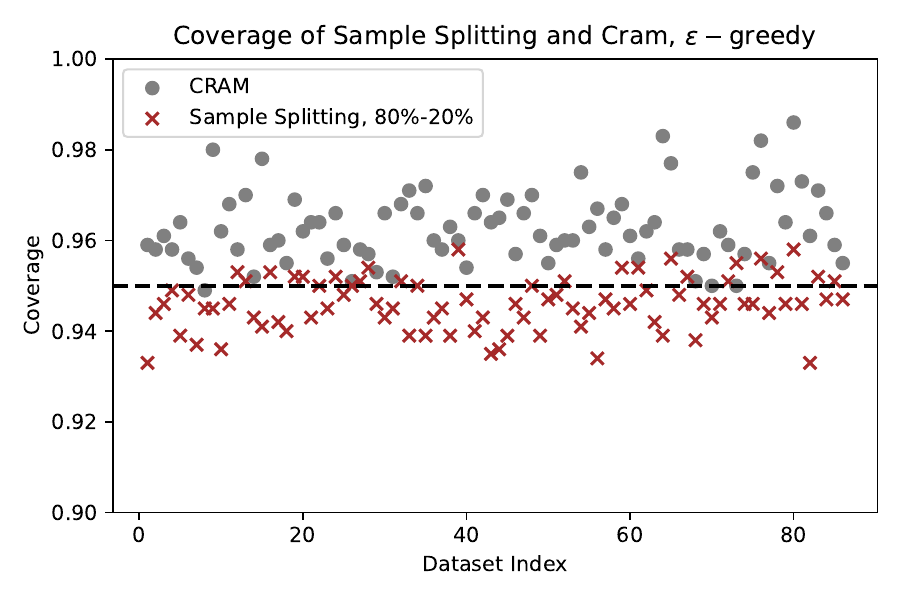}
      \caption{Empirical coverage of 95\% CI, $\epsilon-$greedy}
      \label{fig:Epsilon_coverage_dataset}
  \end{subfigure} 
  \begin{subfigure}[b]{0.49\textwidth}
    \centering
    \includegraphics[width=\textwidth]{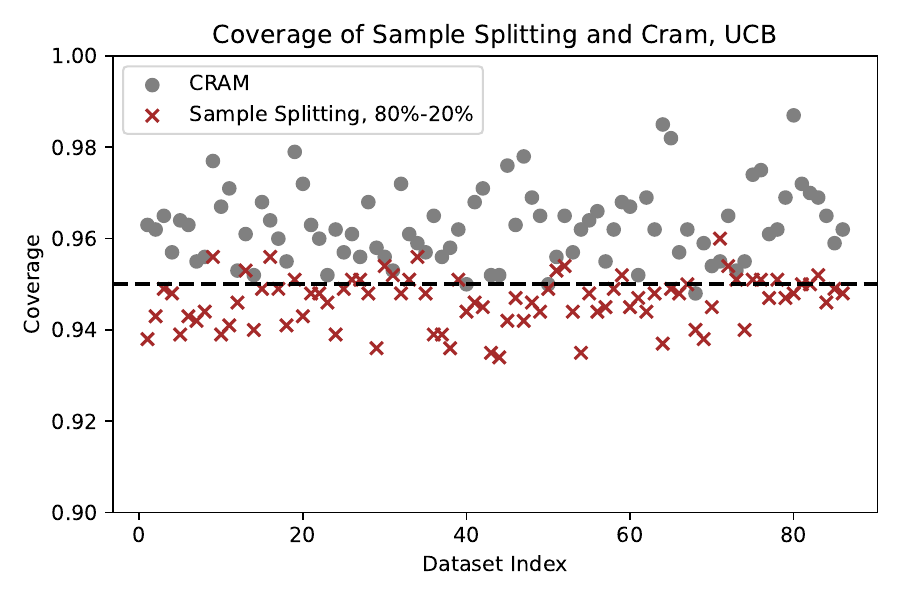}
    \caption{Empirical coverage of 95\% CI, UCB}
    \label{fig:UCB_coverage_dataset}
\end{subfigure} 
\caption{Learning and evaluation performance of cram and sample
  splitting for real-world data.  For sample splitting, we use
  non-contextual variance stabilization weights of \cite{zhan2021off},
  and we consider 80--20\% train-and-test split. We use AWAIPW-NS
  weights as an example and the results for other weights are similar.
  The figure shows policy value based on sample splitting relative to
  cram (a. top left), RMSE of sample splitting relative to cram
  (b. top right), and empirical coverage of 95\% confidence intervals
  for $\epsilon$ greedy (c. bottom left) and UCB (d. bottom right). }
  \label{fig:dataset}  
\end{figure} 
  
Figure~\ref{fig:dataset} summarizes the results across all 86 datasets
with $\epsilon$-greedy and UCB algorithms. We set the decaying rate
parameter $\eta=0.5$.  As shown in
Figure~\ref{fig:Epsilon_value_dataset}, across all 86 settings, the
final learned policies based on cram outperform the corresponding
policies learned with 80--20\% sample splitting.  The cram method also
achieves smaller evaluation RMSE in most datasets
(Figure~\ref{fig:Epsilon_rmse_dataset}). There are few data sets where
sample splitting tends to have a smaller RMSE than the cram method,
but this is mainly due to the low signal-to-noise ratio in these
settings. In such cases, the bandit algorithms learn little and as a
result do not stabilize. In
Figures~\ref{fig:Epsilon_coverage_dataset}~and~\ref{fig:UCB_coverage_dataset},
we show the empirical coverage of the 95\% confidence interval for
$\epsilon$-greedy and UCB.  As is the case in the synthetic data, the
cram method has a slightly conservative coverage.

\section{Concluding remarks}
\label{sec:conclusion}

We develop the cram method as a general methodological framework for
the on-policy evaluation of contextual bandit algorithms.  We believe
that this new methodology contributes to the growing literature on
statistical evaluation of bandit algorithms, which has largely focused
upon off-policy evaluation.  We establish that cramming any general
bandit algorithm, which satisfies a certain stability condition,
yields a consistent and asymptotically normal policy value estimator.
To illustrate the applicability of the cram method, we show that
commonly used linear contextual bandit algorithms satisfy the required
stability condition.  Applications to the synthetic and real-world
data demonstrate the potential power of the cram method.

Future research should consider various extensions of the cram
method.  While we focus on on-policy evaluation, the basic idea of
sequential evaluation can also be applied to off-policy learning and
evaluation with adaptively or non-adaptively collected data.  Another
area of application is active learning where data-efficient methods
like cram can play a critical role.  Finally, in all of these cases
including on-policy evaluation, it is of interest to consider
further improving the statistical efficiency of the cram method by
optimally weighting each step of evaluation.

% \section{Future extensions}
% \label{sec:extensions}

\clearpage
\bibliography{reference}
\pagebreak

\appendix
\spacingset{1}

\setcounter{equation}{0}
\setcounter{figure}{0}
\setcounter{table}{0}
\setcounter{section}{0}
\renewcommand {\theequation} {S\arabic{equation}}
\renewcommand {\thefigure} {S\arabic{figure}}
\renewcommand {\thetable} {S\arabic{table}}
\renewcommand {\thesection} {S\arabic{section}}

\begin{center}
\LARGE {\bf Supplementary Appendix}
\end{center}

\section{Proof of Theorem~\ref{thm:L1Consistency}}
\label{thm:L1Consistency_proof}

\begin{proof}
  Recall Theorem~\ref{thm:Asymptotic_Normality}, which shows,
  \begin{align}
    \sqrt{T-1}\cdot\frac{\widehat{V}(\hat\pi_{T-1}) - V(\hat\pi_{T-1})}{v_T} \xrightarrow{d} \mathcal{N}(0,1) \nonumber.
  \end{align}
  By Condition~\ref{cond:vt_lower_bound} in the proof of
  Theorem~\ref{thm:Asymptotic_Normality}, there exists a constant
  $c_1>0$ such that
  $$
  \lim_{T\rightarrow\infty}\mathbb{P}\left(v_T \ge c_1\right) = 0
  $$
  Therefore, as $T\rightarrow\infty$, 
  $$
\frac{v_T}{\sqrt{T-1}} \xrightarrow{P} 0.
  $$
By Slutsky's theorem, we have,
\begin{align}
  \sqrt{T-1}\cdot\frac{\widehat{V}(\hat\pi_{T-1}) - V(\hat\pi_{T-1})}{v_T} \cdot \frac{v_T}{\sqrt{T-1}} \xrightarrow{d} 0 \nonumber
\end{align}
Therefore,
\begin{align}
\left|\widehat{V}(\hat\pi_{T-1}) - V(\hat\pi_{T-1})\right| \xrightarrow{P} 0 \nonumber.
\end{align}
\end{proof}

\section{Proof of Theorem~\ref{thm:Asymptotic_Normality}}
\label{thm:Asymptotic_Normality_proof}

\begin{proof}
We begin by simplifying notation and introducing the following
quantity,
\begin{align}
\tilde\pi_{T,j}(\bx,a):= \frac{\hat\pi_1(\bx,a)}{T-1}+\sum_{t=2}^{j-1}\frac{\hat\pi_t(\bx,a)-\hat\pi_{t-1}(\bx,a)}{T-t}\label{def:tilde} 
\end{align}
We also define $f(T):= T^{\frac{1}{3(1+\eta)}}$. Here, the choice of
$\frac{1}{3(1+\eta)}$ is arbitrary and can be replaced with any
function of $T$ that satisfies $f(T)\rightarrow\infty$ and
$f(T) = o\left(T^{\frac{1}{2(1+\eta)}}\right)$. 
 
Using the definition of the crammed policy value estimator, we can write,
\begin{align}
  \widehat V(\hat\pi_{T-1}) - V(\hat\pi_{T-1})
  =& \sum_{t=1}^{T-1}\sum_{j=t+1}^T \frac{1}{T-t}\widehat\Gamma_{t,j} - V(\hat\pi_{T-1}) \nonumber\\
  =&\sum_{t=1}^{T-1}\sum_{j=t+1}^T \frac{1}{T-t}\widehat\Gamma_{t,j} - \left(\sum_{j=2}^T\frac{1}{T-1}V(\hat\pi_1) + \sum_{t=2}^{T-1}\sum_{j=t+1}^T \frac{1}{T-t}\Delta(\hat\pi_t;\hat\pi_{t-1})\right) \nonumber
\end{align}
For notational simplicity, we denote $V(\hat\pi_1)$ by
$\Delta(\hat\pi_1, \hat\pi_0)$ and therefore
\begin{align}
  &\widehat V(\hat\pi_{T-1}) - V(\hat\pi_{T-1})\nonumber\\
  =&\sum_{t=1}^{T-1}\sum_{j=t+1}^T \frac{1}{T-t}\left(\widehat\Gamma_{t,j}-\Delta(\hat\pi_{t};\hat\pi_{t-1})\right)\nonumber\\
  =&\sum_{j=2}^T \sum_{t=1}^{j-1}\frac{1}{T-t}\left(\widehat\Gamma_{t,j}
     - \Delta(\hat\pi_t, \hat\pi_{t-1})\right)\nonumber\\
  =&\sum_{j=2}^T
     \left\{\sum_{a=1}^K\left(\frac{\mathbb{I}(A_j=a)R_j}{\hat\pi_{j-1}(\bX_j,a)}\cdot\tilde\pi_{T,j}(\bX_j,a)-\mathbb{E}\left[\mu_a(\bX_j)\tilde\pi_{T,j}(\bX_j,a)\mid
     \bH_{j-1}\right]\right)\right\} \nonumber
\end{align}

divide the above term into the
following four parts:
\begin{align}
&\sum_{j=2}^T \sum_{t=1}^{j-1}\frac{1}{T-t}\left(\widehat\Gamma_{t,j} - \Delta(\hat\pi_t,\hat\pi_{t-1})\right)\nonumber\\
=&\sum_{j=2}^{f(T)}\sum_{t=1}^{j-1}\frac{1}{T-t}\left(\widehat\Gamma_{t,j} - \Delta(\hat\pi_t,\hat\pi_{t-1})\right)+\sum_{j=f(T)+1}^{T}\sum_{t=1}^{j-1}\frac{1}{T-t}\left(\widehat\Gamma_{t,j} - \Delta(\hat\pi_t,\hat\pi_{t-1})\right)\nonumber\\
=&\sum_{j=2}^{f(T)}\sum_{t=1}^{j-1}\frac{1}{T-t}\left(\widehat\Gamma_{t,j} - \Delta(\hat\pi_t,\hat\pi_{t-1})\right) \label{01201}\\
&+\sum_{j=f(T)+1}^{T}\sum_{a=1}^{K}\left(\frac{\mathbb{I}(A_j=a)R_j}{T-1}-\mathbb{E}\left[\frac{\mathbb{I}(A_j=a)R_j}{T-1} \ \Bigr |  \ \bH_{j-1}\right]\right)\label{01202}\\
&+\sum_{j=f(T)+1}^T \sum_{a=1}^K\frac{\mathbb{I}(A_j=a)R_j}{\hat\pi_{j-1}(\bX_j,a)}\cdot\left(\tilde\pi_{T,j}(\bX_j,a)-\frac{\hat\pi_{j-1}(\bX_j,a)}{T-1}\right)\label{01203}\\
&-\sum_{j=f(T)+1}^T\sum_{a=1}^K\mathbb{E}\left[\mu_a(\bX_j)\left(\tilde\pi_{T,j}(\bX_j,a)-\frac{\hat\pi_{j-1}(\bX_j,a)}{T-1}\right)\  \Bigr   | \ \bH_{j-1}\right].\label{01204}
\end{align}
Therefore, to prove Theorem~\ref{thm:Asymptotic_Normality}, it
suffices to verify the following four conditions:
\begin{itemize}
  \item There exists a a constant $c_1>0$ such that 
  \begin{equation}
    \lim_{T\rightarrow\infty}\mathbb{P}\left(v_T \ge c_1\right) = 0 \tag{C1} \label{cond:vt_lower_bound}
  \end{equation}
  \item The term in Equation~\eqref{01201} is 
    $o_p\left(\frac{1}{\sqrt{T}}\right)$ such that
  \begin{equation}
    \sqrt{T-1}\sum_{j=2}^{f(T)}\sum_{t=1}^{j-1}\frac{1}{T-t}\left(\widehat\Gamma_{t,j} - \Delta(\hat\pi_t,\hat\pi_{t-1})\right) \xrightarrow{p} 0 \tag{C2} \label{cond:upperleft}
  \end{equation}
  \item The terms in Equations~\eqref{01203}~and~\eqref{01204} are
    $o_p\left(\frac{1}{\sqrt{T}}\right)$ such that
  \begin{align}
    &\sqrt{T-1}\sum_{j=f(T)+1}^T \sum_{a=1}^K\frac{\mathbb{I}(A_j=a)R_j}{\hat\pi_{j-1}(\bX_j,a)}\cdot\left(\tilde\pi_{T,j}(\bX_j,a)-\frac{\hat\pi_{j-1}(\bX_j,a)}{T-1}\right) \nonumber\\
    &\ \ \ \ -
      \sqrt{T-1}\sum_{j=f(T)+1}^T\sum_{a=1}^K\mathbb{E}\left[\mu_a(\bX_j)\left(\tilde\pi_{T,j}(\bX_j,a)-\frac{\hat\pi_{j-1}(\bX_j,a)}{T-1}\right)\
      \Bigr |\ \bH_{j-1}\right] \xrightarrow{p} 0 \tag{C3} \label{cond:approx_error1}
  \end{align}

\item The term in Equation~\eqref{01202} is asymptotic normal after
  proper scaling such that
  \begin{equation}
    \frac{\sqrt{T-1}}{v_T}\sum_{j=f(T)+1}^{T-1}\sum_{a=1}^{K}\left(\frac{\mathbb{I}(A_j=a)R_j}{T-1}-\mathbb{E}\left[\frac{\mathbb{I}(A_j=a)R_j}{T-1}\
        \Bigr|\ \bH_{j-1}\right]\right) \xrightarrow{d} \mathcal{N}(0,1) \tag{C4} \label{cond:cond_normal}
  \end{equation}
\end{itemize}

Under Conditions~\ref{cond:vt_lower_bound}--\ref{cond:cond_normal}, we
can prove Theorem~\ref{thm:Asymptotic_Normality} by using the above
four-way decomposition,
\begin{align}
  &\sqrt{T-1}\cdot\frac{\widehat{V}(\hat\pi_{T-1}) - V(\hat\pi_{T-1})}{v_T}\nonumber\\
  =&\frac{\sqrt{T-1}}{v_T}\sum_{j=2}^T \sum_{t=1}^{j-1}\frac{1}{T-t}\left(\widehat\Gamma_{t,j} - \Delta(\hat\pi_t,\hat\pi_{t-1})\right)\nonumber\\
=&\frac{1}{v_T}\cdot\sqrt{T-1}\sum_{j=2}^{f(T)}\sum_{t=1}^{j-1}\frac{1}{T-t}\left(\widehat\Gamma_{t,j} - \Delta(\hat\pi_t,\hat\pi_{t-1})\right) \nonumber\\
&+\frac{\sqrt{T-1}}{v_T}\sum_{j=f(T)+1}^{T-1}\sum_{a=1}^{K}\left(\frac{\mathbb{I}(A_j=a)R_j}{T-1}-\mathbb{E}\left[\frac{\mathbb{I}(A_j=a)R_j}{T-1}\  \Bigr|\ \bH_{j-1}\right]\right)\nonumber\\
&+\frac{1}{v_T}\cdot\sqrt{T-1}\sum_{j=f(T)+1}^T \sum_{a=1}^K\frac{\mathbb{I}(A_j=a)R_j}{\hat\pi_{j-1}(\bX_j,a)}\cdot\left(\tilde\pi_{T,j}(\bX_j,a)-\frac{\hat\pi_{j-1}(\bX_j,a)}{T-1}\right)\nonumber\\
&-\frac{1}{v_T}\cdot\sqrt{T-1}\sum_{j=f(T)+1}^T\sum_{a=1}^K\mathbb{E}\left[\mu_a(\bX_j)\left(\tilde\pi_{T,j}(\bX_j,a)-\frac{\hat\pi_{j-1}(\bX_j,a)}{T-1}\right)\ \Bigr |\ \bH_{j-1}\right].
\end{align}
By Slutsky's theorem, the above expression converges in distribution
to the standard normal distribution under
Conditions~\ref{cond:vt_lower_bound}--\ref{cond:cond_normal}.  We now
prove these required conditions.

\subsection{Proof of Condition~\ref{cond:vt_lower_bound}}
By the definition of $v_T$, we have,
\begin{align}
v_T^2 &= (T-1)\sum_{j=2}^T \V(\widehat\Gamma_j(T)\mid \bH_{j-1}) \nonumber\\
&= (T-1)\sum_{j=2}^T  \mathbb{E}\left[\mathbb{E}\left[\sum_{a=1}^K\frac{\tilde\pi_{T,j}(\bX_j,a)^2\cdot (\mu_a(\bX_j)^2+\sigma_a^2(\bX_j))}{\hat\pi_{j-1}(\bX_j,a)}\ \Bigr|  \ \bH_{j-1}\right] - V(\tilde\pi_{T,j})^2\right] \nonumber\\
&= (T-1)\sum_{j=2}^T  \mathbb{E}\left[\mathbb{E}_{\bX}\left[\sum_{a=1}^K\frac{\tilde\pi_{T,j}(\bX,a)^2\cdot (\mu_a(\bX)^2+\sigma_a^2(\bX))}{\hat\pi_{j-1}(\bX,a)}\right] - V(\tilde\pi_{T,j})^2\right] \nonumber
\end{align}
where
\begin{align}
V(\tilde\pi_{T,j}):= \mathbb{E}_{\bX}\left[\sum_{a=1}^K \tilde\pi_{T,j}(\bX,a)\mu_a(\bX)\right] \nonumber.
\end{align}
By Jensen's inequality and Cauchy-Schwarz inequality, we have
\begin{align}
  V(\tilde\pi_{T,j})^2 &= \mathbb{E}_{\bX}\left[\sum_{a=1}^K \tilde\pi_{T,j}(\bX,a)\mu_a(\bX)\right]^2 \nonumber\\
  &\le \mathbb{E}_{\bX}\left[\left(\sum_{a=1}^K \frac{\tilde\pi_{T,j}(\bX,a)\mu_a(\bX)}{\sqrt{\hat\pi_{j-1}(\bX,a)}}\cdot \sqrt{\hat\pi_{j-1}(\bX,a)}\right)^2\right] \nonumber\\
  &\le \mathbb{E}_{\bX}\left[\left(\sum_{a=1}^K \frac{\tilde\pi_{T,j}(\bX,a)^2\mu_a(\bX)^2}{\hat\pi_{j-1}(\bX,a)}\right)\cdot \left(\sum_{a=1}^K \hat\pi_{j-1}(\bX,a)\right)\right] \nonumber\\
  &\le \mathbb{E}_{\bX}\left[\sum_{a=1}^K \frac{\tilde\pi_{T,j}(\bX,a)^2\mu_a(\bX)^2}{\hat\pi_{j-1}(\bX,a)}\right]. \nonumber
\end{align}
Therefore, using Cauchy-Schwartz inequality again and
Lemma~\ref{lemma:centering}, we have,
\begin{align*}
v_T^2 \ge &(T-1)\sum_{j=2}^T
            \mathbb{E}\left[\mathbb{E}_{\bX}\left[\sum_{a=1}^K\frac{\tilde\pi_{T,j}(\bX,a)^2\sigma_a(\bX^2)}{\hat\pi_{j-1}(\bX,a)}\right]\right]\nonumber\\
% \ge&(T-1)\sum_{j=2}^T
%      \mathbb{E}\left[\mathbb{E}\left[\sum_{a=1}^K\frac{\tilde\pi_{T,j}(\bX_j,a)^2\sigma_L^2}{\hat\pi_{j-1}(\bX_j,a)}\
%      \Bigr|\ \bH_{j-1}\right]\right]\nonumber\\
\ge&(T-1)\sigma_L^2\sum_{j=2}^T\sum_{a=1}^K \mathbb{E}_{\bX}\left[\tilde\pi_{T,j}(\bX,a)^2\right]\\
\ge&\sigma_L^2(T-1)\cdot\frac{1}{K}\sum_{j=2}^T \mathbb{E}_{\bX}\left[\left(\sum_{a=1}^K\tilde\pi_{T,j}(\bX,a)\right)^2\right] \nonumber\\
\ge&\sigma_L^2(T-1)\cdot\frac{1}{K}\sum_{j=2}^T \frac{1}{(T-1)^2} = \frac{\sigma_L^2}{K}
\end{align*}

\subsection{Proof of Condition~\eqref{cond:upperleft}}

Note
\begin{align*}
  &\sqrt{T-1}\sum_{j=2}^{f(T)}\sum_{t=1}^{j-1}\frac{1}{T-t}\left(\widehat\Gamma_{t,j} - \Delta(\hat\pi_t,\hat\pi_{t-1})\right)\nonumber\\
=&\sqrt{T-1}\sum_{j=2}^{f(T)}\sum_{t=1}^{j-1}\frac{1}{T-t}\sum_{a=1}^K\left(\frac{\mathbb{I}(A_j=a)R_j}{\hat\pi_{j-1}(\bX_j,a)}\cdot(\hat\pi_t(\bX_j,a)-\hat\pi_{t-1}(\bX_j,a))-\mathbb{E}_{\bX}\left[\mu_a(\bX)(\hat\pi_t(\bX,a)-\hat\pi_{t-1}(\bX,a))\right]\right)\nonumber\\
=&\sqrt{T-1}\sum_{j=2}^{f(T)}\sum_{a=1}^K\left(\frac{\mathbb{I}(A_j=a)R_j}{\hat\pi_{j-1}(\bX_j,a)}\tilde\pi_{T,j}(\bX_j,a)-\mathbb{E}_{\bX}\left[\mu_a(\bX)\tilde\pi_{T,j}(\bX_j,a)\right]\right)
\end{align*}
Therefore, using Lemma~\ref{lemma:centering},
\begin{align}
&\mathbb{E}\left[\left|\sqrt{T-1}\sum_{j=2}^{f(T)}\sum_{t=1}^{j-1}\frac{1}{T-t}\left(\widehat\Gamma_{t,j} - \Delta(\hat\pi_t,\hat\pi_{t-1})\right)\right|\right]\nonumber\\
\le&~2K\sqrt{T-1}\sum_{j=2}^{f(T)}\sup_{1\le a\le
     K}\mathbb{E}\left[\frac{\mathbb{I}(A_j=a)|R_j|}{\hat\pi_{j-1}(\bX_j,a)}\cdot\left|\tilde\pi_{T,j}(\bX_j,a)\right|\right]\nonumber\\
% \le&~2K\sqrt{T-1}\sum_{j=2}^{f(T)}\sup_{1\le a\le K}\mathbb{E}\left[\frac{\mathbb{I}(A_j=a)|R_j|}{\hat\pi_{j-1}(\bX_j,a)}\cdot\left|\sum_{t=1}^{j-1}\frac{\hat\pi_t(\bX_j,a)-\hat\pi_{t-1}(\bX_j,a)}{T-t}\right|\right]\nonumber\\
\le&~2K\sqrt{T-1}\sum_{j=2}^{f(T)}\frac{2K_4^{\frac{1}{4}}(j-1)^\eta}{T-j+1}\nonumber\\
\le&~\frac{4KK_4^{\frac{1}{4}}\sqrt{T-1} f(T)^{1+\eta}}{T-f(T)} \rightarrow 0 \nonumber.
\end{align} 
Here the last two inequalities uses the clipping rate assumption~\ref{ass:clip_rate} and the $4$-th moment condition~\ref{ass:fourth_moment}.
Since the $L_1$ convergence implies convergence in probability, we
have shown that Condition~\ref{cond:upperleft} holds.

\subsection{Proof of Condition~\ref{cond:approx_error1}}

Notice
\begin{align}
  &\sqrt{T-1}\sum_{j=f(T)+1}^T
    \sum_{a=1}^K\left(\frac{\mathbb{I}(A_j=a)R_j}{\hat\pi_{j-1}(\bX_j,a)}\cdot\left(\tilde\pi_{T,j}(\bX_j,a)-\frac{\hat\pi_{j-1}(\bX_j,a)}{T-1}\right)\right.
    \nonumber \\
 & \hspace{2in}  -\left.
    \mathbb{E}\left[\mu_a(\bX_j)\left(\tilde\pi_{T,j}(\bX_j,a)-\frac{\hat\pi_{j-1}(\bX_j,a)}{T-1}\right)\
    \Bigr|\ \bH_{j-1}\right]\right)\nonumber
\end{align}
is a martingale difference sequence and 
\begin{align*}
&\mathbb{E}\left[\left\{\sqrt{T-1}\sum_{j=f(T)+1}^T \sum_{a=1}^K\left(\frac{\mathbb{I}(A_j=a)R_j}{\hat\pi_{j-1}(\bX_j,a)}\cdot\left(\tilde\pi_{T,j}(\bX_j,a)-\frac{\hat\pi_{j-1}(\bX_j,a)}{T-1}\right) \right.\right.\right.\nonumber\\
&\ \ \ \ \left.\left.\left.- \mathbb{E}\left[\mu_a(\bX_j)\left(\tilde\pi_{T,j}(\bX_j,a)-\frac{\hat\pi_{j-1}(\bX_j,a)}{T-1}\right)\ \Bigr  |\ \bH_{j-1}\right]\right)\right\}^2\right]\nonumber\\
=&(T-1)\sum_{j=f(T)+1}^T\sum_{a=1}^K \mathbb{E}\left[\frac{\mu_a(\bX_j)^2+\sigma_a^2(\bX_j)}{\hat\pi_{j-1}(\bX_j,a)}\left(\tilde\pi_{T,j}(\bX_j,a)-\frac{\hat\pi_{j-1}(\bX_j,a)}{T-1}\right)^2\right]\nonumber\\
&-(T-1)\sum_{j=f(T)+1}^T  \mathbb{E}\left[\mathbb{E}\left[\sum_{a=1}^K\mu_a(\bX_j)\left(\tilde\pi_{T,j}(\bX_j,a)-\frac{\hat\pi_{j-1}(\bX_j,a)}{T-1}\right) \  \Bigr  |  \  \bH_{j-1}\right]^2\right]  \nonumber\\
\le&(T-1)\sum_{j=f(T)+1}^T\sum_{a=1}^K \mathbb{E}\left[\frac{\mu_a(\bX_j)^2+\sigma_a^2(\bX_j)}{\hat\pi_{j-1}(\bX_j,a)}\left(\tilde\pi_{T,j}(\bX_j,a)-\frac{\hat\pi_{j-1}(\bX_j,a)}{T-1}\right)^2\right]\nonumber\\
\le&(T-1)^{1+\eta}(\mu_U^2+\sigma_U^2)\sum_{j=f(T)+1}^T\sum_{a=1}^K \mathbb{E}\left[\left(\tilde\pi_{T,j}(\bX_j,a)-\frac{\hat\pi_{j-1}(\bX_j,a)}{T-1}\right)^2\right],
\end{align*}
where  the last inequality is due to
Assumptions~\ref{ass:clip_rate}~and~\ref{ass:bounded}. By
Lemma~\ref{lemma:L2_approx}, we have,
\begin{align*}
\lim_{T\rightarrow\infty}(T-1)^{1+\eta}(\mu_U^2+\sigma_U^2)\sum_{j=2}^T\sum_{a=1}^K \mathbb{E}\left[\left(\tilde\pi_{T,j}(\bX_j,a)-\frac{\hat\pi_{j-1}(\bX_j,a)}{T-1}\right)^2\right] = 0.
\end{align*}
Since convergence in $L_2$ implies convergence in probability, we have
shown that Condition~\ref{cond:approx_error1} holds.

\subsection{Proof of Condition~\ref{cond:cond_normal}}

For notational simplicity, we define the following,
\begin{align}
\Lambda_{T,j}&:=
               \sum_{a=1}^{K}\left(\frac{\mathbb{I}(A_j=a)R_j}{T-1}-\mathbb{E}\left[\frac{\mathbb{I}(A_j=a)R_j}{T-1}\
               \Bigr|  \ \bH_{j-1}\right]\right) \label{eq:LambdaTj}\\
\sigma_{T,j}^2&:= \E\left[\Lambda_{T,j}^2\mid \bH_{j-1}\right] \label{eq:SigmaTj}\\
S_T^2&:= \sum_{j=f(T)+1}^{T-1}\mathbb{E}\left[\sigma_{T,j}^2\mid \bH_{f(T)}\right] \label{eq:ST}
\end{align}
Therefore, $\{\Lambda_{T,j}\}_{j\ge f(T)+1}$ is a martingale difference sequence with respect to the filtration $\{\bH_{j-1}\}_{j\ge f(T)+1}$. 
\begin{lemma}
  \label{lemma:martingale_berry}
  For any $\epsilon >0$, 
  \begin{align}
      \mathbb{P}\left(\sup_w
    \biggr|\mathbb{P}\left(\sum_{j=f(T)+1}^{T}\Lambda_{T,j}\le S_T w \
    \Bigr| \ \bH_{f(T)}\right)-\Phi(w)\biggr|>\epsilon \right)\rightarrow 0 \text{ as } T\rightarrow\infty \label{eq:conditional_convergence}.
  \end{align}
  where $\{\Lambda_{T,j}\}_{j=f(T)+1}^{T}$ is defined in
  Equation~\eqref{eq:LambdaTj}, $\Phi(w)$ is the cumulative
  distribution function (CDF) of the standard normal distribution, and
  the probability is taken over the randomness of $\sigma(\bH_{f(T)})$
  (i.e., the first $f(T)$ observations).
\end{lemma}
Proof is given in  Appendix~\ref{lemma:martingale_berry_proof}.

\begin{lemma}
  \label{lemma:xi_consistency}
  For any $\epsilon>0$,
  \begin{align}
      \lim_{T\rightarrow\infty}\mathbb{P}\left( \left|{v_T^2} - (T-1)S_T^2\right| >\epsilon\right) = 0 \nonumber.
  \end{align}
  Furthermore,
  \begin{align}
  \frac{\sqrt{T-1}S_T}{v_T} \xrightarrow{p} 1 \text{ as } T\rightarrow\infty \nonumber.
  \end{align}
\end{lemma}
Proof is given in Appendix~\ref{lemma:xi_consistency_proof}.

By Lemma~\ref{lemma:martingale_berry}, we have,
$$
  \sup_w \left|\mathbb{P}\left(\sum_{j=f(T)+1}^{T}\Lambda_{T,j}\le
    S_T w \  \Bigr| \ \bH_{f(T)}\right)-\Phi(w)\right| \xrightarrow{P} 0 \nonumber \text{ as } T\rightarrow\infty,
$$
and
$$
  \sup_w \left|\mathbb{P}\left(\sum_{j=f(T)+1}^{T}\Lambda_{T,j}\le
    S_T w \ \Bigr | \ \bH_{
    f(T)}\right)-\Phi(w)\right| \le 2 \nonumber.
$$
By the dominated convergence theorem,
\begin{equation}
\lim_{T\rightarrow\infty} \mathbb{E}\left[\sup_w
  \left|\mathbb{P}\left(\sum_{j=f(T)+1}^{T}\Lambda_{T,j}\le S_T w \
      \Bigr| \ \bH_{
  f(T)}\right)-\Phi(w)\right|\right] = 0 \label{eq:051802}.
\end{equation}
Notice
\begin{align}
  \sup_w \left|\mathbb{P}\left(\sum_{j=f(T)+1}^{T}\Lambda_{T,j}\le S_T
  w \right)-\Phi(w)\right| &=
                             \sup_w\left|\mathbb{E}\left[\mathbb{P}\left(\sum_{j=f(T)+1}^{T}\Lambda_{T,j}\le
                             S_T w \ \Bigr | \ \bH_{
    f(T)}\right)-\Phi(w)\right]\right|\nonumber\\
  &\le \mathbb{E}\left[\sup_w
    \left|\mathbb{P}\left(\sum_{j=f(T)+1}^{T}\Lambda_{T,j}\le S_T w \
    \Bigr | \ \bH_{
    f(T)}\right)-\Phi(w)\right|\right]\nonumber.
\end{align}
Therefore, Equation~\eqref{eq:051802} implies 
\begin{equation*}
\lim_{T\rightarrow\infty} \sup_w \left|\mathbb{P}\left(\sum_{j=f(T)+1}^{T}\Lambda_{T,j}\le S_T w \right)-\Phi(w)\right| = 0 \nonumber,
\end{equation*}
or equivalently,
$$
\frac{\sum_{j=f(T)+1}^{T}\Lambda_{T,j}}{S_T} \xrightarrow{d} N(0,1) \nonumber.
$$
Because of Lemma~\ref{lemma:xi_consistency}, as $T\rightarrow\infty$,
$$
\frac{\sqrt{T-1}S_T}{v_T} \xrightarrow{p} 1 \nonumber.
$$
Therefore, using Slutsky's theorem, we can show that
Condition~\ref{cond:cond_normal} holds,
\begin{align*}
  \frac{\sqrt{T-1}}{v_T}\sum_{j=f(T)+1}^{T-1}\sum_{a=1}^{K}\left(\frac{\mathbb{I}(A_j=a)R_j}{T-1}-\mathbb{E}\left[\frac{\mathbb{I}(A_j=a)R_j}{T-1}\
  \Bigr| \ \bH_{j-1}\right]\right) \xrightarrow{d} \mathcal{N}(0,1).
\end{align*}
\end{proof}

\section{Proof of Theorem~\ref{thm:variance_estimator}}
\label{thm:variance_estimator_proof}

\begin{proof}
By the definition of $v_T^2$, we can write,
\begin{align*}
v_T^2 &= (T-1)\sum_{j=2}^T \V(\widehat\Gamma_j(T)\mid \bH_{j-1}) \nonumber\\
&=(T-1)\sum_{j=2}^T \V\left(\sum_{a=1}^K \frac{\mathbb{I}(A_j=a)R_j}{\hat\pi_{j-1}(\bX_j,a)}\tilde\pi_{T,j}(\bX_j,a)  \ \Bigr| \ \bH_{j-1}\right) \nonumber\\
&= (T-1)\sum_{j=2}^T\sum_{a=1}^K  \mathbb{E}\left[\frac{(\mu_a(\bX_j)^2+\sigma_a^2(\bX_j))}{\hat\pi_{j-1}(\bX_j,a)}\tilde\pi_{T,j}(\bX_j,a)^2\  \Bigr| \  \bH_{j-1}\right]
  \\
  & \hspace{1in} - (T-1)\sum_{j=2}^T\left(\sum_{a=1}^K
    \mathbb{E}\left[\mu_a(\bX_j)\tilde\pi_{T,j}(\bX_j,a)\mid \bH_{j-1}\right]\right)^2. \nonumber
\end{align*}
In addition, we have,
\begin{align}
\hat v_T^2 &= T\sum_{j=2}^{T-1} \hat v_{Tj}^2 \nonumber\\
&=(T-1)\sum_{j=2}^{T}\frac{1}{T-j+1}\sum_{k=j}^T(G_{Tjk} - \bar G_{Tj\cdot})^2 \nonumber\\
&= (T-1)\sum_{j=2}^{T} \frac{\sum_{k=j}^T G_{Tjk}^2}{T-j+1} - (T-1)\sum_{j=2}^{T-1} \bar G_{Tj\cdot}^2. \nonumber
\end{align}
Therefore,
\begin{align}
|v_T^2 - \hat v_T^2| \le&(T-1)\left|\sum_{j=2}^{T}\left(\sum_{a=1}^K
                          \mathbb{E}\left[\frac{(\mu_a(\bX_j)^2+\sigma_a^2(\bX_j))}{\hat\pi_{j-1}(\bX_j,a)}\tilde\pi_{T,j}(\bX_j,a)^2
                          \ \Bigr| \ \bH_{j-1}\right]-\frac{\sum_{k=j}^T G_{Tjk}^2}{T-j+1}\right)\right| \label{01241}\\
&+(T-1)\left|\sum_{j=2}^T\left(\sum_{a=1}^K \mathbb{E}\left[\mu_a(\bX_j)\tilde\pi_{T,j}(\bX_j,a)   \mid   \bH_{j-1}\right]\right)^2-\bar G_{Tj\cdot}^2\right|.\label{01242}
\end{align}
We next show both terms in Equations~\eqref{01241} and~\eqref{01242} converge
to 0 in probability as $T\rightarrow\infty$.

\subsection{Convergence of the term in Equation~\eqref{01241}}

According to the definition of $G_{Tjk}$, we have,
\begin{align*}
  G_{Tjk} \ : &=  \ \sum_{a=1}^K\frac{R_k\mathbb{I}(A_k=a)}{\hat\pi_{k-1}(\bX_k, a)}\tilde\pi_{T,j}(\bX_k,a)\nonumber\\
  G_{Tjk}^2 \ : &=  \ \sum_{a=1}^K\frac{R_k^2\mathbb{I}(A_k=a)}{\hat\pi_{k-1}(\bX_k, a)^2}\tilde\pi_{T,j}(\bX_k,a)^2\nonumber\\
  &=  \ \sum_{a=1}^K\frac{R_k^2\mathbb{I}(A_k=a)}{\hat\pi_{k-1}(\bX_k, a)^2}\left(\tilde\pi_{T,j}(\bX_k,a)-\frac{\hat\pi_{k-1}(\bX_k,a)}{T-1}\right)^2\nonumber\\
  &\ \ \ \ +\frac{1}{T-1}\sum_{a=1}^K\frac{R_k^2\mathbb{I}(A_k=a)}{\hat\pi_{k-1}(\bX_k, a)}\left(\tilde\pi_{T,j}(\bX_k,a)-\frac{\hat\pi_{k-1}(\bX_k,a)}{T-1}\right) +\frac{1}{(T-1)^2}\sum_{a=1}^K R_k^2\mathbb{I}(A_k=a).
\end{align*}
In addition, we have,
\begin{align}
  &~\sum_{a=1}^K
    \mathbb{E}\left[\frac{\mu_a(\bX_j)^2+\sigma_a^2(\bX_j)}{\hat\pi_{j-1}(\bX_j,a)}\tilde\pi_{T,j}(\bX_j,a)^2\
    \Bigr| \ \bH_{j-1}\right] \nonumber\\
  =&~\sum_{a=1}^K
     \mathbb{E}\left[\frac{\mu_a(\bX_j)^2+\sigma_a^2(\bX_j)}{\hat\pi_{j-1}(\bX_j,a)}\left(\tilde\pi_{T,j}(\bX_j,a)-\frac{\hat\pi_{j-1}(\bX_j,a)}{T-1}\right)^2\
     \Bigr| \ \bH_{j-1}\right]\nonumber\\
  &~+\sum_{a=1}^K
    \mathbb{E}\left[\frac{\mu_a(X_j)^2+\sigma_a^2(X_j)}{T-1}\left(\tilde\pi_{T,j}(\bX_j,a)-\frac{\hat\pi_{j-1}(\bX_j,a)}{T-1}\right) \ \Bigr| \ \bH_{j-1}\right]\nonumber\\
  &~+\frac{1}{(T-1)^2}\sum_{a=1}^K
    \mathbb{E}\left[R_j^2\mathbb{I}(A_j=a) \mid \bH_{j-1}\right].\nonumber
\end{align}
Therefore, 
\begin{align}
&~(T-1)\left|\sum_{j=2}^T\left(\sum_{a=1}^K
                \mathbb{E}\left[\frac{\mu_a(\bX_j)^2+\sigma_a^2(\bX_j)}{\hat\pi_{j-1}(\bX_j,a)}\tilde\pi_{T,j}(\bX_j,a)^2\
                \Bigr| \ \bH_{j-1}\right]-\frac{\sum_{k=j}^T G_{Tjk}^2}{T-j}\right)\right|\nonumber\\
\le&~(T-1)\left|\sum_{j=2}^T\sum_{a=1}^K \mathbb{E}\left[\frac{\mu_a(\bX_j)^2+\sigma_a^2(\bX_j)}{T-1}\left(\tilde\pi_{T,j}(\bX_j,a)-\frac{\hat\pi_{j-1}(\bX_j,a)}{T-1}\right)  \   \Bigr| \  \bH_{j-1}\right]\right| \label{01252}\\
&~+(T-1)\left|\sum_{j=2}^T\sum_{a=1}^K \mathbb{E}\left[\frac{\mu_a(\bX_j)^2+\sigma_a^2(\bX_j)}{\hat\pi_{j-1}(\bX_j,a)}\left(\tilde\pi_{T,j}(\bX_j,a)-\frac{\hat\pi_{j-1}(\bX_j,a)}{T-1}\right)^2\ \Bigr|   \  \bH_{j-1}\right]\right|\label{01253}\\
&~+(T-1)\left|\sum_{j=2}^T\frac{1}{T-j+1}\sum_{k=j}^T\frac{1}{T-1}\sum_{a=1}^K\frac{R_k^2\mathbb{I}(A_k=a)}{\hat\pi_{k-1}(\bX_k, a)}\left(\tilde\pi_{T,j}(\bX_k,a)-\frac{\hat\pi_{k-1}(\bX_k,a)}{T-1}\right)\right|\label{01254}\\
&~+(T-1)\left|\sum_{j=2}^T\sum_{k=j}^T\frac{1}{T-j+1}\sum_{a=1}^K\frac{R_k^2\mathbb{I}(A_k=a)}{\hat\pi_{k-1}(\bX_k, a)^2}\left(\tilde\pi_{T,j}(\bX_k,a)-\frac{\hat\pi_{k-1}(\bX_k,a)}{T-1}\right)^2\right|\label{01255}\\
&~+\left|\sum_{j=2}^T \frac{1}{T-1}\sum_{a=1}^K \mathbb{E}\left[R_j^2\mathbb{I}(A_j=a)\mid \bH_{j-1}\right] - \sum_{j=2}^T\frac{1}{T-1}\sum_{k=j}^T\frac{\sum_{a=1}^K R_k^2\mathbb{I}(A_k=a)}{T-j}\right| \label{01251}.
\end{align}
We now show that each term in the above equation converges to zero in
probability. 

\paragraph{Term in Equation~\eqref{01252} converges to zero in probability.} Notice 
\begin{align}
&~\mathbb{E}\left[(T-1)\left|\sum_{j=2}^T\sum_{a=1}^K
                \mathbb{E}\left[\frac{\mu_a(\bX_j)^2+\sigma_a^2(\bX_j)}{T-1}\left(\tilde\pi_{T,j}(\bX_j,a)-\frac{\hat\pi_{j-1}(\bX_j,a)}{T-1}\right) \ \Bigr| \ \bH_{j-1}\right]\right|\right]\nonumber\\
\le&~(\mu_U^2+\sigma_U^2)\sum_{j=2}^T\sum_{a=1}^K \mathbb{E}\left[\left|\tilde\pi_{T,j}(\bX_j,a)-\frac{\hat\pi_{j-1}(\bX_j,a)}{T-1}\right|\right]  \label{0126temp1}
\end{align}
By Lemma~\ref{lemma:L1_approx}, Equation~\eqref{0126temp1} converges
to zero as $T\rightarrow\infty$. As $L_1$ convergence implies
convergence in probability, the term in Equation~\eqref{01252}
converges to zero in probability as $T\rightarrow\infty$.

\paragraph{Term in Equation~\eqref{01253} converges to zero in probability.}
Similar to the above, we write,
\begin{align}
&~\mathbb{E}\left[(T-1)\left|\sum_{j=2}^T\sum_{a=1}^K
                \mathbb{E}\left[\frac{\mu_a(\bX_j)^2+\sigma_a^2(\bX_j)}{\hat\pi_{j-1}(\bX_j,a)}\left(\tilde\pi_{T,j}(\bX_j,a)-\frac{\hat\pi_{j-1}(\bX_j,a)}{T-1}\right)^2\
                \Bigr| \ \bH_{j-1}\right]\right|\right] \nonumber\\
\le&~(T-1)^{1+\eta}(\mu_U^2+\sigma_U^2)\sum_{j=2}^T \mathbb{E}\left[\sum_{a=1}^K\left(\tilde\pi_{T,j}(\bX_j,a)-\frac{\hat\pi_{j-1}(\bX_j,a)}{T-1}\right)^2\right] \label{0126temp2}
\end{align} 
By Lemma~\ref{lemma:L2_approx}, Equation~\eqref{0126temp2} converges
to zero as $T\rightarrow\infty$. As $L_1$ convergence implies
convergence in probability, the term in Equation~\eqref{01253}
converges to zero in probability as $T\rightarrow\infty$.

\paragraph{Term in Equation~\eqref{01254} converges to zero in probability.} Notice
\begin{align*}
  &~(T-1)\left|\sum_{j=2}^T\frac{1}{T-j+1}\sum_{k=j}^T\frac{1}{T-1}\sum_{a=1}^K\frac{R_k^2\mathbb{I}(A_k=a)}{\hat\pi_{k-1}(\bX_k, a)}\left(\tilde\pi_{T,j}(\bX_k,a)-\frac{\hat\pi_{k-1}(\bX_k,a)}{T-1}\right)\right|\nonumber\\
  \le&~\sum_{j=2}^T\frac{1}{T-j+1}\sum_{k=j}^T\sum_{a=1}^K\frac{R_k^2\mathbb{I}(A_k=a)}{\hat\pi_{k-1}(\bX_k, a)}\left|\tilde\pi_{T,j}(\bX_k,a)-\frac{\hat\pi_{k-1}(\bX_k,a)}{T-1}\right|.
\end{align*}
Taking the expectation of this bound yields,
\begin{align*}
&~\mathbb{E}\left[\sum_{j=2}^T\frac{1}{T-j+1}\sum_{k=j}^T\sum_{a=1}^K\frac{R_k^2\mathbb{I}(A_k=a)}{\hat\pi_{k-1}(\bX_k, a)}\left|\tilde\pi_{T,j}(\bX_k,a)-\frac{\hat\pi_{k-1}(\bX_k,a)}{T-1}\right|\right]\nonumber\\
=&~\mathbb{E}\left[\sum_{j=2}^T\frac{1}{T-j+1}\sum_{k=j}^T\sum_{a=1}^K\mathbb{E}\left[\frac{R_k^2\mathbb{I}(A_k=a)}{\hat\pi_{k-1}(\bX_k,
   a)}\left|\tilde\pi_{T,j}(\bX_k,a)-\frac{\hat\pi_{k-1}(\bX_k,a)}{T-1}\right|\
   \Bigr| \ \bH_{k-1}\right]\right]\nonumber\\
=&~\mathbb{E}\left[\sum_{j=2}^T\frac{1}{T-j+1}\sum_{k=j}^T\sum_{a=1}^K\mathbb{E}\left[(\mu_a(\bX_k)^2+\sigma_a^2(\bX_k))\cdot\left|\tilde\pi_{T,j}(\bX_k,a)-\frac{\hat\pi_{k-1}(\bX_k,a)}{T-1}\right|\
   \Bigr| \ \bH_{k-1}\right]\right]\nonumber\\
\le&~(\mu_U^2+\sigma_U^2)\sum_{j=2}^T\frac{1}{T-j+1}\sum_{k=j}^T\sum_{a=1}^K\mathbb{E}\left[\left|\tilde\pi_{T,j}(\bX_k,a)-\frac{\hat\pi_{k-1}(\bX_k,a)}{T-1}\right|\right]\nonumber
\end{align*}
Using the triangular inequality by subtracting and adding
$\frac{\hat\pi_{j-1}(\bX_j,a)}{T-1}$ from the absolute value term, we
have,
\begin{align*}
  &~(\mu_U^2+\sigma_U^2)\sum_{j=2}^T\frac{1}{T-j+1}\sum_{k=j}^T\sum_{a=1}^K\mathbb{E}\left[\left|\tilde\pi_{T,j}(\bX_k,a)-\frac{\hat\pi_{k-1}(\bX_k,a)}{T-1}\right|\right]\nonumber\\
\le&~(\mu_U^2+\sigma_U^2)\sum_{j=2}^T\frac{1}{T-j+1}\sum_{k=j}^T\sum_{a=1}^K\mathbb{E}\left[\left|\tilde\pi_{T,j}(\bX_k,a)-\frac{\hat\pi_{j-1}(\bX_k,a)}{T-1}\right|\right]\nonumber\\ 
&~+(\mu_U^2+\sigma_U^2)\sum_{j=2}^T\frac{1}{T-j+1}\sum_{k=j}^T\sum_{a=1}^K\mathbb{E}\left[\left|\frac{\hat\pi_{j-1}(\bX_j,a)-\hat\pi_{k-1}(\bX_k,a)}{T-1}\right|\right]\nonumber\\
\le&~(\mu_U^2+\sigma_U^2)\sum_{j=2}^T\sum_{a=1}^K\mathbb{E}\left[\left|\tilde\pi_{T,j}(\bX,a)-\frac{\hat\pi_{j-1}(\bX,a)}{T-1}\right|\right]+\frac{(\mu_U^2+\sigma_U^2)}{T-1}\sum_{j=2}^T\frac{1}{T-j+1}\sum_{k=j}^T\sum_{a=1}^K\sum_{i_1=j}^{k-1}\mathbb{E}\left[Q_{i_1}\right]\nonumber\\
\le&~2(\mu_U^2+\sigma_U^2)\sum_{j=2}^T\sum_{a=1}^K\mathbb{E}\left[\left|\tilde\pi_{T,j}(\bX,a)-\frac{\hat\pi_{j-1}(\bX,a)}{T-1}\right|\right]+\frac{K(\mu_U^2+\sigma_U^2)\cdot\sup_{i}\mathbb{E}\left[Q_ii^{1+\delta}\right]}{(T-1)\delta}\sum_{j=2}^T\sum_{k=j}^T\frac{1}{(T-j+1)j^{\delta}}\nonumber
\end{align*}
By Lemma~\ref{lemma:L1_approx},
$$2(\mu_U^2+\sigma_U^2)\sum_{j=2}^T\sum_{a=1}^K\mathbb{E}\left[\left|\tilde\pi_{T,j}(\bX_k,a)-\frac{\hat\pi_{j-1}(\bX_k,a)}{T-1}\right|\right] 
\rightarrow 0.$$ Algebraic calculation shows
$\frac{K(\mu_U^2+\sigma_U^2)\cdot\sup_{i}\mathbb{E}\left[Q_ii^{1+\delta}\right]}{(T-1)\delta}\sum_{j=2}^T\sum_{k=j}^T\frac{1}{(T-j+1)j^{\delta}}$
is $O(\frac{1}{T})$. Therefore, the term in Equation~\eqref{01254}
converges to zero in probability as $T\rightarrow\infty$.

\paragraph{Term in Equation~\eqref{01255} converges to zero in probability.} Notice
\begin{align}
  &~\mathbb{E}\left[(T-1)\left|\sum_{j=2}^T\sum_{k=j}^T\frac{1}{T-j+1}\sum_{a=1}^K\frac{R_k^2\mathbb{I}(A_k=a)}{\hat\pi_{k-1}(\bX_k, a)^2}\left(\tilde\pi_{T,j}(\bX_k,a)-\frac{\hat\pi_{k-1}(\bX_k,a)}{T-1}\right)^2\right|\right] \nonumber\\
  =&~(T-1)\mathbb{E}\left[\sum_{j=2}^T\sum_{k=j}^T\frac{1}{T-j+1}\sum_{a=1}^K\mathbb{E}\left[\frac{R_k^2\mathbb{I}(A_k=a)}{\hat\pi_{k-1}(\bX_k,
     a)^2}\left(\tilde\pi_{T,j}(\bX_k,a)-\frac{\hat\pi_{k-1}(\bX_k,a)}{T-1}\right)^2\
     \Bigr| \ \bH_{k-1}\right]\right]\nonumber\\
  =&~(T-1)\mathbb{E}\left[\sum_{j=2}^T\sum_{k=j}^T\frac{1}{T-j+1}\sum_{a=1}^K\mathbb{E}\left[\frac{\mu_a(\bX_k)^2+\sigma_a^2(\bX_k)}{\hat\pi_{k-1}(\bX_k,
     a)}\left(\tilde\pi_{T,j}(\bX_k,a)-\frac{\hat\pi_{k-1}(\bX_k,a)}{T-1}\right)^2\
     \Bigr| \ \bH_{k-1}\right]\right]\nonumber\\
  \le&~  (T-1)^{1+\eta}(\mu_U^2+\sigma_U^2)\sum_{j=2}^T\sum_{k=j}^T\frac{1}{T-j+1}\sum_{a=1}^K\mathbb{E}\left[\left(\tilde\pi_{T,j}(\bX_k,a)-\frac{\hat\pi_{k-1}(\bX_k,a)}{T-1}\right)^2\right]\nonumber
\end{align}
Rewriting the term $\left(\tilde\pi_{T,j}(\bX_k,a)-\frac{\hat\pi_{k-1}(\bX_k,a)}{T-1}\right)^2$ as
$\left\{\left(\tilde\pi_{T,j}(\bX_k,a)-\frac{\hat\pi_{j-1}(\bX_k,a)}{T-1}\right)+\left(\frac{\hat\pi_{j-1}(\bX_k,a)}{T-1}-\frac{\hat\pi_{k-1}(\bX_k,a)}{T-1}\right)\right\}^2$
and expanding it out yield,
\begin{align}
  &~\mathbb{E}\left[(T-1)\left|\sum_{j=2}^T\sum_{k=j}^T\frac{1}{T-j+1}\sum_{a=1}^K\frac{R_k^2\mathbb{I}(A_k=a)}{\hat\pi_{k-1}(\bX_k, a)^2}\left(\tilde\pi_{T,j}(\bX_k,a)-\frac{\hat\pi_{k-1}(\bX_k,a)}{T-1}\right)^2\right|\right] \nonumber\\
  \le&~  (T-1)^{1+\eta}(\mu_U^2+\sigma_U^2)\sum_{j=2}^T\sum_{a=1}^K\mathbb{E}\left[\left(\tilde\pi_{T,j}(\bX,a)-\frac{\hat\pi_{j-1}(\bX,a)}{T-1}\right)^2\right]\label{0126temp3}\\
  &~+2(T-1)^{1+\eta}(\mu_U^2+\sigma_U^2)\sum_{j=2}^T\sum_{k=j}^T\frac{1}{T-j+1}\sum_{a=1}^K\mathbb{E}\left[\left|\tilde\pi_{T,j}(\bX_k,a)-\frac{\hat\pi_{j-1}(\bX_k,a)}{T-1}\right|\cdot\frac{|\hat\pi_{j-1}(\bX_j,a)-\hat\pi_{k-1}(\bX_j,a)|}{T-1}\right]\label{0126temp4}\\
  &~+(T-1)^{1+\eta}(\mu_U^2+\sigma_U^2)\sum_{j=2}^T\sum_{k=j}^T\frac{1}{T-j+1}\sum_{a=1}^K\mathbb{E}\left[\frac{|\hat\pi_{j-1}(\bX_j,a)-\hat\pi_{k-1}(\bX_j,a)|}{(T-1)^2}\right].\label{0126temp5}
\end{align}

Lemma~\ref{lemma:L2_approx} implies that the term in
Equation~\eqref{0126temp3} converges to zero.  For the term in 
Equation~\eqref{0126temp4}, notice
\begin{align}
&~2(T-1)^{1+\eta}(\mu_U^2+\sigma_U^2)\sum_{j=2}^T\sum_{k=j}^T\frac{1}{T-j+1}\sum_{a=1}^K\mathbb{E}\left[\left|\tilde\pi_{T,j}(\bX_k,a)-\frac{\hat\pi_{j-1}(\bX_k,a)}{T-1}\right|\cdot\frac{|\hat\pi_{j-1}(\bX_j,a)-\hat\pi_{k-1}(\bX_j,a)|}{T-1}\right]\nonumber\\
\le&~4(T-1)^\eta(\mu_U^2+\sigma_U^2)\sum_{j=2}^T\sum_{k=j}^T\frac{1}{T-j+1}\sum_{a=1}^K\mathbb{E}\left[\left|\tilde\pi_{T,j}(\bX_k,a)-\frac{\hat\pi_{j-1}(\bX_k,a)}{T-1}\right|\right]\nonumber\\
=&~4(T-1)^\eta(\mu_U^2+\sigma_U^2)\sum_{j=2}^T\sum_{a=1}^K\mathbb{E}\left[\left|\tilde\pi_{T,j}(\bX,a)-\frac{\hat\pi_{j-1}(\bX,a)}{T-1}\right|\right].\nonumber
\end{align}
Lemma~\ref{lemma:L1_approx} implies,
$$\sum_{j=2}^T\sum_{a=1}^K\mathbb{E}\left[\left|\tilde\pi_{T,j}(\bX_k,a)-\frac{\hat\pi_{j-1}(\bX_k,a)}{T-1}\right|\right]
= O\left(\frac{\log (T-1)}{(T-1)^{\min(1,\delta)}}\right).$$
By Assumption~\ref{ass:clip_rate}, we know $\eta<\frac{\delta}{1+\delta}$ and therefore, 
$$
4(T-1)^\eta(\mu_U^2+\sigma_U^2)\sum_{j=2}^T\sum_{a=1}^K\mathbb{E}\left[\left|\tilde\pi_{T,j}(\bX_k,a)-\frac{\hat\pi_{j-1}(\bX_k,a)}{T-1}\right|\right] \rightarrow 0.
$$
This shows that the term in Equation~\eqref{0126temp4} converges to zero.

Lastly, for the term in Equation~\eqref{0126temp5}, we have,
\begin{align*}
&~(T-1)^{1+\eta}(\mu_U^2+\sigma_U^2)\sum_{j=2}^T\sum_{k=j}^T\frac{1}{T-j+1}\sum_{a=1}^K\mathbb{E}\left[\frac{|\hat\pi_{j-1}(\bX_j,a)-\hat\pi_{k-1}(\bX_j,a)|}{(T-1)^2}\right] \nonumber\\
\le&~(T-1)^{-1+\eta}(\mu_U^2+\sigma_U^2)K\sum_{j=2}^T\sum_{k=j}^T\frac{1}{T-j+1}\sum_{i=j}^{k-1}\mathbb{E}[Q_i] \nonumber\\
\le&~(T-1)^{-1+\eta}(\mu_U^2+\sigma_U^2)K\cdot\sup_i \mathbb{E}\left[Q_ii^{1+\delta}\right]\sum_{j=2}^T\sum_{k=j}^T\frac{1}{T-j+1}\sum_{i=j}^{k-1}\frac{1}{i^{1+\delta}}\nonumber\\
\le&~(T-1)^{-1+\eta}(\mu_U^2+\sigma_U^2)K\cdot\sup_i \mathbb{E}\left[Q_ii^{1+\delta}\right]\sum_{j=2}^T\frac{1}{\delta j^{\delta}} = O\left(\frac{(T-1)^{1-\delta}}{T^{1-\eta}}\right) = o(1).
\end{align*}
Therefore, all terms in Equations~\eqref{0126temp3},~\eqref{0126temp4}
and~\eqref{0126temp5} converge to zero as $T\rightarrow\infty$. This
shows that the term in Equation~\eqref{01255} converges to zero in
probability.

\paragraph{Term in Equation~\eqref{01251} converges to zero in probability.}
We prove a stronger $L_1$ convergence.
\begin{align}
&~\mathbb{E}\left[\left|\sum_{j=2}^T \frac{1}{T-1}\sum_{a=1}^K \mathbb{E}\left[R_j^2\mathbb{I}(A_j=a)\mid \bH_{j-1}\right] - \sum_{j=2}^T\frac{1}{T-1}\sum_{k=j}^T\frac{\sum_{a=1}^K R_k^2\mathbb{I}(A_k=a)}{T-j+1}\right|\right] \nonumber\\
\le&~\frac{1}{T-1}\sum_{j=2}^T\sum_{a=1}^K \mathbb{E}\left[\left|\sum_{k=j}^T\frac{ R_k^2\mathbb{I}(A_k=a)}{T-j+1}-\mathbb{E}\left[R_j^2\mathbb{I}(A_j=a)\mid \bH_{j-1}\right]\right|\right]\nonumber\\
\le&~\frac{1}{T-1}\sum_{j=2}^T\sum_{a=1}^K \mathbb{E}\left[\left|\sum_{k=j}^T\frac{ R_k^2\mathbb{I}(A_k=a)-\mathbb{E}\left[R_k^2\mathbb{I}(A_k=a)\mid \bH_{k-1}\right]}{T-j+1}\right|\right]\label{0126temp6}\\
&~+\frac{1}{T-1}\sum_{j=2}^T\sum_{a=1}^K \mathbb{E}\left[\left|\sum_{k=j}^T \frac{\mathbb{E}\left[R_k^2\mathbb{I}(A_k=a)\mid \bH_{k-1}\right]-\mathbb{E}\left[R_j^2\mathbb{I}(A_j=a)\mid \bH_{j-1}\right]}{T-j+1}\right|\right]\label{0126temp7}
\end{align}
Notice by Jensen's inequality, for any $j$,
\begin{align}
  &~\mathbb{E}\left[\left|\sum_{k=j}^T\frac{ R_k^2\mathbb{I}(A_k=a)-\mathbb{E}\left[R_k^2\mathbb{I}(A_k=a)\mid \bH_{k-1}\right]}{T-j+1}\right|\right] \nonumber\\
  \le&~\mathbb{E}\left[\left|\sum_{k=j}^T\frac{ R_k^2\mathbb{I}(A_k=a)-\mathbb{E}\left[R_k^2\mathbb{I}(A_k=a)\mid \bH_{k-1}\right]}{T-j+1}\right|^2\right]^{\frac{1}{2}} \nonumber\\
  \le&~\mathbb{E}\left[\sum_{k=j}^T\frac{ \left(R_k^2\mathbb{I}(A_k=a)-\mathbb{E}\left[R_k^2\mathbb{I}(A_k=a)\mid \bH_{k-1}\right]\right)^2}{(T-j+1)^2}\right]^{\frac{1}{2}}, \nonumber
\end{align}
where the last inequality uses the martingale difference
property. Using this inequality, the term in
Equation~\eqref{0126temp6} can be bounded as follows,
\begin{align*}
&~\frac{1}{T-1}\sum_{j=2}^T\sum_{a=1}^K \mathbb{E}\left[\left|\sum_{k=j}^T\frac{ R_k^2\mathbb{I}(A_k=a)-\mathbb{E}\left[R_k^2\mathbb{I}(A_k=a)\mid \bH_{k-1}\right]}{T-j+1}\right|\right]\nonumber\\
\le&~\mathbb{E}\left[\sum_{k=j}^T\frac{ \left(R_k^2\mathbb{I}(A_k=a)-\mathbb{E}\left[R_k^2\mathbb{I}(A_k=a)\mid \bH_{k-1}\right]\right)^2}{(T-j+1)^2}\right]^{\frac{1}{2}} \nonumber\\
=&~\frac{1}{T-1}\sum_{j=2}^T\sum_{a=1}^K \mathbb{E}\left[\sum_{k=j}^T\frac{ \mathbb{E}\left[\left(R_k^2\mathbb{I}(A_k=a)-\mathbb{E}\left[R_k^2\mathbb{I}(A_k=a)\mid \bH_{k-1}\right]\right)^2\mid \bH_{k-1}\right]}{(T-j+1)^2}\right]^\frac{1}{2} \nonumber\\
\le&~\frac{2}{T-1}\sum_{j=2}^T\sum_{a=1}^K \mathbb{E}\left[\sum_{k=j}^T\frac{ \mathbb{E}\left[\left(R_k^2\mathbb{I}(A_k=a)\right)^2\mid \bH_{k-1}\right]}{(T-j+1)^2}\right]^\frac{1}{2} \nonumber\\
\le&~\frac{2}{T-1}\sum_{j=2}^T\sum_{a=1}^K \mathbb{E}\left[\sum_{k=j}^T\frac{ \mathbb{E}\left[R_k^4\right]}{(T-j+1)^2}\right]^\frac{1}{2} \nonumber\\
\le&~\frac{2\sqrt{K_4}K}{T-1}\sum_{j=2}^T\frac{1}{\sqrt{T-j+1}} = o(1) 
\end{align*}

Similarly, the term in Equation~\eqref{0126temp7} can be bounded as follows,
\begin{align}
&~\frac{1}{T-1}\sum_{j=2}^T\sum_{a=1}^K \mathbb{E}\left[\left|\sum_{k=j}^T \frac{\mathbb{E}\left[R_k^2\mathbb{I}(A_k=a)\mid \bH_{k-1}\right]-\mathbb{E}[R_j^2\mathbb{I}(A_j=a)\mid \bH_{j-1}]}{T-j+1}\right|\right] \nonumber\\
\le&~\frac{1}{T-1}\sum_{j=2}^T\sum_{a=1}^K \mathbb{E}\left[\left|\sum_{k=j}^T \frac{\mathbb{E}_{\bX}\left[(\sigma_a^2(\bX)+\mu_a(\bX)^2)\cdot(\hat\pi_{k-1}(\bX,a)-\hat\pi_{j-1}(\bX,a))\right]}{T-j+1}\right|\right]\nonumber\\
\le&~\frac{\mu_U^2+\sigma_U^2}{T-1}\sum_{j=2}^T\sum_{a=1}^K \sum_{k=j}^T \frac{\mathbb{E}\left[|\hat\pi_{k-1}(\bX,a)-\hat\pi_{j-1}(\bX,a)|\right]}{T-j+1}\nonumber\\
\le&~\frac{K(\mu_U^2+\sigma_U^2)}{T-1}\sum_{j=2}^T\sum_{k=j}^T \frac{\sum_{i=j}^{k-1}\mathbb{E}\left[Q_i\right]}{T-j+1}\nonumber\\
\le&~\frac{K(\mu_U^2+\sigma_U^2)\sup_i \mathbb{E}\left[Q_ii^{1+\delta}\right]}{(T-1)\delta}\sum_{j=2}^T\sum_{k=j}^T \frac{1}{(T-j+1)j^\delta}\nonumber\\
\le&~\frac{K(\mu_U^2+\sigma_U^2)\sup_i \mathbb{E}\left[Q_ii^{1+\delta}\right]}{(T-1)\delta}\sum_{j=2}^T\frac{1}{j^\delta} = o(1) \nonumber.
\end{align}
This shows that the term in Equation~\eqref{01251} converges to zero in probability.

\subsection{Convergence of the term in Equation~\eqref{01242}}

We begin by bounding this term as,
\begin{align*}
&(T-1)\left|\sum_{j=2}^T\left(\sum_{a=1}^K \mathbb{E}\left[\mu_a(\bX_j)\tilde\pi_{T,j}(\bX_j,a)\mid \bH_{j-1}\right]\right)^2-\bar G_{Tj\cdot}^2\right| \nonumber\\
\le&(T-1)\sum_{j=2}^T\left|\sum_{a=1}^K \left(\mathbb{E}\left[\mu_a(\bX_j)\tilde\pi_{T,j}(\bX_j,a)\mid \bH_{j-1}\right]-\frac{1}{T-j+1}\sum_{k=j}^T\frac{R_k\mathbb{I}(A_k=a)}{\hat\pi_{k-1}(\bX_k, a)}\tilde\pi_{T,j}(\bX_k,a)\right)\right| \nonumber\\
&\cdot \left|\sum_{a=1}^K \left(\mathbb{E}\left[\mu_a(\bX_j)\tilde\pi_{T,j}(\bX_j,a)\mid \bH_{j-1}\right]+\frac{1}{T-j+1}\sum_{k=j}^T\frac{R_k\mathbb{I}(A_k=a)}{\hat\pi_{k-1}(\bX_k, a)}\tilde\pi_{T,j}(\bX_k,a)\right)\right|.
\end{align*}
Therefore, taking the expectation of this upper bound yields,
\begin{align}
 &\mathbb{E} \left[(T-1)\left|\sum_{j=2}^T\left(\sum_{a=1}^K
   \mathbb{E}\left[\mu_a(X_j)\tilde\pi_{T,j}(X_j,a)\mid \bH_{j-1}\right]\right)^2-\bar G_{Tj\cdot}^2\right|\right] \nonumber\\
\le&(T-1)\sum_{j=2}^T\mathbb{E}\left[\left|\sum_{a=1}^K
     \left(\mathbb{E}\left[\mu_a(\bX_j)\tilde\pi_{T,j}(\bX_j,a)\mid \bH_{j-1}\right]-\frac{1}{T-j+1}\sum_{k=j}^T\frac{R_k\mathbb{I}(A_k=a)}{\hat\pi_{k-1}(\bX_k, a)}\tilde\pi_{T,j}(\bX_k,a)\right)\right|\right. \nonumber\\
  &\hspace{1in}\cdot \left.\left|\sum_{a=1}^K \left(\mathbb{E}\left[\mu_a(\bX_j)\tilde\pi_{T,j}(\bX_j,a)\mid \bH_{j-1}\right]+\frac{1}{T-j+1}\sum_{k=j}^T\frac{R_k\mathbb{I}(A_k=a)}{\hat\pi_{k-1}(\bX_k, a)}\tilde\pi_{T,j}(\bX_k,a)\right)\right|\right]\nonumber\\
\le&(T-1)\sum_{j=2}^T
     \mathbb{E}\left[\left\{\frac{1}{T-j+1}\sum_{k=j}^T\sum_{a=1}^K\left(\frac{R_k\mathbb{I}(A_k=a)}{\hat\pi_{k-1}(\bX_k,
     a)}\tilde\pi_{T,j}(\bX_k,a)
     -\mathbb{E}\left[\frac{R_k\mathbb{I}(A_k=a)}{\hat\pi_{k-1}(\bX_k,
     a)}\tilde\pi_{T,j}(\bX_k,a)\ \Bigr | \ \bH_{k-1}\right]\right)\right\}^2\right]^{\frac{1}{2}}\nonumber\\
&\cdot\mathbb{E}\left[\left\{\frac{1}{T-j+1}\sum_{k=j}^T\sum_{a=1}^K\left(\frac{R_k\mathbb{I}(A_k=a)}{\hat\pi_{k-1}(\bX_k, a)}\tilde\pi_{T,j}(\bX_k,a)  +\mathbb{E}\left[\frac{R_k\mathbb{I}(A_k=a)}{\hat\pi_{k-1}(\bX_k,  a)}\tilde\pi_{T,j}(\bX_k,a)\  \Bigr|\  \bH_{k-1}\right]\right)\right\}^2\right]^{\frac{1}{2}}\label{0127temp0}
\end{align}
Notice
$\left\{\sum_{a=1}^K\left(\frac{R_k\mathbb{I}(A_k=a)}{\hat\pi_{k-1}(\bX_k,
      a)}\tilde\pi_{T,j}(\bX_k,a)
    -\mathbb{E}\left[\frac{R_k\mathbb{I}(A_k=a)}{\hat\pi_{k-1}(\bX_k,
        a)}\tilde\pi_{T,j}(\bX_k,a)\mid \bH_{k-1}\right]\right)\right\}_{k=j}^T$ is a martingale difference sequence. Therefore,
\begin{align}
&\mathbb{E}\left[\left\{\frac{1}{T-j+1}\sum_{k=j}^T\sum_{a=1}^K\left(\frac{R_k\mathbb{I}(A_k=a)}{\hat\pi_{k-1}(\bX_k,
                a)}\tilde\pi_{T,j}(\bX_k,a)
                -\mathbb{E}\left[\frac{R_k\mathbb{I}(A_k=a)}{\hat\pi_{k-1}(\bX_k,
                a)}\tilde\pi_{T,j}(\bX_k,a)\  \Bigr| \ \bH_{k-1}\right]\right)\right\}^2\right]\nonumber\\
=&\mathbb{E}\left[\frac{1}{(T-j+1)^2}\sum_{k=j}^T\left\{\sum_{a=1}^K\left(\frac{R_k\mathbb{I}(A_k=a)}{\hat\pi_{k-1}(\bX_k, a)}\tilde\pi_{T,j}(\bX_k,a)
   -\mathbb{E}\left[\frac{R_k\mathbb{I}(A_k=a)}{\hat\pi_{k-1}(\bX_k,
   a)}\tilde\pi_{T,j}(\bX_k,a)\ \Bigr| \ \bH_{k-1}\right]\right)\right\}^2\right]\nonumber\\
\le&\frac{4}{(T-j+1)^2}\sum_{k=j}^T\mathbb{E}\left[\left(\sum_{a=1}^K\frac{R_k\mathbb{I}(A_k=a)}{\hat\pi_{k-1}(\bX_k, a)}\tilde\pi_{T,j}(\bX_k,a)\right)^2\right]\nonumber\\
\le&\frac{4}{(T-j+1)^2}\sum_{k=j}^T\sum_{a=1}^K\mathbb{E}\left[\frac{\mu_a(\bX)^2+\sigma_a^2(\bX)}{\hat\pi_{k-1}(\bX, a)}\tilde\pi_{T,j}(\bX,a)^2\right]\nonumber\\
\le&\frac{4(\mu_U^2+\sigma_U^2)(T-1)^\eta}{(T-j+1)^2}\sum_{k=j}^T\sum_{a=1}^K\mathbb{E}\left[\tilde\pi_{T,j}(\bX,a)^2\right]\nonumber\\
\le&\frac{4K(\mu_U^2+\sigma_U^2)(T-1)^\eta}{T-j+1}\sup_{a}\mathbb{E}\left[\tilde\pi_{T,j}(\bX,a)^2\right].\label{0127temp3}
\end{align} 
For the second term in Equation~\eqref{0127temp0}, we have,
\begin{align}
  &\mathbb{E}\left[\left\{\frac{1}{T-j+1}\sum_{k=j}^T\sum_{a=1}^K\left(\frac{R_k\mathbb{I}(A_k=a)}{\hat\pi_{k-1}(\bX_k,a)}\tilde\pi_{T,j}(\bX_k,a)
    +\mathbb{E}\left[\frac{R_k\mathbb{I}(A_k=a)}{\hat\pi_{k-1}(\bX_k,
    a)}\tilde\pi_{T,j}(\bX_k,a)\ \Bigr | \  \bH_{k-1}\right]\right)\right\}^2\right]\nonumber\\
  \le&2\mathbb{E}\left[\frac{1}{(T-j+1)^2}\left\{\sum_{k=j}^T\sum_{a=1}^K\left(\frac{R_k\mathbb{I}(A_k=a)}{\hat\pi_{k-1}(\bX_k, a)}\tilde\pi_{T,j}(\bX_k,a)\right)\right\}^2\right]\nonumber\\
  &+2\mathbb{E}\left[\frac{1}{(T-j+1)^2}\left\{\sum_{k=j}^T\sum_{a=1}^K\left(\mathbb{E}\left[\frac{R_k\mathbb{I}(A_k=a)}{\hat\pi_{k-1}(\bX_k,
    a)}\tilde\pi_{T,j}(\bX_k,a)\ \Bigr| \ \bH_{k-1}\right]\right)\right\}^2\right]\nonumber\\
  \le&2\mathbb{E}\left[\frac{1}{(T-j+1)^2}\sum_{k_1=j}^T\sum_{k_2=j}^T\sum_{a_1=1}^K\sum_{a_2=1}^K\frac{R_{k_1}R_{k_2}\mathbb{I}(A_{k_1}=a_1)\mathbb{I}(A_{k_2}=a_2)}{\hat\pi_{k_1-1}(\bX_{k_1}, a_1)\hat\pi_{k_2-1}(\bX_{k_2}, a_2)}\tilde\pi_{T,j}(\bX_{k_1},a)\tilde\pi_{T,j}(\bX_{k_2},a)\right]\nonumber\\
  &+2\mathbb{E}\left[\frac{1}{(T-j+1)^2}\left\{\sum_{k=j}^T\sum_{a=1}^K\mathbb{E}_{\bX}\left[\mu_a(\bX,a)\tilde\pi_{T,j}(\bX,a)\right]\right\}^2\right]\nonumber\\
  \le&\frac{2}{(T-j+1)^2}\sum_{ j\le k_1<k_2\le T}\sum_{a_1=1}^K\sum_{a_2=1}^K\mathbb{E}\left[\frac{R_{k_1}R_{k_2}\mathbb{I}(A_{k_1}=a_1)\mathbb{I}(A_{k_2}=a_2)}{\hat\pi_{k_1-1}(\bX_{k_1}, a_1)\hat\pi_{k_2-1}(\bX_{k_2}, a_2)}\tilde\pi_{T,j}(\bX_{k_1},a)\tilde\pi_{T,j}(\bX_{k_2},a)\right]\label{0127temp1}\\
  &+\frac{1}{(T-j+1)^2}\sum_{a_1=1}^K\sum_{a_2=1}^K\sum_{k=j}^T \mathbb{E}\left[\frac{R_k^2\mathbb{I}(A_k=a_1)\mathbb{I}(A_k=a_2)}{\hat\pi_{k-1}(\bX_k,a)^2}\tilde\pi_{T,j}(\bX_k,a)^2\right]\label{0127temp2}\\
  &+2K^2\mu_U^2 \sup_{1\le a\le K} \mathbb{E}\left[\tilde\pi_{T,j}(\bX,a)^2\right].
\end{align}
We can further bound the term in  Equation~\eqref{0127temp1} as,
\begin{align*}
  &~\frac{2}{(T-j+1)^2}\sum_{ j\le k_1<k_2\le T}\sum_{a_1=1}^K\sum_{a_2=1}^K\mathbb{E}\left[\frac{R_{k_1}R_{k_2}\mathbb{I}(A_{k_1}=a_1)\mathbb{I}(A_{k_2}=a_2)}{\hat\pi_{k_1-1}(\bX_{k_1}, a_1)\hat\pi_{k_2-1}(\bX_{k_2}, a_2)}\tilde\pi_{T,j}(\bX_{k_1},a)\tilde\pi_{T,j}(\bX_{k_2},a)\right] \nonumber\\
  =&~\frac{2}{(T-j+1)^2}\sum_{ j\le k_1<k_2\le T}\sum_{a_1=1}^K\sum_{a_2=1}^K\mathbb{E}\left[\frac{R_{k_1}\mathbb{I}(A_{k_1}=a_1)}{\hat\pi_{k_1-1}(\bX_{k_1}, a_1)}\tilde\pi_{T,j}(\bX_{k_1},a)\mathbb{E}\left[\tilde\pi_{T,j}(\bX_{k_2},a)\frac{R_{k_2}\mathbb{I}(A_{k_2}=a)}{\hat\pi_{k_2-1}(\bX_k,a)}\Bigr| \bH_{k_2-1}\right]\right]\nonumber\\
  =&~\frac{2}{(T-j+1)^2}\sum_{ j\le k_1<k_2\le T}\sum_{a_1=1}^K\sum_{a_2=1}^K\mathbb{E}\left[\frac{R_{k_1}\mathbb{I}(A_{k_1}=a_1)}{\hat\pi_{k_1-1}(\bX_{k_1}, a_1)}\tilde\pi_{T,j}(\bX_{k_1},a)\mathbb{E}_{\bX}\left[\tilde\pi_{T,j}(\bX,a)\mu_a(\bX)\right]\right]\nonumber\\
  =&~\frac{2}{(T-j+1)^2}\sum_{ j\le k_1<k_2\le T}\sum_{a_1=1}^K\sum_{a_2=1}^K\mathbb{E}\left[\mathbb{E}\left[\frac{R_{k_1}\mathbb{I}(A_{k_1}=a_1)}{\hat\pi_{k_1-1}(\bX_{k_1}, a_1)}\tilde\pi_{T,j}(\bX_{k_1},a)\Bigr| \bH_{k_1-1}\right]\mathbb{E}_{\bX}\left[\tilde\pi_{T,j}(\bX,a)\mu_a(\bX)\right]\right]\nonumber\\
  =&~\frac{2}{(T-j+1)^2}\sum_{ j\le k_1<k_2\le T}\sum_{a_1=1}^K\sum_{a_2=1}^K\mathbb{E}\left[\mathbb{E}_{\bX}\left[\tilde\pi_{T,j}(\bX,a)\mu_a(\bX)\right]^2\right] \nonumber\\
  \le&~\frac{2K^2\mu_U^2(T-j)}{T-j+1}\sup_a\mathbb{E}\left[\tilde\pi_{T,j}(\bX,a)^2\right]\nonumber\\
  \le &~ 2K^2\mu_U^2\sup_a\mathbb{E}\left[\tilde\pi_{T,j}(\bX,a)^2\right]
\end{align*}

Lastly, we bound the term in Equation~\eqref{0127temp2} as,
\begin{align*}
&~\frac{1}{(T-j+1)^2}\sum_{a_1=1}^K\sum_{a_2=1}^K\sum_{k=j}^T \mathbb{E}\left[\frac{R_k^2\mathbb{I}(A_k=a_1)\mathbb{I}(A_k=a_2)}{\hat\pi_{k-1}(\bX_k,a)^2}\tilde\pi_{T,j}(\bX_k,a)^2\right]\nonumber\\
=&~\frac{1}{(T-j+1)^2}\sum_{a=1}^K\sum_{k=j}^T \mathbb{E}\left[\frac{R_k^2\mathbb{I}(A_k=a)}{\hat\pi_{k-1}(\bX_k,a)^2}\tilde\pi_{T,j}(\bX_k,a)^2\right] \nonumber\\
\le&~\frac{(\mu_U^2+\sigma_U^2)(T-1)^{\eta}K}{T-j+1} \sup_a \mathbb{E}\left[\tilde\pi_{T,j}(\bX,a)^2\right]
\end{align*}
Therefore,
\begin{align}
  &~\mathbb{E}\left[\left\{\frac{1}{T-j+1}\sum_{k=j}^T\sum_{a=1}^K\left(\frac{R_k\mathbb{I}(A_k=a)}{\hat\pi_{k-1}(\bX_k,a)}\tilde\pi_{T,j}(\bX_k,a)
    +\mathbb{E}\left[\frac{R_k\mathbb{I}(A_k=a)}{\hat\pi_{k-1}(\bX_k,
    a)}\tilde\pi_{T,j}(\bX_k,a)\ \Bigr|\ \bH_{k-1}\right]\right)\right\}^2\right]\nonumber\\
\le&~2K^2(\mu_U^2+\sigma_U^2)(T-1)^\eta \sup_a \mathbb{E}\left[\tilde\pi_{T,j}(\bX,a)^2\right]\label{0127temp4}
\end{align}
Plugging Equations~\eqref{0127temp3}~and~\eqref{0127temp4} into Equation~\eqref{0127temp0}, we have
\begin{align*}
  &(T-1)\sum_{j=2}^T
    \mathbb{E}\left[\left\{\frac{1}{T-j+1}\sum_{k=j}^T\sum_{a=1}^K\left(\frac{R_k\mathbb{I}(A_k=a)}{\hat\pi_{k-1}(\bX_k,
    a)}\tilde\pi_{T,j}(\bX_k,a)
    -\mathbb{E}\left[\frac{R_k\mathbb{I}(A_k=a)}{\hat\pi_{k-1}(\bX_k,
    a)}\tilde\pi_{T,j}(\bX_k,a)\ \Bigr| \ \bH_{k-1}\right]\right)\right\}^2\right]^{\frac{1}{2}}\nonumber\\
  &~\cdot\mathbb{E}\left[\left\{\frac{1}{T-j+1}\sum_{k=j}^T\sum_{a=1}^K\left(\frac{R_k\mathbb{I}(A_k=a)}{\hat\pi_{k-1}(\bX_k,
    a)}\tilde\pi_{T,j}(\bX_k,a)
    +\mathbb{E}\left[\frac{R_k\mathbb{I}(A_k=a)}{\hat\pi_{k-1}(\bX_k,
    a)}\tilde\pi_{T,j}(\bX_k,a)\ \Bigr| \ \bH_{k-1}\right]\right)\right\}^2\right]^{\frac{1}{2}} \nonumber\\
  \le&~(T-1)\sum_{j=2}^T \sqrt{\frac{4K(\mu_U^2+\sigma_U^2)(T-1)^\eta}{T-j+1}\sup_{a}\mathbb{E}\left[\tilde\pi_{T,j}(\bX,a)^2\right]}\cdot\sqrt{2K^2(\mu_U^2+\sigma_U^2)(T-1)^\eta \sup_a \mathbb{E}\left[\tilde\pi_{T,j}(\bX,a)^2\right]}\nonumber\\
  \le&~2\sqrt{2}(\mu_U^2+\sigma_U^2)K^{\frac{3}{2}}(T-1)^{1+\eta}\sum_{j=2}^T\frac{1}{\sqrt{T-j+1}}\sup_{a} \mathbb{E}\left[\tilde\pi_{T,j}(\bX,a)^2\right].
\end{align*}
Notice for any $a$, $T$, $j$,
\begin{align*}
  & \mathbb{E}\left[\tilde\pi_{T,j}(\bX,a)^2\right] \\
  = & \mathbb{E}\left[\left(\tilde\pi_{T,j}(\bX,a)-\frac{\hat\pi_{j-1(\bX,a)}}{T-1}\right)^2\right] + 2\mathbb{E}\left[\left(\tilde\pi_{T,j}(\bX,a)-\frac{\hat\pi_{j-1(\bX,a)}}{T-1}\right)\frac{\hat\pi_{j-1}(\bX,a)}{T-1}\right] + \mathbb{E}\left[\frac{\hat\pi_{j-1}(\bX,a)^2}{(T-1)^2}\right].
\end{align*}
By Lemma~\ref{lemma:L2_approx},
$$
\sum_{j=2}^T\sum_{a=1}^K \mathbb{E}\left[\left(\tilde\pi_{T,j}(X_j,a)-\frac{\hat\pi_{j-1}(X_j,a)}{T-1}\right)^2\right] = o\left(\frac{1}{(T-1)^{1+\eta}}\right).
$$
In addition, by Lemma~\ref{lemma:L1_approx}, we have,
$$\sum_{j=2}^T
2\mathbb{E}\left[\left(\tilde\pi_{T,j}(\bX,a)-\frac{\hat\pi_{j-1}(\bX,a)}{T-1}\right)\frac{\hat\pi_{j-1}(\bX,a)}{T-1}\right]
= O\left(\frac{\log T}{(T-1)^{1+\min(\delta,1)}}\right).$$
Furthermore,
\begin{align*}
\sum_{j=2}^T\frac{1}{\sqrt{T-j+1}}\sum_{a=1}^K\mathbb{E}\left[\frac{\hat\pi_{j-1}(\bX,a)^2}{(T-1)^2}\right]\le&\frac{K}{(T-1)^2}\sum_{j=2}^T \frac{1}{\sqrt{T-j+1}} = O\left(\frac{1}{(T-1)^{\frac{3}{2}}}\right).
\end{align*}
Therefore, 
\begin{align*}
\sum_{j=2}^T\frac{1}{\sqrt{T-j+1}}\mathbb{E}\left[\tilde\pi_{T,j}(\bX,a)^2\right] &= o\left(\frac{1}{(T-1)^{1+\eta}}\right)+O\left(\frac{\log T}{(T-1)^{1+\min(\delta,1)}}\right)+O\left(\frac{1}{(T-1)^{\frac{3}{2}}}\right).\nonumber\\
&= o\left(\frac{1}{(T-1)^{1+\eta}}\right)
\end{align*}
All together, we have shown,
\begin{align*}
  &(T-1)\sum_{j=2}^T
    \mathbb{E}\left[\left\{\frac{1}{T-j+1}\sum_{k=j}^T\sum_{a=1}^K\left(\frac{R_k\mathbb{I}(A_k=a)}{\hat\pi_{k-1}(\bX_k, a)}\tilde\pi_{T,j}(\bX_k,a)
    -\mathbb{E}\left[\frac{R_k\mathbb{I}(A_k=a)}{\hat\pi_{k-1}(\bX_k,
    a)}\tilde\pi_{T,j}(\bX_k,a)\ \Bigr| \ \bH_{k-1}\right]\right)\right\}^2\right]^{\frac{1}{2}}\nonumber\\
  &\cdot\mathbb{E}\left[\left\{\frac{1}{T-j+1}\sum_{k=j}^T\sum_{a=1}^K\left(\frac{R_k\mathbb{I}(A_k=a)}{\hat\pi_{k-1}(\bX_k, a)}\tilde\pi_{T,j}(\bX_k,a)
    +\mathbb{E}\left[\frac{R_k\mathbb{I}(A_k=a)}{\hat\pi_{k-1}(\bX_k,
    a)}\tilde\pi_{T,j}(\bX_k,a)\ \Bigr|\ \bH_{k-1}\right]\right)\right\}^2\right]^{\frac{1}{2}} \rightarrow 0.
\end{align*}
As $L_1$ convergence implies convergence in probability, Equation~\eqref{01242} converges to zero in probability.

\end{proof}

\section{Proof of Corollary~\ref{cor:ci}}
\label{cor:ci_proof}

\begin{proof}
  Given Theorem~\ref{thm:Asymptotic_Normality}, it suffices to show:
$$
\left|\frac{v_T}{\hat v_T}-1\right| \stackrel{p}{\longrightarrow} 0.
$$
We begin by noting that $\forall \epsilon >0$,
$$
\mathbb{P}\left(\left|\frac{v_T}{\hat v_T}-1\right|>\epsilon\right) \
= \ \mathbb{P}\left(\left|{v_T}-{\hat v_T}\right|>\epsilon \hat v_T\right).
$$
Condition~\eqref{cond:vt_lower_bound} in the proof of
Theorem~\ref{thm:Asymptotic_Normality} implies that there exists a
constant $c_1>0$ such that
$\lim_{T\rightarrow\infty}\mathbb{P}\left(v_T<c_1/2\right)=0$.
Theorem~\ref{thm:variance_estimator}, together with the Continuous
Mapping Theorem, yields
$\left|\hat v_{T}-v_{T}\right| \stackrel{p}{\longrightarrow} 0$.
Therefore,
$$
\lim_{T\rightarrow\infty}\mathbb{P}\left(|\hat
  v_{T}|<\frac{c_1}{4}\right)  \le \lim_{T\rightarrow\infty}\mathbb{P}\left(|\hat v_{T}-v_{T}|>\frac{c_1}{4}\right) + \lim_{T\rightarrow\infty}\mathbb{P}\left(|v_{T}|<{\frac{c_1}{2}}\right) = 0.
$$
Finally, for all $\epsilon >0$, 
\begin{align}
	\mathbb{P}\left(\Bigr|{v_{T}}-{\hat v_{T}}\Bigr|>\epsilon \hat v_{T}\right) \nonumber &\le \mathbb{P}\left(|\hat v_{T}|\ge\frac{c_1}{4}, \Bigr|{v_{T}}-{\hat v_{T}}\Bigr|>\frac{c_1\epsilon}{4}\right) + \mathbb{P}\left(|v_{T}|<\frac{c_1}{4}\right)\nonumber\\
	&\le  \mathbb{P}\left(\Bigr|{v_{T}}-{\hat v_{T}}\Bigr|>\frac{c_1\epsilon}{4}\right) + \mathbb{P}\left(|v_{T}|<\frac{c_1}{4}\right)\nonumber,
\end{align}
and the right-hand-side goes to 0 as $T\rightarrow\infty$. Therefore,
\begin{align}
	\Bigr|\frac{v_{T}}{\hat v_{T}}-1\Bigr| \xrightarrow{p} 0 \nonumber.
\end{align}
\end{proof}

\section{Proof of Corollary~\ref{cor:final_policy}}
\label{cor:final_policy_proof}
\begin{proof}
  By the definition of the value function,
  \begin{align*}
    &(T-1)\mathbb{E}\left[|\Delta(\hat\pi_T,\hat\pi_{T-1})|\right] \nonumber\\
    =&(T-1)\mathbb{E}\left[\left|\sum_{a=1}^K \mu_a(\bX)\left(\hat\pi_T(\bX)-\hat\pi_{T-1}(\bX)\right)\right|\right] \nonumber\\
    =&(T-1)K\mu_U \cdot\sup_a \mathbb{E}\left[|\hat\pi_T(\bX,a) - \hat\pi_{T-1}(\bX,a)|\right]\nonumber\\
    \le& (T-1)K\mu_U \cdot \mathbb{E}\left[Q_T\right] \le K\mu_UT^{-\delta} \rightarrow 0.
  \end{align*}
Therefore, 
\begin{align}
  \mathbb{E}\left[|\Delta(\hat\pi_T,\hat\pi_{T-1})|\right] = o\left(\frac{1}{T}\right)\label{0129temp1}
\end{align}
and $\Delta(\hat\pi_T,\hat\pi_{T-1}) \xrightarrow{P} 0.$ Since
$\bigr|\widehat V(\hat\pi_{T-1}) - V(\hat\pi_{T})\bigr|
\le\bigr|\widehat V(\hat\pi_{T-1}) - V(\hat\pi_{T-1})\bigr| +
\bigr|V(\hat\pi_{T-1}) - V(\hat\pi_{T})\bigr|$,
Theorem~\ref{thm:L1Consistency} implies
$\bigr|\widehat V(\hat\pi_{T-1}) - V(\hat\pi_{T-1})\bigr| = o_p(1)$,
and as $T\rightarrow\infty$, we have
$ \bigr|\widehat V(\hat\pi_{T-1}) - V(\hat\pi_{T})\bigr|
\xrightarrow{P} 0.  $ Since
$$\sqrt{T-1}\cdot\frac{\widehat{V}(\hat\pi_{T-1}) -
  V(\hat\pi_{T})}{\hat v_T} =
\sqrt{T-1}\cdot\frac{\widehat{V}(\hat\pi_{T-1}) -
  V(\hat\pi_{T-1})}{\hat v_T} + \sqrt{T-1}\cdot\frac{V(\hat\pi_{T-1})
  - V(\hat\pi_{T})}{\hat v_T},$$
and by Corollary~\ref{cor:ci},
$$
\sqrt{T-1}\cdot\frac{\widehat{V}(\hat\pi_{T-1}) - V(\hat\pi_{T-1})}{\hat v_T} \xrightarrow{d} \mathcal{N}(0,1),
$$
it suffices to show that 
\begin{align}
  \sqrt{T-1}\cdot\frac{V(\hat\pi_{T-1}) - V(\hat\pi_{T})}{\hat v_T} \xrightarrow{P} 0. \label{0129temp2}
\end{align}
In the proof of corollary~\ref{cor:ci}, we have shown that
\begin{align*}
  \mathbb{P}\left(|\hat v_{T}|<\frac{c_1}{4}\right) \rightarrow 0.
\end{align*}
Equation~\eqref{0129temp1} implies
$\sqrt{T-1}(V(\hat\pi_{T-1}) -
V(\hat\pi_{T}))\xrightarrow{P}0$. Therefore, Equation~\eqref{0129temp2} holds.
\end{proof}

\section{Lemmas and their proofs}
\label{app:lemmas}

\subsection{Lemma \ref{lemma:martingale_berry}}
\label{lemma:martingale_berry_proof}
Recall the following definitions,
\begin{align}
  \Lambda_{T,j}&:=
                 \sum_{a=1}^{K}\frac{\mathbb{I}(A_j=a)R_j}{T-1}-\mathbb{E}\left[\frac{\mathbb{I}(A_j=a)R_j}{T-1}
                 \ \Bigr|
                 \ \bH_{j-1}\right] \nonumber\\
  \sigma_{T,j}^2&:= \E\left[\Lambda_{T,j}^2\mid \bH_{j-1}\right] \nonumber\\
  S_T^2&:= \sum_{j=f(T)+1}^{T}\mathbb{E}\left[\sigma_{T,j}^2\mid \bH_{f(T)}\right]. \nonumber
\end{align}
 
\begin{proof}
  Using the Theorem of \cite{heyde2010departure}, there exists a
  constant $K_1$ such that
\begin{align}
  &\sup_w \biggr|\mathbb{P}\left(\sum_{j=f(T)+1}^{T}\Lambda_{T,j}\le
    S_T w \ \Bigr| \ \bH_{f(T)}\right)-\Phi(w)\biggr|\nonumber\\
  \le& K_1
       \left\{S_T^{-4}\left(\sum_{j=f(T)+1}^{T}\E[\Lambda_{T,j}^4\mid
       \bH_{f(T)}]+\mathbb{E}\left[\left(\sum_{j=f(T)+1}^{T}\sigma_{T,j}^2-S_T^2\right)^2\
       \Biggr| \ \bH_{f(T)}\right]\right)\right\}^{\frac{1}{5}}. \label{eq:240518}
\end{align}
Therefore, it suffices to show
\begin{enumerate}
\item There exists a constant $c_2$ such that 
  $$
  \mathbb{P}\left(\sqrt{T-f(T)} S_T \ge c_2\right)\rightarrow 1.
  $$
\item $$
  (T-f(T))^2\sum_{j=f(T)+1}^{T}\E[\Lambda_{T,j}^4\mid \bH_{f(T)}] \xrightarrow{P} 0.
  $$
\item $$
  (T-f(T))^2
  \mathbb{E}\left[\left(\sum_{j=f(T)+1}^{T}\sigma_{T,j}^2-S_T^2\right)^2\
    \Biggr| \ \bH_{f(T)}\right] \xrightarrow{P} 0
  $$
\end{enumerate}

\paragraph{The term $\sqrt{T-f(T)}S_T$ is lower bounded by a constant in probability.}
Using the definition of $\sigma_{T,j}$, we can write,
\begin{align*}
  & \sum_{j=f(T)+1}^{T}\sigma^2_{T,j} \\
  = & \sum_{j=f(T)+1}^T \E\left[\Lambda_{T,j}^2\mid \bH_{j-1}\right] \nonumber\\
  = & \sum_{j=f(T)+1}^T
    \E\left[\left(\sum_{a=1}^{K}\frac{\mathbb{I}(A_j=a)R_j}{T-1}-\mathbb{E}\left[\frac{\mathbb{I}(A_j=a)R_j}{T-1}
    \ \Bigr|  \ \bH_{j-1}\right]\right)^2 \ \Bigr| \ \bH_{j-1}\right]\nonumber\\
  = & \sum_{j=f(T)+1}^T \sum_{a=1}^K \frac{\mathbb{E}_{\bX}\left[(\mu_a(\bX)^2+\sigma_a^2(\bX))\hat\pi_{j-1}(\bX_j,a)\right]}{(T-1)^2} - \frac{\mathbb{E}_{\bX_j}\left[\sum_{a=1}^K\mu_a(\bX_j)\hat\pi_{j-1}(\bX_j,a)\right]^2}{(T-1)^2}.\nonumber
\end{align*}
By Jensen's inequality and Cauchy-Schwartz inequality, for all $j$ and
$\bX\in\mathcal{X}$, we have,
\begin{align}
\mathbb{E}_{\bX}\left[\sum_{a=1}^K\hat\pi_{j-1}(\bX,a)\mu_a(\bX)\right]^2 &\le \mathbb{E}_{\bX}\left[\left(\sum_{a=1}^K\hat\pi_{j-1}(\bX,a)\mu_a(\bX)\right)^2\right]\nonumber\\
 &\le \mathbb{E}_{\bX}\left[\left(\sum_{a=1}^K\sqrt{\hat\pi_{j-1}(\bX,a)}\mu_a(\bX)\cdot \sqrt{\hat\pi_{j-1}(\bX,a)}\right)^2\right]\nonumber\\
 &\le \mathbb{E}_{\bX}\left[\sum_{a=1}^K{\hat\pi_{j-1}(\bX,a)}\mu_a(\bX)^2\cdot \sum_{a=1}^K{\hat\pi_{j-1}(\bX,a)}\right]\nonumber\\
 &= \mathbb{E}_{\bX}\left[\sum_{a=1}^K{\hat\pi_{j-1}(\bX,a)}\mu_a(\bX)^2\right]\nonumber
\end{align}
Therefore, using Assumption~\ref{ass:bounded}, we have,
\begin{align*}
  \sum_{j=f(T)+1}^{T}\sigma^2_{T,j} &\ge \sum_{j=f(T)+1}^T \sum_{a=1}^K \frac{\mathbb{E}_{\bX}\left[\sigma_a^2(\bX)\hat\pi_{j-1}(\bX,a)\right]}{(T-1)^2} \nonumber\\
  &\ge
    \frac{T-f(T)}{(T-1)^2}\sigma_L^2\sum_{a=1}^K\mathbb{E}_{\bX}\left[\hat\pi_{j-1}(\bX,a)\right]
  \\
  & = \frac{T-f(T)}{(T-1)^2}\sigma_L^2.\nonumber
\end{align*}
This implies,
\begin{align*}
\sqrt{T-f(T)}S_T &= \sqrt{T-f(T)}\left(\sum_{j=f(T)+1}^T
                   \mathbb{E}\left[\sigma_{T,j}^2\mid \bH_{f(T)}\right]\right)^{\frac{1}{2}}\nonumber\\
                 &\ge \sqrt{T-f(T)}\cdot \frac{\sqrt{T-f(T)}}{T-1}\sigma_L  \\
  & = \frac{T-f(T)}{T-1}\sigma_L.
\end{align*}
When T is sufficiently large, $\frac{T-f(T)}{T-1}$ will be greater than
$\frac{1}{2}$, yielding,
$$
\sqrt{T-f(T)}S_T \ge \frac{\sigma_L}{2}.
$$

\paragraph{The term
  $\sum_{j=f(T)+1}^{T}(T-f(T))^2\E[\Lambda_{T,j}^4\mid \bH_{f(T)}]$
  converges to zero in probability.}
We can write,
\begin{align*}
&\sum_{j=f(T)+1}^{T}(T-f(T))^2\E[\Lambda_{T,j}^4\mid \bH_{f(T)}] \nonumber\\
=&(T-f(T))^2\sum_{j=f(T)+1}^T
   \mathbb{E}\left[\left(\sum_{a=1}^{K}\frac{\mathbb{I}(A_j=a)R_j}{T-1}-\mathbb{E}\left[\frac{\mathbb{I}(A_j=a)R_j}{T-1}\  \Bigr|\ \bH_{j-1}\right]\right)^4\
   \Bigr| \ \bH_{j-1}\right]\nonumber\\
\le&16(T-f(T))^2\sum_{j=f(T)+1}^T
     \mathbb{E}\left[\left(\sum_{a=1}^{K}\frac{\mathbb{I}(A_j=a)R_j}{T-1}\right)^4
     \ \Bigr| \ \bH_{j-1}\right]\nonumber\\
\le&\frac{16KK_4(T-f(T))^3}{(T-1)^4} = o(1),
\end{align*}
where the last inequality is due to
Assumption~\ref{ass:fourth_moment}. Therefore, we have the desired result,
\begin{align*}
  \sum_{j=f(T)+1}^{T}(T-f(T))^2\E[\Lambda_{T,j}^4\mid \bH_{f(T)}] \xrightarrow{P} 0.
\end{align*}

\paragraph{The term
  $(T-f(T))^2
  \mathbb{E}\left[\left(\sum_{j=f(T)+1}^{T}\sigma_{T,j}^2-S_T^2\right)^2\mid
    \bH_{f(T)}\right]$ converges to zero in probability.}

Using the expression of $\sigma_{T,j}^2$ derived above, we have,
\begin{align}
  \sigma_{T,j}^2 =& \sum_{a=1}^K \frac{\mathbb{E}_{\bX}\left[(\mu_a(\bX)^2+\sigma_a^2(\bX))\hat\pi_{j-1}(\bX,a)\right]}{(T-1)^2} - \frac{\mathbb{E}_{\bX}\left[\sum_{a=1}^K\mu_a(\bX)\hat\pi_{j-1}(\bX,a)\right]^2}{(T-1)^2}\nonumber.
\end{align}
Since
\begin{align}
  &~\sum_{a=1}^K \frac{\mathbb{E}_{\bX}\left[(\mu_a(\bX)^2+\sigma_a^2(\bX))\hat\pi_{j-1}(\bX,a)\right]}{(T-1)^2}\nonumber\\
  =&~\sum_{a=1}^K \frac{\mathbb{E}_{\bX}\left[(\mu_a(\bX)^2+\sigma_a^2(\bX))(\hat\pi_{j-1}(\bX,a)-\hat\pi_{f(T)}(\bX,a))\right]}{(T-1)^2}\nonumber + \sum_{a=1}^K\frac{\mathbb{E}_{\bX}\left[(\mu_a(\bX)^2+\sigma_a^2(\bX))\hat\pi_{f(T)}(\bX,a)\right]}{(T-1)^2} \nonumber
\end{align}
and
\begin{align}
&~\frac{\mathbb{E}_{\bX}\left[\sum_{a=1}^K\mu_a(\bX)\hat\pi_{j-1}(\bX,a)\right]^2}{(T-1)^2}\nonumber\\
=&~\frac{\mathbb{E}_{\bX}\left[\sum_{a=1}^K\mu_a(\bX)\hat\pi_{j-1}(\bX,a)\right]^2-\mathbb{E}_{\bX}\left[\sum_{a=1}^K\mu_a(\bX)\hat\pi_{f(T)}(\bX,a)\right]^2}{(T-1)^2}\nonumber\\
&~+\frac{\mathbb{E}_{\bX}\left[\sum_{a=1}^K\mu_a(\bX)\hat\pi_{f(T)}(\bX,a)\right]^2}{(T-1)^2}\nonumber,
\end{align} 
we can write $\sigma_{T,j}^2$ as follows,
  \begin{align}
    \sigma_{T,j}^2=& \sum_{a=1}^K \frac{\mathbb{E}_{\bX}\left[(\mu_a(\bX)^2+\sigma_a^2(\bX))\left(\hat\pi_{j-1}(\bX,a)-\hat\pi_{f(T)}(\bX,a)\right)\right]}{(T-1)^2}\nonumber\\
    &+ \sum_{a=1}^K \frac{\mathbb{E}_{\bX}\left[(\mu_a(\bX)^2+\sigma_a^2(\bX))\hat\pi_{f(T)}(\bX,a)\right]}{(T-1)^2} - \frac{\mathbb{E}_{\bX}\left[\sum_{a=1}^K\mu_a(\bX)\hat\pi_{f(T)}(\bX,a)\right]^2}{(T-1)^2}\nonumber\\
   &- \frac{\mathbb{E}_{\bX}\left[\sum_{a=1}^K\mu_a(\bX)\hat\pi_{j-1}(\bX,a)\right]^2-\mathbb{E}_{\bX}\left[\sum_{a=1}^K\mu_a(\bX)\hat\pi_{f(T)}(\bX,a)\right]^2}{(T-1)^2}\nonumber.
\end{align}
Since $\hat\pi_{f(T)}\in \sigma(\bH_{f(T)})$, 
\begin{align}
\mathbb{E}\left[\sigma_{T,j}^2\mid\bH_{f(T)}\right] =&
                                                       \mathbb{E}\left[\sum_{a=1}^K
                                                       \frac{\mathbb{E}_{\bX}\left[(\mu_a(\bX)^2+\sigma_a^2(\bX))\left(\hat\pi_{j-1}(\bX,a)-\hat\pi_{f(T)}(\bX,a)\right)\right]}{(T-1)^2}\
                                                       \Bigr| \ \bH_{f(T)}\right]\nonumber\\
&+ \sum_{a=1}^K \frac{\mathbb{E}_{\bX}\left[(\mu_a(\bX)^2+\sigma_a^2(\bX))\hat\pi_{f(T)}(\bX,a)\right]}{(T-1)^2} - \frac{\mathbb{E}_{\bX}\left[\sum_{a=1}^K\mu_a(\bX)\hat\pi_{f(T)}(\bX,a)\right]^2}{(T-1)^2}\nonumber\\
&- \mathbb{E}\left[\frac{\mathbb{E}_{\bX}\left[\sum_{a=1}^K\mu_a(\bX)\hat\pi_{j-1}(\bX,a)\right]^2-\mathbb{E}_{\bX}\left[\sum_{a=1}^K\mu_a(\bX)\hat\pi_{f(T)}(\bX,a)\right]^2}{(T-1)^2}\ \Bigr|  \ \bH_{f(T)}\right]\nonumber.
\end{align}
Therefore,
\begin{align}
  &\sigma_{T,j}^2 - \mathbb{E}\left[\sigma_{T,j}^2\mid \bH_{f(T)}\right]\nonumber\\
=&\sum_{a=1}^K \frac{\mathbb{E}_{\bX}\left[(\mu_a(\bX)^2+\sigma_a^2(\bX))\left(\hat\pi_{j-1}(\bX,a)-\hat\pi_{f(T)}(\bX,a)\right)\right]}{(T-1)^2}\nonumber\\
&- \frac{\mathbb{E}_{\bX}\left[\sum_{a=1}^K\mu_a(\bX)\hat\pi_{j-1}(\bX,a)\right]^2-\mathbb{E}_{\bX}\left[\sum_{a=1}^K\mu_a(\bX)\hat\pi_{f(T)}(\bX,a)\right]^2}{(T-1)^2}\nonumber\\
&-\mathbb{E}\left[\sum_{a=1}^K \frac{\mathbb{E}_{\bX}\left[(\mu_a(\bX)^2+\sigma_a^2(\bX))\left(\hat\pi_{j-1}(\bX,a)-\hat\pi_{f(T)}(\bX,a)\right)\right]}{(T-1)^2}\ \Bigr|\ \bH_{f(T)}\right]\nonumber\\
&+ \mathbb{E}\left[\frac{\mathbb{E}_{\bX}\left[\sum_{a=1}^K\mu_a(\bX)\hat\pi_{j-1}(\bX,a)\right]^2-\mathbb{E}_{\bX}\left[\sum_{a=1}^K\mu_a(\bX)\hat\pi_{f(T)}(\bX,a)\right]^2}{(T-1)^2}\ \Bigr|\ \bH_{f(T)}\right]\nonumber.
\end{align}
Therefore,
\begin{align}
&\left|\sigma_{T,j}^2 - \mathbb{E}\left[\sigma_{T,j}^2\mid \bH_{f(T)}\right]\right|\nonumber\\
\le&\sum_{a=1}^K \frac{\mathbb{E}_{\bX}\left[(\mu_a(\bX)^2+\sigma_a^2(\bX))\left|\hat\pi_{j-1}(\bX,a)-\hat\pi_{f(T)}(\bX,a)\right|\right]}{(T-1)^2} \label{01211}\\
&+\left|\frac{\mathbb{E}_{\bX}\left[\sum_{a=1}^K\mu_a(\bX)\hat\pi_{j-1}(\bX,a)\right]^2-\mathbb{E}_{\bX}\left[\sum_{a=1}^K\mu_a(\bX)\hat\pi_{f(T)}(\bX,a)\right]^2}{(T-1)^2}\right| \label{01212}\\
&+\mathbb{E}\left[\sum_{a=1}^K \frac{\mathbb{E}_{\bX}\left[(\mu_a(\bX)^2+\sigma_a^2(\bX))\left(\hat\pi_{j-1}(\bX,a)-\hat\pi_{f(T)}(\bX,a)\right)\right]}{(T-1)^2}\Bigr| \bH_{f(T)}\right]\label{01213}\\
&+  \mathbb{E}\left[\biggr|\frac{\mathbb{E}_{\bX}\left[\sum_{a=1}^K\mu_a(\bX)\hat\pi_{j-1}(\bX,a)\right]^2-\mathbb{E}_{\bX}\left[\sum_{a=1}^K\mu_a(\bX)\hat\pi_{f(T)}(\bX,a)\right]^2}{(T-1)^2}\biggr|\Bigr| \bH_{f(T)}\right]\label{01214}
\end{align}
We will bound each of the above four terms.  First, the term~\eqref{01211} can be bounded as follows,
\begin{align*}
&\sum_{a=1}^K \frac{\mathbb{E}_{\bX}\left[(\mu_a(\bX)^2+\sigma_a^2(\bX))\left|\hat\pi_{j-1}(\bX,a)-\hat\pi_{f(T)}(\bX,a)\right|\right]}{(T-1)^2} \nonumber\\
\le&\frac{K(\mu_U^2+\sigma_U^2)}{(T-1)^2}\cdot\sup_{1\le a\le K} \mathbb{E}_{\bX}\left[\left|\hat\pi_{j-1}(\bX,a)-\hat\pi_{f(T)}(\bX,a)\right|\right] \nonumber\\
\le&\frac{K(\mu_U^2+\sigma_U^2)}{(T-1)^2}\cdot\sum_{t=f(T)}^{j-1} Q_t.
\end{align*}
Second, we bound the term~\eqref{01212} as follows,
\begin{align*}
&\left|\frac{\mathbb{E}_{\bX}\left[\sum_{a=1}^K\mu_a(\bX)\hat\pi_{j-1}(\bX,a)\right]^2-\mathbb{E}_{\bX}\left[\sum_{a=1}^K\mu_a(\bX)\hat\pi_{f(T)}(\bX,a)\right]^2}{(T-1)^2}\right|\nonumber\\
\le&\frac{1}{(T-1)^2}\cdot\left|\mathbb{E}_{\bX}\left[\sum_{a=1}^K\mu_a(\bX)\cdot\left(\hat\pi_{j-1}(\bX,a) + \hat\pi_{f(T)}(\bX,a)\right)\right]\right|\nonumber\\
&\cdot \left|\mathbb{E}_{\bX}\left[\sum_{a=1}^K\mu_a(\bX)\cdot\left(\hat\pi_{j-1}(\bX,a) - \hat\pi_{f(T)}(\bX,a)\right)\right]\right| \nonumber\\
\le&\frac{2K\mu_U}{(T-1)^2}\cdot K\mu_U \mathbb{E}_{\bX}\left[\left|\hat\pi_{j-1}(\bX,a) - \hat\pi_{f(T)}(\bX,a)\right|\right] \nonumber\\
\le&\frac{2K^2\mu_U^2}{(T-1)^2}\sum_{t=f(T)+1}^{j-1}Q_t.
\end{align*}
The third term~\eqref{01213} can be bounded similarly, 
\begin{align*}
  &\mathbb{E}\left[\sum_{a=1}^K \frac{\mathbb{E}_{\bX}\left[(\mu_a(\bX)^2+\sigma_a^2(\bX))\left(\hat\pi_{j-1}(\bX,a)-\hat\pi_{f(T)}(\bX,a)\right)\right]}{(T-1)^2}\Bigr| \bH_{f(T)}\right]\nonumber\\
  \le&\frac{K(\mu_U^2+\sigma_U^2)}{(T-1)^2}\sup_{1\le a\le K}\mathbb{E}\left[\mathbb{E}_{\bX}\left[\left|\hat\pi_{j-1}(\bX,a)-\hat\pi_{f(T)}(\bX,a)\right|\right] \mid \bH_{f(T)}\right] \nonumber\\
  \le&\frac{K(\mu_U^2+\sigma_U^2)}{(T-1)^2}\sum_{t=f(T)+1}^{j-1}
       \mathbb{E}\left[Q_t\mid \bH_{f(T)}\right].
\end{align*}
Lastly, the final term~\eqref{01214} can be bounded as follows,
\begin{align*}
&\mathbb{E}\left[\biggr|\frac{\mathbb{E}_{\bX}\left[\sum_{a=1}^K\mu_a(\bX)\hat\pi_{j-1}(\bX,a)\right]^2-\mathbb{E}_{\bX}\left[\sum_{a=1}^K\mu_a(\bX)\hat\pi_{f(T)}(\bX,a)\right]^2}{(T-1)^2}\biggr|\Bigr| \bH_{f(T)}\right]\nonumber\\
% \le&\frac{1}{(T-1)^2}\cdot\left|\mathbb{E}_{\bX_j}\left[\sum_{a=1}^K\mu_a(\bX_j)\cdot\left(\hat\pi_{j-1}(\bX_j,a)
%      + \hat\pi_{f(T)}(\bX_j,a)\right)\ \Bigr| \ \bH_{f(T)}\right]\right|\nonumber\\
% &\cdot \left|\mathbb{E}_{\bX_j}\left[\sum_{a=1}^K\mu_a(\bX_j)\cdot\left(\hat\pi_{j-1}(\bX_j,a)  \hat\pi_{f(T)}(\bX_j,a)\right)\ \Bigr | \ \bH_{f(T)}\right]\right|  \nonumber\\
\le&\frac{2K^2\mu_U^2}{(T-1)^2}\mathbb{E}\left[\mathbb{E}_{\bX}\left[\left|\hat\pi_{j-1}(\bX,a) - \hat\pi_{f(T)}(\bX,a)\right|\right]\mid \bH_{f(T)}\right] \nonumber\\
\le&\frac{2K^2\mu_U^2}{(T-1)^2}\sum_{t=f(T)+1}^{j-1}\mathbb{E}\left[Q_t\mid
     \bH_{f(T)}\right].
\end{align*}
Putting all these results together, we have,
\begin{align*}
  \left|\sigma_{T,j}^2 - \mathbb{E}\left[\sigma_{T,j}^2\mid \bH_{f(T)}\right]\right| 
  \le&\frac{4K^2(\mu_U^2+\sigma_U^2)}{(T-1)^2}
       \left(\sum_{t=f(T)+1}^{j-1}\mathbb{E}\left[Q_t\mid \bH_{f(T)}\right]+\sum_{t=f(T)+1}^{j-1}Q_t\right).
\end{align*}
Therefore,
\begin{align}
    \left|\sum_{j=f(T)+1}^{T}\sigma_{T,j}^2-S_T^2\right| =& \left|\sum_{j=f(T)+1}^{T} \sigma_{T,j}^2 - \mathbb{E}\left[\sigma_{T,j}^2 \mid \bH_{f(T)}\right]\right|\nonumber \\
    \le&\frac{4K^2(\mu_U^2+\sigma_U^2)}{(T-1)^2}\sum_{j=f(T)+1}^T\left(\sum_{t=f(T)+1}^{j-1}\mathbb{E}\left[Q_t\mid
         \bH_{f(T)}\right]+\sum_{t=f(T)+1}^{j-1}Q_t\right).\nonumber
\end{align}
Since
$\mathbb{E}\left[(T-1)^2\left|\sum_{j=f(T)+1}^{T}\sigma_{T,j}^2-S_T^2\right|^2\mid
  \bH_{f(T)}\right]$ is a nonnegative random variable, to show it
converges to zero in probability, it suffices to show
$\mathbb{E}\left[(T-1)^2\left|\sum_{j=f(T)+1}^{T}\sigma_{T,j}^2-S_T^2\right|^2\right]$
converges to zero in probability.
\begin{align*}
  &\mathbb{E}\left[(T-1)^2\left|\sum_{j=f(T)+1}^{T}\sigma_{T,j}^2-S_T^2\right|^2\right] \nonumber\\
  \le&\frac{16K^4(\mu_U^2+\sigma_U^2)^2}{(T-1)^2}\mathbb{E}\left[\left\{\sum_{j=f(T)+1}^T\left(\sum_{t=f(T)+1}^{j-1}\mathbb{E}\left[Q_t\mid
       \bH_{f(T)}\right]+\sum_{t=f(T)+1}^{j-1}Q_t\right)\right\}^2\right]\nonumber\\
  \le&\frac{64K^4(\mu_U^2+\sigma_U^2)^2}{(T-1)^2}\mathbb{E}\left[\left(\sum_{j=f(T)+1}^T\sum_{t=f(T)+1}^{j-1}Q_t\right)^2\right]\nonumber\\
  \le&\frac{64K^4(\mu_U^2+\sigma_U^2)^2}{(T-1)^2}\cdot\sup_{t}\mathbb{E}\left[t^{1+\delta}Q_t\right]\cdot\mathbb{E}\left[\left(\sum_{j=f(T)+1}^T\sum_{t=f(T)+1}^{j-1}t^{-1-\delta}\right)^2\right]
  \\
 = &  O(f(T)^{-2\delta}).
\end{align*}

\end{proof}

\subsection{Lemma~\ref{lemma:xi_consistency}}

\begin{proof}
  \label{lemma:xi_consistency_proof}
  From the definitions of $v_T^2$ in Theorem~\ref{thm:Asymptotic_Normality},
  $$
  v_T^2 := (T-1)\sum_{j=2}^T \V(\widehat\Gamma_j(T)\mid \bH_{j-1}).
  $$
  Using the definition of $\widehat\Gamma_j(T)$ in~\eqref{eq:Gamma_j_definition2} and the definition of  $\tilde\pi_{T,j}$ in~\eqref{def:tilde} 
  \begin{align*}
    v_T^2 &:= (T-1)\sum_{j=2}^T \V(\widehat\Gamma_j(T)\mid \bH_{j-1})\nonumber\\
          &\ = (T-1)\sum_{j=2}^T \V\left(\sum_{a=1}^K \frac{ \mathbb{I}(A_j=a)R_j}{\hat\pi_{j-1}(\bX_j,a)}\cdot\left(\frac{\hat\pi_1(\bX_j,a)}{T-1}+\sum_{t=2}^{j-1}\frac{\hat\pi_t(\bX_j,a)-\hat\pi_{t-1}(\bX_j,a)}{T-t}\right)\ \Bigr|\ \bH_{j-1}\right)\nonumber\\
          &\ = (T-1)\sum_{j=2}^T \V\left(\sum_{a=1}^K \frac{ \mathbb{I}(A_j=a)R_j}{\hat\pi_{j-1}(\bX_j,a)}\cdot\tilde\pi_{T,j}(\bX_j,a)\ \Bigr|\ \bH_{j-1}\right)\nonumber\\
          &\ = (T-1)\sum_{j=2}^T \mathbb{E}\left[\sum_{a=1}^K
            \frac{\mu_a(\bX_j)^2+\sigma_a^2(\bX_j)}{\hat\pi_{j-1}(\bX_j,a)}
            \cdot \tilde\pi_{T,j}(\bX_j,a) \ \Bigr| \ \bH_{j-1}\right] \nonumber\\
          &\ \ \ \ \ \ - (T-1)\sum_{j=2}^T
            \mathbb{E}\left[\sum_{a=1}^K\mu_a(\bX_j)\tilde\pi_{T,j}(\bX_j,a)\
            \Bigr | \ \bH_{j-1}\right]^2.
            \nonumber
\end{align*} 
From the definition of $S_T^2$ and $\sigma_{T,j}^2$ in~\eqref{eq:ST} and~\eqref{eq:SigmaTj},
\begin{align*} 
  & (T-1)S_T^2\\
  := & (T-1)\sum_{j=f(T)+1}^{T}\mathbb{E}\left[\sigma_{T,j}^2\mid
    \bH_{f(T)}\right] \nonumber\\
    = & (T-1)\sum_{j=f(T)+1}^T \mathbb{E}\left[\mathbb{E}\left[\left(\sum_{a=1}^{K}\left(\frac{\mathbb{I}(A_j=a)R_j}{T-1}-\mathbb{E}\left[\frac{\mathbb{I}(A_j=a)R_j}{T-1}\
    \Bigr|  \ \bH_{j-1}\right]\right)\right)^2\Bigr| \bH_{j-1}\right]
        \ \biggr| \ \bH_{f(T)}\right]\nonumber\\
    = & (T-1)\sum_{j=f(T)+1}^T
        \mathbb{E}\left[\sum_{a=1}^K\frac{\mathbb{E}_{\bX}\left[(\mu_a(\bX)^2+\sigma_a^2(\bX))\hat\pi_{j-1}(\bX,a)\right]}{(T-1)^2}-\frac{\mathbb{E}_{\bX}\left[\sum_{a=1}^K\mu_a(\bX)\hat\pi_{j-1}(\bX,a)\right]^2}{(T-1)^2}\
        \biggr| \
      \bH_{f(T)}\right]\nonumber\\
  = & \sum_{j=f(T)+1}^T \sum_{a=1}^K \frac{\mathbb{E}\left[(\mu_a(\bX_j)^2+\sigma_a^2(\bX_j))\hat\pi_{j-1}(\bX_j,a)\mid \bH_{f(T)}\right]}{T-1}\nonumber\\
& \ \ \ \ -  \sum_{j=f(T)+1}^T\frac{\mathbb{E}\left[\mathbb{E}\left[\sum_{a=1}^K\mu_a(\bX_j)\hat\pi_{j-1}(\bX_j,a)\mid \bH_{j-1}\right]^2\  \Bigr|  \  \bH_{f(T)}\right]}{T-1}.
\end{align*}

To show $v_T^2 - (T-1)S_T^2$ converges to zero in probability, we will
show later the following two convergence results,
\begin{align}
  &\left|(T-1)\sum_{j=2}^T \mathbb{E}\left[\sum_{a=1}^K
    \frac{\mu_a(\bX_j)^2+\sigma_a^2(\bX_j)}{\hat\pi_{j-1}(\bX_j,a)} \cdot
    \tilde\pi_{T,j}(\bX_j,a) \ \Bigr| \ \bH_{j-1}\right]\right.\nonumber\\
   &\ \ \ \left.- \sum_{j=f(T)+1}^T \sum_{a=1}^K \frac{\mathbb{E}\left[(\mu_a(\bX_j)^2+\sigma_a^2(\bX_j))\hat\pi_{j-1}(X,a)\mid \bH_{f(T)}\right]}{T-1}\right| \xrightarrow{P} 0 \label{01215}
\end{align}
\begin{align}
  &\left|(T-1)\sum_{j=2}^T
    \mathbb{E}\left[\sum_{a=1}^K\mu_a(\bX_j)\tilde\pi_{T,j}(\bX_j,a)\
    \Bigr | \ \bH_{j-1}\right]^2\right. \nonumber\\ 
    &\ \ \ \ \left.- \sum_{j=f(T)+1}^T\frac{\mathbb{E}\left[\mathbb{E}\left[\sum_{a=1}^K\mu_a(\bX_j)\hat\pi_{j-1}(\bX_j,a)\mid \bH_{j-1}\right]^2\Bigr| \bH_{f(T)}\right]}{T-1}\right| \xrightarrow{P} 0 \label{012152}
\end{align}

Once we prove Equations~\eqref{01215} and~\eqref{012152}, then, for
any $\epsilon>0$, we have,
\begin{equation*}
  \lim_{T\rightarrow\infty}\mathbb{P}\left(\Bigr|v_T^2 - (T-1)S_T^2 \Bigr|\ge \epsilon\right) = 0,
\end{equation*}
which completes the proof of the first part of
Lemma~\ref{lemma:xi_consistency}. To show the second part of Lemma~\ref{lemma:xi_consistency}, i.e.
$\frac{\sqrt{T-1}S_T}{v_T} \xrightarrow{p} 1$ as $T\rightarrow\infty$,
by Condition~\ref{cond:vt_lower_bound}, there exists a constant
$c_1>0$ such that
$\lim_{T\rightarrow\infty}\mathbb{P}\left(v_T\ge c_1\right) = 1$,
which implies
$\lim_{T\rightarrow\infty}\mathbb{P}\left(v_T^2<c_1^2\right) = 0$.
For all $\epsilon>0$,
\begin{align}
  \mathbb{P}\left(\left|\frac{{(T-1)}S_T^2}{v_T^2}-1\right|>\epsilon\right) &\le \mathbb{P}\left(v_T^2<c_1^2\right) + \mathbb{P}\left(\Bigr|\frac{{(T-1)}S_T^2}{v_T^2}-1\Bigr|>\epsilon, v_T^2\ge c_1^2\right)\nonumber\\
                                                                           &\le \mathbb{P}\left(v_T^2<c_1^2\right) + \mathbb{P}\left(|(T-1)S_T^2-v_T^2|>c_1^2\epsilon\right). \label{eq:2405182}
\end{align}
Since both terms in Equation~\eqref{eq:2405182} converge to zero as
$t\rightarrow\infty$, we have $\forall\epsilon>0$,
$$
\lim_{T\rightarrow\infty}\mathbb{P}\left(\left|\frac{{(T-1)}S_T^2}{v_T^2}-1\right|>\epsilon\right) = 0 \nonumber.
$$
As both $S_T$ and $\xi(T)$ are non-negative, the above equation implies,
$$
\frac{\sqrt{T-1}S_T}{v_T} \xrightarrow{p} 1 \nonumber.
$$
Now, we prove Equations~\eqref{01215} and~\eqref{012152}.

\paragraph{Proof of Equation~\eqref{01215}.}
We first write $\tilde\pi_{T,j}(\bX_j,a)$ as $\frac{\hat\pi_{j-1}(\bX_j,a)}{T-1}+(\tilde\pi_{T,j}(\bX_j,a)-\frac{\hat\pi_{j-1}(\bX_j,a)}{T-1})$,
\begin{align}
  &(T-1)\sum_{j=2}^T \mathbb{E}\left[\sum_{a=1}^K
    \frac{\mu_a(\bX_j)^2+\sigma_a^2(\bX_j)}{\hat\pi_{j-1}(\bX_j,a)}
    \cdot \tilde\pi_{T,j}(\bX_j,a) \ \Bigr| \ \bH_{j-1}\right] \nonumber\\
  =&\frac{1}{T-1}\sum_{j=2}^T
     \mathbb{E}\left[\sum_{a=1}^K\hat\pi_{j-1}(\bX_j,a)\cdot
     (\mu_a(\bX_j)^2+\sigma_a^2(\bX_j)) \mid \bH_{j-1}\right]\nonumber\\
  &+2\sum_{j=2}^T
    \mathbb{E}\left[\sum_{a=1}^K\left(\tilde\pi_{T,j}(\bX_j,a)-\frac{\hat\pi_{j-1}(\bX_j,a)}{T-1}\right)\cdot
    (\mu_a(\bX_j)^2+\sigma_a^2(\bX_j))\ \Bigr| \ \bH_{j-1}\right]\nonumber\\
  &+(T-1)\sum_{j=2}^T
    \mathbb{E}\left[\sum_{a=1}^K\frac{\mu_a(\bX_j)^2+\sigma_a^2(\bX_j)}{\pi_{j-1}(\bX_j,a)}\cdot\left(\tilde\pi_{T,j}(\bX_j,a)-\frac{\hat\pi_{j-1}(\bX_j,a)}{T-1}\right)^2\
    \Bigr| \ \bH_{j-1}\right]\nonumber
\end{align}
Therefore, ~\eqref{01215} can be written as
\begin{align}
  &\left|(T-1)\sum_{j=2}^T \mathbb{E}\left[\sum_{a=1}^K
    \frac{\mu_a(\bX_j)^2+\sigma_a^2(\bX_j)}{\hat\pi_{j-1}(\bX_j,a)}
    \cdot \tilde\pi_{T,j}(\bX_j,a) \ \Bigr|\ \bH_{j-1}\right]\right.\nonumber\\
   &\ \ \ \left.- \sum_{j=f(T)+1}^T \sum_{a=1}^K
     \frac{\mathbb{E}\left[(\mu_a(\bX_j)^2+\sigma_a^2(\bX_j))\hat\pi_{j-1}(\bX_j,a)\
     \mid \bH_{f(T)}\right]}{T-1}\right| \nonumber\\
  \le&2\sum_{j=2}^T
  \mathbb{E}\left[\sum_{a=1}^K\left|\tilde\pi_{T,j}(\bX_j,a)-\frac{\hat\pi_{j-1}(\bX_j,a)}{T-1}\right|\cdot
  (\mu_a(\bX_j)^2+\sigma_a^2(\bX_j))\ \Bigr | \ \bH_{j-1}\right]\nonumber\\
&+(T-1)\sum_{j=2}^T
  \mathbb{E}\left[\sum_{a=1}^K\frac{\mu_a(\bX_j)^2+\sigma_a^2(\bX_j)}{\pi_{j-1}(\bX_j,a)}\cdot\left(\tilde\pi_{T,j}(\bX_j,a)-\frac{\hat\pi_{j-1}(\bX_j,a)}{T-1}\right)^2\
  \Bigr|\ \bH_{j-1}\right]\nonumber\\
  &+\left|\frac{1}{T-1}\sum_{j=2}^T
  \mathbb{E}\left[\sum_{a=1}^K\hat\pi_{j-1}(\bX_j,a)\cdot
  (\mu_a(\bX_j)^2+\sigma_a^2(\bX_j)) \mid \bH_{j-1}\right]\right. \nonumber\\
  &\ \ \ -\left. \sum_{j=f(T)+1}^T \sum_{a=1}^K
  \frac{\mathbb{E}\left[(\mu_a(\bX_j)^2+\sigma_a^2(\bX_j))\hat\pi_{j-1}(\bX_j,a)\
  \mid \bH_{f(T)}\right]}{T-1}\right|\nonumber\\
   \le&2\sum_{j=2}^T
     \mathbb{E}\left[\sum_{a=1}^K\left|\tilde\pi_{T,j}(\bX_j,a)-\frac{\hat\pi_{j-1}(\bX_j,a)}{T-1}\right|\cdot
     (\mu_a(\bX_j)^2+\sigma_a^2(\bX_j))\ \Bigr | \ \bH_{j-1}\right]\label{01217}\\
   &+(T-1)\sum_{j=2}^T
     \mathbb{E}\left[\sum_{a=1}^K\frac{\mu_a(\bX_j)^2+\sigma_a^2(\bX_j)}{\pi_{j-1}(\bX_j,a)}\cdot\left(\tilde\pi_{T,j}(\bX_j,a)-\frac{\hat\pi_{j-1}(\bX_j,a)}{T-1}\right)^2\
     \Bigr|\ \bH_{j-1}\right] \label{01218}\\
    &+\left|\sum_{j=2}^{f(T)}\sum_{a=1}^K \frac{\mathbb{E}\left[(\mu_a(\bX_j)^2+\sigma_a^2(\bX_j))\hat\pi_{j-1}(\bX_j,a)\mid \bH_{j-1}\right]}{T-1}\right| \label{01216}\\
   &+\left|\sum_{j=f(T)+1}^{T}\sum_{a=1}^K
     \frac{\mathbb{E}\left[(\mu_a(\bX_j)^2+\sigma_a^2(\bX_j))\hat\pi_{j-1}(X,a)\mid
     \bH_{j-1}\right]-\mathbb{E}\left[(\mu_a(\bX_j)^2+\sigma_a^2(\bX_j))\hat\pi_{j-1}(\bX_j,a)\
     \mid \bH_{f(T)}\right]}{T-1}\right|. \label{01219}
\end{align}
We bound the above four terms in turn.  
For the term in Equation~\eqref{01217}, we use
Assumption~\ref{ass:bounded} to obtain,
\begin{align}
  &2\sum_{j=2}^T
    \mathbb{E}\left[\sum_{a=1}^K\left|\tilde\pi_{T,j}(\bX_j,a)-\frac{\hat\pi_{j-1}(\bX_j,a)}{T-1}\right|\cdot
    (\mu_a(\bX_j)^2+\sigma_a^2(\bX_j))\ \Bigr| \ \bH_{j-1}\right] \nonumber\\
  \le&2(\mu_U^2+\sigma_U^2)\sum_{j=2}^T\mathbb{E}\left[\sum_{a=1}^K\left|\tilde\pi_{T,j}(\bX_j,a)-\frac{\hat\pi_{j-1}(\bX_j,a)}{T-1}\right|\
       \Bigr| \ \bH_{j-1}\right]\nonumber\\
  =&2(\mu_U^2+\sigma_U^2)\sum_{j=2}^T\sum_{a=1}^K\mathbb{E}_{\bX}\left[\left|\tilde\pi_{T,j}(\bX,a)-\frac{\hat\pi_{j-1}(\bX,a)}{T-1}\right|\right].\nonumber
\end{align}
Notice Lemma~\ref{lemma:L1_approx} shows that
$2(\mu_U^2+\sigma_U^2)\sum_{j=2}^T\sum_{a=1}^K\mathbb{E}\left[\left|\tilde\pi_{T,j}(\bX_j,a)-\frac{\hat\pi_{j-1}(\bX_j,a)}{T-1}\right|\right]\rightarrow
0$. Since $L_1$ convergence implies convergence in probability, the
term in Equation~\eqref{01217} converges to zero in probability.

For the term in Equation~\eqref{01218}, we apply Assumptions~\ref{ass:clip_rate}~and~\ref{ass:bounded},
\begin{align}
  &(T-1)\sum_{j=2}^T
    \mathbb{E}\left[\sum_{a=1}^K\frac{\mu_a(\bX_j)^2+\sigma_a^2(\bX_j)}{\hat\pi_{j-1}(\bX_j,a)}\cdot\left(\tilde\pi_{T,j}(\bX_j,a)-\frac{\hat\pi_{j-1}(\bX_j,a)}{T-1}\right)^2\
    \Bigr|
    \ \bH_{j-1}\right] \nonumber\\
  \le&(T-1)^{1+\eta}(\mu_U^2+\sigma_U^2)\sum_{j=2}^T
       \mathbb{E}\left[\sum_{a=1}^K\left(\tilde\pi_{T,j}(\bX_j,a)-\frac{\hat\pi_{j-1}(\bX_j,a)}{T-1}\right)^2\
       \Bigr|\ \bH_{j-1}\right]\nonumber\\
  =&(T-1)^{1+\eta}(\mu_U^2+\sigma_U^2)\sum_{j=2}^T\sum_{a=1}^K
  \mathbb{E}_{\bX}\left[\left(\tilde\pi_{T,j}(\bX,a)-\frac{\hat\pi_{j-1}(\bX,a)}{T-1}\right)^2 \right].
\end{align}
By Lemma~\ref{lemma:L2_approx},
\begin{align}
  \lim_{T\rightarrow\infty}(T-1)^{1+\eta}\sum_{j=2}^T\sum_{a=1}^K \mathbb{E}\left[\left(\tilde\pi_{T,j}(\bX_j,a)-\frac{\hat\pi_{j-1}(\bX_j,a)}{T-1}\right)^2\right] = 0.
\end{align}
Since the convergence in $L_1$ implies convergence in probability, the
term in Equation~\eqref{01218} converges to zero in probability.

The term in
Equation~\eqref{01216} can be bounded as follows,
\begin{align*}
&\left|\sum_{j=2}^{f(T)}\sum_{a=1}^K \frac{\mathbb{E}\left[(\mu_a(\bX_j)^2+\sigma_a^2(\bX_j))\hat\pi_{j-1}(\bX_j,a)\mid \bH_{j-1}\right]}{T-1}\right|
\le \frac{Kf(T)}{T-1}(\sigma_U^2+\mu_U^2) = o_p(1).
\end{align*}

Finally, for the term in Equation~\eqref{01219}, we have
\begin{align*}
&\left|\sum_{j=f(T)+1}^{T}\sum_{a=1}^K \frac{\mathbb{E}\left[(\mu_a(\bX_j)^2+\sigma_a^2(\bX_j))\hat\pi_{j-1}(\bX_j,a)\mid \bH_{j-1}\right]-\mathbb{E}\left[(\mu_a(\bX_j)^2+\sigma_a^2(\bX_j))\hat\pi_{j-1}(\bX_j,a)\mid \bH_{f(T)}\right]}{T-1}\right| \nonumber\\
\le&\left|\sum_{j=f(T)+1}^{T}\sum_{a=1}^K \frac{\mathbb{E}\left[(\mu_a(\bX_j)^2+\sigma_a^2(\bX_j))(\hat\pi_{j-1}(\bX_j,a)-\hat\pi_{f(T)}(\bX_j,a))\mid \bH_{j-1}\right]}{T-1}\right| \nonumber\\
&+\left|\sum_{j=f(T)+1}^{T}\sum_{a=1}^K \frac{\mathbb{E}\left[(\mu_a(\bX_j)^2+\sigma_a^2(\bX_j))(\hat\pi_{j-1}(\bX_j,a)-\hat\pi_{f(T)}(\bX_j,a))\mid \bH_{f(T)}\right]}{T-1}\right| \nonumber\\
\le& \frac{K(\mu_U^2+\sigma_U^2)}{T-1}\sum_{j=f(T)+1}^T \sup_{a}\left(\mathbb{E}\left[\left|\hat\pi_{j-1}(\bX_j,a)-\hat\pi_{f(T)}(\bX_j,a)\right| \mid \bH_{j-1}\right]+\mathbb{E}\left[\left|\hat\pi_{j-1}(\bX_j,a)-\hat\pi_{f(T)}(\bX_j,a)\right| \mid\bH_{f(T)}\right]\right) \nonumber\\
=&\frac{K(\mu_U^2+\sigma_U^2)}{T-1}\sum_{j=f(T)+1}^T \sup_{a}\left(\mathbb{E}_{\bX}\left[\left|\hat\pi_{j-1}(\bX,a)-\hat\pi_{f(T)}(\bX,a)\right| \right]+\mathbb{E}\left[\mathbb{E}_{\bX}\left[\left|\hat\pi_{j-1}(\bX,a)-\hat\pi_{f(T)}(\bX,a)\right|\right]\mid\bH_{f(T)}\right]\right) \nonumber\\
\le&\frac{K(\sigma_U^2+\mu_U^2)}{T-1}\sum_{j=f(T)+1}^T\sum_{t=f(T)+1}^{j-1}\left(Q_t + \mathbb{E}\left[Q_t\mid \bH_{f(T)}\right]\right) \nonumber
\end{align*}

Since
\begin{align}
  &~\mathbb{E}\left[\left|\frac{K(\sigma_U^2+\mu_U^2)}{T-1}\sum_{j=f(T)+1}^T\sum_{t=f(T)+1}^{j-1}\left(Q_t + \mathbb{E}\left[Q_t\mid \bH_{f(T)}\right]\right)\right|\right] \nonumber\\
  \le&~\frac{2K(\sigma_U^2+\mu_U^2)}{T-1}\cdot\sup_{t\ge
  1}\mathbb{E}\left[(Q_tt^{1+\delta})\right]\sum_{j=f(T)+1}^T\sum_{t=f(T)+1}^{j-1}\frac{1}{t^{1+\delta}} \nonumber
\\
=&~ O_p(f(T)^{-\delta}) \nonumber
\end{align}
and $L_1$ convergence implies convergence in probability, the term in Equation~\eqref{01219} converges to zero in probability.

\paragraph{Proof of Equation~\eqref{012152}.}
\begin{align}
  &\left|(T-1)\sum_{j=2}^T
  \mathbb{E}\left[\sum_{a=1}^K\mu_a(\bX_j)\tilde\pi_{T,j}(\bX_j,a)\
  \Bigr | \ \bH_{j-1}\right]^2\right. \nonumber\\ 
  &\ \ \ \ \left.- \sum_{j=f(T)+1}^T\frac{\mathbb{E}\left[\mathbb{E}\left[\sum_{a=1}^K\mu_a(\bX_j)\hat\pi_{j-1}(\bX_j,a)\mid \bH_{j-1}\right]^2\Bigr| \bH_{f(T)}\right]}{T-1}\right| \nonumber\\
  \le&(T-1)\cdot\left|\sum_{j=2}^{f(T)}
  \mathbb{E}\left[\sum_{a=1}^K\mu_a(\bX_j)\tilde\pi_{T,j}(\bX_j,a)\mid
  \bH_{j-1}\right]^2\right| + (T-1)\sum_{j=f(T)+1}^T \nonumber\\
  \ \ \ \ \ &\left|\mathbb{E}\left[\sum_{a=1}^K\mu_a(\bX_j)\tilde\pi_{T,j}(\bX_j,a)\
  \Bigr | \ \bH_{j-1}\right]^2 - \mathbb{E}\left[\mathbb{E}\left[\sum_{a=1}^K\mu_a(\bX_j)\frac{\hat\pi_{j-1}(\bX_j,a)}{T-1} \mid \bH_{j-1}\right]^2\Bigr| \bH_{f(T)}\right]\right| \nonumber\\
  \le&(T-1)\cdot\left|\sum_{j=2}^{f(T)}
  \mathbb{E}\left[\sum_{a=1}^K\mu_a(\bX_j)\tilde\pi_{T,j}(\bX_j,a)\mid
  \bH_{j-1}\right]^2\right| \nonumber\\
  &+ (T-1)\sum_{j=f(T)+1}^T \left|\mathbb{E}\left[\sum_{a=1}^K\mu_a(\bX_j)\tilde\pi_{T,j}(\bX_j,a)\
  \Bigr | \ \bH_{j-1}\right]^2 - \mathbb{E}\left[\sum_{a=1}^K\mu_a(\bX_j)\frac{\hat\pi_{j-1}(\bX_j,a)}{T-1}\
  \Bigr | \ \bH_{j-1}\right]^2 \right| \nonumber\\
  &+ (T-1)\sum_{j=f(T)+1}^T\left|\mathbb{E}\left[\sum_{a=1}^K\mu_a(\bX_j)\frac{\hat\pi_{j-1}(\bX_j,a)}{T-1}\ \Bigr | \ \bH_{j-1}\right]^2 - \mathbb{E}\left[\mathbb{E}\left[\sum_{a=1}^K\mu_a(\bX_j)\frac{\hat\pi_{j-1}(\bX_j,a)}{T-1} \mid \bH_{j-1}\right]^2\Bigr| \bH_{f(T)}\right]\right| \nonumber
\end{align}
Firstly, By lemma~\ref{lemma:centering}, $|\tilde\pi_{T,j}(\bX_j,a)|\le \frac{2}{T-j+1}$. Therefore,
\begin{align}
&(T-1)\cdot\left|\sum_{j=2}^{f(T)}
\mathbb{E}\left[\sum_{a=1}^K\mu_a(\bX_j)\tilde\pi_{T,j}(\bX_j,a)\mid
\bH_{j-1}\right]^2\right|\nonumber\\
\le&(T-1)\sum_{j=2}^{f(T)} \frac{4 \sup_{\bx,a} \mu_a(\bx)^2 }{(T-j+1)^2}\nonumber\\
\le&\frac{4(T-1)f(T)K^2\mu_U^2}{(T-f(T)+1)^2} \rightarrow 0\ \  \text{as}\ \ T\rightarrow\infty\nonumber
\end{align}
Secondly,
\begin{align}
&(T-1)\sum_{j=f(T)+1}^T \left|\mathbb{E}\left[\sum_{a=1}^K\mu_a(\bX_j)\tilde\pi_{T,j}(\bX_j,a)\
\Bigr | \ \bH_{j-1}\right]^2 - \mathbb{E}\left[\sum_{a=1}^K\mu_a(\bX_j)\frac{\hat\pi_{j-1}(\bX_j,a)}{T-1}\
\Bigr | \ \bH_{j-1}\right]^2 \right|\nonumber\\
\le&(T-1)\cdot\sum_{j=f(T)+1}^T \left|
    \mathbb{E}\left[\sum_{a=1}^K\mu_a(\bX_j)\tilde\pi_{T,j}(\bX_j,a)
    \mid \bH_{j-1}\right]-
    \mathbb{E}\left[\sum_{a=1}^K\mu_a(\bX_j)\frac{\hat\pi_{j-1}(\bX_j,a)}{T-1}
    \ \Bigr|  \ \bH_{j-1}\right]\right| \nonumber\\
  &\cdot \left|
    \mathbb{E}\left[\sum_{a=1}^K\mu_a(\bX_j)\tilde\pi_{T,j}(\bX_j,a)\mid
    \bH_{j-1}\right]+
    \mathbb{E}\left[\sum_{a=1}^K\mu_a(\bX_j)\frac{\hat\pi_{j-1}(\bX_j,a)}{T-1}\
    \Bigr| \ \bH_{j-1}\right]\right| \nonumber\\
    \le&(T-1)\cdot\sum_{j=f(T)+1}^T \left|
    \mathbb{E}\left[\sum_{a=1}^K\mu_a(\bX_j)\tilde\pi_{T,j}(\bX_j,a)
    \mid \bH_{j-1}\right]-
    \mathbb{E}\left[\sum_{a=1}^K\mu_a(\bX_j)\frac{\hat\pi_{j-1}(\bX_j,a)}{T-1}
    \ \Bigr|  \ \bH_{j-1}\right]\right| \nonumber\\
    &\cdot \left(\frac{2K\mu_L}{T-j+1} + \frac{2K\mu_L}{T-1}\right)\nonumber\\
    \le& (T-1)\cdot\sum_{j=f(T)+1}^T \frac{4K\mu_L}{T-j+1}\cdot\left|
      \mathbb{E}\left[\sum_{a=1}^K\mu_a(\bX_j)\left(\tilde\pi_{T,j}(\bX_j,a)-\frac{\hat\pi_{j-1}(\bX_j,a)}{T-1}\right)
      \mid \bH_{j-1}\right]\right| \nonumber\\
    =& (T-1)\cdot\sum_{j=f(T)+1}^T \frac{4K\mu_L}{T-j+1}\cdot\left|
      \mathbb{E}_{\bX}\left[\sum_{a=1}^K\mu_a(\bX)\left(\tilde\pi_{T,j}(\bX,a)-\frac{\hat\pi_{j-1}(\bX,a)}{T-1}\right)
      \right]\right| \nonumber\\
    \le& (T-1)\cdot\sum_{j=f(T)+1}^T \frac{4K^2\mu_L^2}{T-j+1}\cdot
    \sup_{a}\mathbb{E}_{\bX}\left[\left|\tilde\pi_{T,j}(\bX,a)-\frac{\hat\pi_{j-1}(\bX,a)}{T-1}\right|\right] \nonumber\\
    \le& (T-1)\cdot\sum_{j=f(T)+1}^T \frac{4K^2\mu_L^2}{T-j+1}\cdot\sum_{t=1}^{j-2}\frac{\mathbb{E}_{\bX}\left[|\hat\pi_t(\bX,a)-\hat\pi_{j-1}(\bX,a)|\right]}{(T-t)(T-t-1)} \nonumber
\end{align}
Here the last inequality uses Lemma~\ref{lemma:centering}.
Since $L_1$ convergence implies convergence in probabilities, it suffices to show that 
$$
\sum_{j=f(T)+1}^T\sum_{t=1}^{j-2} \frac{4K^2\mu_L^2(T-1)\mathbb{E}\left[\mathbb{E}_{\bX}\left[|\hat\pi_t(\bX,a)-\hat\pi_{j-1}(\bX,a)|\right]\right]}{(T-j+1)(T-t)(T-t-1)} \rightarrow 0\ \text{as}\ T\rightarrow\infty
$$
This is already proved in equation~\ref{011902} of Lemma~\ref{lemma:L2_approx}.

Thirdly, 
\begin{align}
&(T-1)\sum_{j=f(T)+1}^T\left|\mathbb{E}\left[\sum_{a=1}^K\mu_a(\bX_j)\frac{\hat\pi_{j-1}(\bX_j,a)}{T-1}\ \Bigr | \ \bH_{j-1}\right]^2 - \mathbb{E}\left[\mathbb{E}\left[\sum_{a=1}^K\mu_a(\bX_j)\frac{\hat\pi_{j-1}(\bX_j,a)}{T-1} \mid \bH_{j-1}\right]^2\Bigr| \bH_{f(T)}\right]\right| \nonumber\\
=&(T-1)\sum_{j=f(T)+1}^T\left|\mathbb{E}_{\bX}\left[\sum_{a=1}^K\mu_a(\bX)\frac{\hat\pi_{j-1}(\bX,a)}{T-1}\right]^2 - \mathbb{E}\left[\mathbb{E}_{\bX}\left[\sum_{a=1}^K\mu_a(\bX)\frac{\hat\pi_{j-1}(\bX,a)}{T-1} \right]^2\Bigr| \bH_{f(T)}\right]\right| \nonumber\\
\le&(T-1)\sum_{j=f(T)+1}^T\left|\mathbb{E}_{\bX}\left[\sum_{a=1}^K\mu_a(\bX)\frac{\hat\pi_{j-1}(\bX,a)}{T-1}\right]^2 - \mathbb{E}_{\bX}\left[\sum_{a=1}^K\mu_a(\bX)\frac{\hat\pi_{f(T)}(\bX,a)}{T-1}\right]^2\right| \nonumber\\
&+(T-1)\sum_{j=f(T)+1}^T\left|\mathbb{E}_{\bX}\left[\sum_{a=1}^K\mu_a(\bX)\frac{\hat\pi_{f(T)}(\bX,a)}{T-1}\right]^2-\mathbb{E}\left[\mathbb{E}_{\bX}\left[\sum_{a=1}^K\mu_a(\bX)\frac{\hat\pi_{j-1}(\bX,a)}{T-1} \right]^2\Bigr| \bH_{f(T)}\right]\right|\nonumber\\
\le&(T-1)\sum_{j=f(T)+1}^T\left|\mathbb{E}_{\bX}\left[\sum_{a=1}^K\mu_a(\bX)\frac{\hat\pi_{j-1}(\bX,a)}{T-1}\right] - \mathbb{E}_{\bX}\left[\sum_{a=1}^K\mu_a(\bX)\frac{\hat\pi_{f(T)}(\bX,a)}{T-1}\right]\right| \nonumber\\
&\ \ \cdot \left|\mathbb{E}_{\bX}\left[\sum_{a=1}^K\mu_a(\bX)\frac{\hat\pi_{j-1}(\bX,a)}{T-1}\right] + \mathbb{E}_{\bX}\left[\sum_{a=1}^K\mu_a(\bX)\frac{\hat\pi_{f(T)}(\bX,a)}{T-1}\right]\right|\nonumber\\
&+(T-1)\sum_{j=f(T)+1}^T\left|\mathbb{E}\left[\mathbb{E}_{\bX}\left[\sum_{a=1}^K\mu_a(\bX)\frac{\hat\pi_{f(T)}(\bX,a)}{T-1}\right]^2-\mathbb{E}_{\bX}\left[\sum_{a=1}^K\mu_a(\bX)\frac{\hat\pi_{j-1}(\bX,a)}{T-1} \right]^2\Bigr| \bH_{f(T)}\right]\right|\nonumber\\
\le&2K\mu_U\sum_{j=f(T)+1}^T\left|\mathbb{E}_{\bX}\left[\sum_{a=1}^K\mu_a(\bX)\frac{\hat\pi_{j-1}(\bX,a)}{T-1}\right] - \mathbb{E}_{\bX}\left[\sum_{a=1}^K\mu_a(\bX)\frac{\hat\pi_{f(T)}(\bX,a)}{T-1}\right]\right| \nonumber\\
&+2K\mu_U\sum_{j=f(T)+1}^T\mathbb{E}\left[\left|\mathbb{E}_{\bX}\left[\sum_{a=1}^K\mu_a(\bX)\frac{\hat\pi_{f(T)}(\bX,a)}{T-1}\right]-\mathbb{E}_{\bX}\left[\sum_{a=1}^K\mu_a(\bX)\frac{\hat\pi_{j-1}(\bX,a)}{T-1} \right]\right|\Bigr| \bH_{f(T)}\right]\nonumber
\end{align} 
Since convergence in $L_1$ implies convergence in probability, it suffices to show that
$$
\frac{1}{T-1}\sum_{j=f(T)+1}^T \mathbb{E}\left[\sum_{a=1}^K|\mu_a(\bX,a)|\cdot |\hat\pi_{j-1}(\bX,a) - \hat\pi_{f(T)}(\bX,a)| \right] \rightarrow 0\ \text{as}\ T\rightarrow\infty
$$
Since
\begin{align}
&\frac{1}{T-1}\sum_{j=f(T)+1}^T \mathbb{E}\left[\sum_{a=1}^K|\mu_a(\bX,a)|\cdot |\hat\pi_{j-1}(\bX,a) - \hat\pi_{f(T)}(\bX,a)| \right]\nonumber\\
\le&\frac{K\mu_U}{T-1}\sum_{j=f(T)+1}^T\sum_{t=f(T)+1}^{j-1} \mathbb{E}\left[Q_t\right]\nonumber\\
\le&\frac{K\mu_U\sup\mathbb{E}\left[{t^{1+\delta }Q_t}\right]}{T-1}\sum_{j=f(T)+1}^T\sum_{t=f(T)+1}^{j-1} \frac{1}{t^{1+\delta}}\rightarrow 0\ \text{as}\ T\rightarrow\infty\nonumber
\end{align}
This finishes the proof of Equation~\eqref{012152} converges to zero in probability.

\end{proof}

\subsection{Lemma \ref{lemma:math_order}}

\begin{lemma}
    \label{lemma:math_order}
    Suppose $\alpha$ is a positive constant, then
    \begin{align*}
    \lim_{T\rightarrow\infty}T\sum_{t=1}^{T-1} \frac{1}{t^\alpha (T-t)} =\begin{cases}
    	\sum_{t=1}^\infty t^{-\alpha}, &\alpha > 1\\
    	\infty, & 0<\alpha\le 1.
    \end{cases} 
    \end{align*}
    Furthermore, 
    \begin{align*}
        \lim_{T\rightarrow\infty} \sum_{t=1}^{T-1} \frac{1}{t^\alpha (T-t)} = O\left(\frac{\log T}{T^\alpha}\right),\ \  \forall \ 0<\alpha<1
    \end{align*}
\end{lemma}
\begin{proof}
  We consider three different scenarios separately: $0<\alpha<1$, $\alpha=1$, and $1<\alpha$.
  \paragraph{Case 1: $\alpha=1$.}
  \begin{align*}
    \sum_{t=1}^{T-1} \frac{1}{t^\alpha (T-t)} = \sum_{t=1}^{T-1} \frac{1}{t(T-t)} = \frac{1}{T}\left(\sum_{t=1}^{T-1} \frac{1}{t} +\sum_{t=1}^{T-1} \frac{1}{T-t}\right) = \frac{1}{T}\sum_{t=1}^{T-1}\frac{2}{t}.
  \end{align*}
  Therefore, 
  $$\lim_{T\rightarrow\infty}T \sum_{t=1}^{T-1} \frac{1}{t^\alpha (T-t)} = \lim_{T\rightarrow\infty} 2\sum_{t=1}^{T-1} \frac{1}{t} = \infty,
  $$
  and 
$$
  \lim_{T\rightarrow\infty} \sqrt{T}\sum_{t=1}^{T-1} \frac{1}{t^\alpha (T-t)}  = \lim_{T\rightarrow\infty}\frac{2}{\sqrt{T}}\sum_{t=1}^{T-1}\frac{1}{t} = 0.
$$
  
\paragraph{Case 2: $0<\alpha <1$.}
Since 
\begin{align}
  \sum_{t=1}^{T-1} \frac{1}{t^\alpha (T-t)}\ge \sum_{t=1}^{T-1} \frac{1}{t (T-t)},\nonumber
\end{align}
we have
\begin{align}
  \lim_{T\rightarrow\infty}T\sum_{t=1}^{T-1} \frac{1}{t^\alpha (T-t)}  \ge \lim_{T\rightarrow\infty}T\sum_{t=1}^{T-1} \frac{1}{t (T-t)} = \infty.\nonumber
\end{align} 
On the other side, by Holder's inequality,
\begin{align}
  \sum_{t=1}^{T-1} \frac{1}{t^\alpha (T-t)}&= \sum_{t=1}^{T-1} \frac{1}{t^\alpha (T-t)^\alpha}\cdot \frac{1}{(T-t)^{1-\alpha}}\nonumber\\
                                           &\le \left[\sum_{t=1}^{T-1}\left(t^{-\alpha}(T-t)^{-\alpha}\right)^{\frac{1}{\alpha}}\right]^{\alpha} \left[\sum_{t=1}^{T-1}\left((T-t)^{-1+\alpha}\right)^{\frac{1}{1-\alpha}}\right]^{1-\alpha}\nonumber\\
                                           &\le \left[\sum_{t=1}^{T-1}\frac{1}{t(T-t)}\right]^\alpha \left[\sum_{t=1}^{T-1}\frac{1}{T-t}\right]^{1-\alpha}\nonumber\\
                                           &\le \left[\frac{1}{T}\sum_{t=1}^{T-1}\frac{2}{t}\right]^\alpha \left[\sum_{t=1}^{T-1}\frac{1}{t}\right]^{1-\alpha}\nonumber\\
                                           &\le \frac{2^\alpha}{T^\alpha}\sum_{t=1}^{T-1}\frac{1}{t}.\nonumber
\end{align}
Therefore,
$$
\sum_{t=1}^{T-1} \frac{1}{t^\alpha (T-t)} = O\left(\frac{\log T}{T^\alpha}\right).
$$
	
\paragraph{Case 3: $\alpha > 1$.}
For any $\alpha>1$,
\begin{align}
  \frac{1}{(T-t)t^\alpha} = \frac{1}{T t^\alpha} + \frac{1}{T^2t^{\alpha-1}} + \cdots + \frac{1}{T^{[\alpha]}t^{1+\alpha-[\alpha]}} + \frac{1}{(T-t)T^{[\alpha]}t^{\alpha-[\alpha]}},\nonumber
\end{align}
where $[x]$ denotes the largest integer that is no larger than $x$. Therefore, 
\begin{align}
  \sum_{t=1}^{T-1}\frac{1}{(T-t)t^\alpha} &=
                                            \frac{1}{T}\sum_{t=1}^{T-1}t^{-\alpha}
                                            +\sum_{k=1}^{[\alpha]-1}\sum_{t=1}^{T-1}\frac{t^{-\alpha+k}}{T^{k+1}}
                                            +
                                            \frac{1}{T^{[\alpha]}}\sum_{t=1}^{T-1}\frac{1}{(T-t)t^{\alpha-[\alpha]}},\nonumber
\end{align}
which implies                        
\begin{align}
  T\sum_{t=1}^{T-1}\frac{1}{(T-t)t^\alpha} &= \sum_{t=1}^{T-1}t^{-\alpha} +\sum_{k=1}^{[\alpha]-1}\sum_{t=1}^{T-1}\frac{t^{-\alpha+k}}{T^{k}} + \frac{1}{T^{[\alpha]-1}}\sum_{t=1}^{T-1}\frac{1}{(T-t)t^{\alpha-[\alpha]}}.\nonumber
\end{align}
We consider each term in turn.  For the first term, since $\alpha >
0$, we have,
\begin{align*}
  \lim_{T\rightarrow\infty} \sum_{t=1}^{T-1}t^{-\alpha} = \sum_{t=1}^\infty t^{-\alpha}.
\end{align*}
For the second term, since $1\le k\le [\alpha]-1$ and
$0\le t^{-\alpha+k}\le \ t^{-\alpha + [\alpha]-1}\le t^{-1}$, we have,

$$\begin{aligned}
  0\le \frac{t^{-\alpha+k}}{T^{k}}\le \frac{1}{tT}.
\end{aligned}
$$
Therefore, we have
\begin{align}
  0\le
  \sum_{k=1}^{[\alpha]-1}\sum_{t=1}^{T-1}\frac{t^{-\alpha+k}}{T^{k}}
  &\le ([\alpha]-1)\sum_{t=1}^{T-1}\frac{1}{t T}, \nonumber
\end{align}
which implies,
\begin{align}
  0 \le \lim_{T\rightarrow\infty} \sum_{k=1}^{[\alpha]-1}\sum_{t=1}^{T-1}\frac{t^{-\alpha+k}}{T^{k}} &\le ([\alpha]-1)\lim_{T\rightarrow\infty}\sum_{t=1}^{T-1}\frac{1}{t T} = 0,\nonumber
\end{align}
yielding the desired result,
$$
\lim_{T\rightarrow\infty} \sum_{k=1}^{[\alpha]-1}\sum_{t=1}^{T-1}\frac{t^{-\alpha+k}}{T^{k}} = 0.
$$
Lastly, for the third term, if $\alpha$ is an integer, we must have
$\alpha\ge 2$ because $\alpha>1$.  Therefore,
\begin{align}
  \lim_{T\rightarrow \infty}\frac{1}{T^{[\alpha]-1}}\sum_{t=1}^{T-1}\frac{1}{(T-t)t^{\alpha-[\alpha]}} = \lim_{T\rightarrow \infty}\frac{1}{T^{\alpha-1}}\sum_{t=1}^{T-1}\frac{1}{T-t} = 0.\nonumber
\end{align}
If $\alpha$ is not an integer, then $0<\alpha - [\alpha]<1$. Applying the result of Case~2, we have 
$$
0\le \lim_{T\rightarrow \infty}\frac{1}{T^{[\alpha]-1}}\sum_{t=1}^{T-1}\frac{1}{(T-t)t^{\alpha-[\alpha]}}\le \lim_{T\rightarrow \infty}\sum_{t=1}^{T-1}\frac{1}{(T-t)t^{\alpha-[\alpha]}} = 0. 
$$
Therefore, when $\alpha>1$,
\begin{align*}
  \lim_{T\rightarrow \infty} T\sum_{t=1}^{T-1}\frac{1}{(T-t)t^\alpha} &= \lim_{T\rightarrow\infty}\left(\sum_{t=1}^{T-1}t^{-\alpha} +\sum_{k=1}^{[\alpha]-1}\sum_{t=1}^{T-1}\frac{t^{-\alpha+k}}{T^{k}} + \frac{1}{T^{[\alpha]-1}}\sum_{t=1}^{T-1}\frac{1}{(T-t)t^{\alpha-[\alpha]}}\right)\nonumber\\
                                                                      &=\sum_{t=1}^\infty t^{-\alpha} + 0 + 0\nonumber \\
                                                                      &= \sum_{t=1}^\infty t^{-\alpha}<\infty.
\end{align*}
\end{proof}

\subsection{Lemma~\ref{lemma:centering}}

\begin{lemma}
  \label{lemma:centering}
  for any $\bx\in\mathcal{X}, 1\le a\le K$, $2\le j\le T$, we have
  \begin{align}
  \sum_{a=1}^K\tilde\pi_{T,j}(\bx,a) &= \frac{1}{T-1} \nonumber\\
  \tilde \pi_{T,j}(\bx,a) - \frac{\hat\pi_{j-1}(\bx,a)}{T-1} &= -\sum_{t=1}^{j-2}\frac{\hat\pi_t(\bx,a)-\hat\pi_{j-1}(\bx,a)}{(T-t)(T-t-1)} \nonumber\\
  \left(\tilde \pi_{T,j}(\bx,a) - \frac{\hat\pi_{j-1}(\bx,a)}{T-1}\right)^2 
  &\le \sum_{t_1=1}^{j-2} \frac{2|\hat\pi_{t_1}(\bx,a)-\hat\pi_{j-1}(\bx,a)|}{(T-t_1)(T-t_1-1)(T-j+1)} \nonumber
  \end{align}
Furthermore, 
\begin{align*}
|\tilde\pi_{T,j}(\bx,a)| \le \frac{2}{T-j+1}
\end{align*}
\end{lemma}
\begin{proof}
  First, notice that $\sum_{a=1}^K\hat\pi_j(\bx,a) = 1$ for any policy $\hat\pi_j$ and context $\bx$. Therefore,
\begin{align}
\sum_{a=1}^K\tilde\pi_{T,j}(\bx,a) &= \sum_{a=1}^K\left(\frac{\hat\pi_1(\bx,a)}{T-1}+\sum_{t=2}^{j-1}\frac{\hat\pi_t(\bx,a)-\hat\pi_{t-1}(\bx,a)}{T-t}\right) = \frac{1}{T-1} \nonumber
\end{align}

For the second equality, the definition of $\tilde\pi_{T,j}$ implies,
\begin{align*}
  & \tilde \pi_{T,j}(\bx,a) - \frac{\hat\pi_{j-1}(\bx,a)}{T-1} \\
  = &  \frac{\hat\pi_1(\bx,a)}{T-1}+\sum_{t=2}^{j-1} \frac{\hat\pi_t(\bx,a) - \hat\pi_{t-1}(\bx,a)}{T-t} - \frac{\hat\pi_{j-1}(\bx,a)}{T-1}\nonumber\\
  = & \frac{\hat\pi_1(\bx,a) - \hat\pi_{j-1}(\bx,a)}{T-1} + \sum_{t=2}^{j-1} \frac{(\hat\pi_t(\bx,a)-\hat\pi_{j-1}(\bx,a))-(\hat\pi_{t-1}(\bx,a)-\hat\pi_{j-1}(\bx,a))}{T-t}\nonumber\\
  = & -\sum_{t=1}^{j-2}\frac{\hat\pi_t(\bx,a)-\hat\pi_{j-1}(\bx,a)}{(T-t)(T-t-1)}.
\end{align*}
Therefore,
\begin{align*}
  \left(\tilde \pi_{T,j}(\bx,a) - \frac{\hat\pi_{j-1}(\bx,a)}{T-1}\right)^2 &=\sum_{t_1=1}^{j-2}\sum_{t_2=1}^{j-2}\frac{(\hat\pi_{t_1}(\bx,a)-\hat\pi_{j-1}(\bx,a))\cdot(\hat\pi_{t_2}(\bx,a)-\hat\pi_{j-1}(\bx,a))}{(T-t_1)(T-t_1-1)(T-t_2)(T-t_2-1)} \nonumber\\
  &\le \sum_{t_1=1}^{j-2} \frac{2|\hat\pi_{t_1}(\bx,a)-\hat\pi_{j-1}(\bx,a)|}{(T-t_1)(T-t_1-1)}\cdot\sum_{t_2=1}^{j-2}\frac{1}{(T-t_2)(T-t_2-1)}\nonumber\\
  &\le \sum_{t_1=1}^{j-2} \frac{2|\hat\pi_{t_1}(\bx,a)-\hat\pi_{j-1}(\bx,a)|}{(T-t_1)(T-t_1-1)(T-j+1)}
\end{align*}
Finally, since $|\pi_t(\bx,a)|<1$ for all $t, x, a$, 
\begin{align*}
\left|\tilde\pi_{T,j}(\bx,a)\right| &\le \frac{|\hat\pi_{j-1}|}{T-1} + \sum_{t=1}^{j-2}\frac{|\hat\pi_t(\bx,a)-\hat\pi_{j-1}(\bx,a)|}{(T-t)(T-t-1)} \nonumber \\
&\le \frac{1}{T-1} + 2\sum_{t=1}^{j-2}\frac{1}{(T-t)(T-t-1)} \nonumber \\
                                    &\le \frac{1}{T-1} + 2\left(\frac{1}{T-j+1}-\frac{1}{T-1}\right)\\
  & \le \frac{2}{T-j+1}
\end{align*}

\end{proof}

\subsection{Lemma~\ref{lemma:L1_approx}}
\begin{lemma}
  \label{lemma:L1_approx} Under
  Assumptions~\ref{ass:stationarity}--\ref{ass:fourth_moment}, we have,
  \begin{align*}
    \lim_{T'\rightarrow\infty} \sup_{T\ge T'}\frac{(T-1)^{\min(\delta,1)}}{\log T}\sum_{j=2}^T\sum_{a=1}^K \mathbb{E}\left[\left|\tilde\pi_{T,j}(\bX_j,a)-\frac{\hat\pi_{j-1}(\bX_j,a)}{T-1}\right|\right] < \infty,
  \end{align*}
  where $\tilde\pi_{T,j}$ is defined in Equation~\eqref{def:tilde}.
\end{lemma}
\begin{proof}
  Applying Lemma~\ref{lemma:centering}, we have, 
\begin{align}
  &\frac{(T-1)^{\min(\delta,1)}}{\log T}\sum_{j=2}^T\sum_{a=1}^K \mathbb{E}\left[\left|\tilde\pi_{T,j}(\bX_j,a)-\frac{\hat\pi_{j-1}(\bX_j,a)}{T-1}\right|\right] \nonumber\\
  \le&\frac{(T-1)^{\min(\delta,1)}}{\log T}\sum_{j=2}^T\sum_{t=1}^{j-2} \sup_{1\le a\le K}\mathbb{E}\left[\frac{|\hat\pi_t(\bX_j,a)-\hat\pi_{j-1}(\bX_j,a)|}{(T-t)(T-t-1)}\right]\nonumber\\
  \le&\frac{(T-1)^{\min(\delta,1)}}{\log T}\sum_{j=2}^T\sum_{t=1}^{j-2}\sum_{i=t+1}^{j-1} \sup_{1\le a\le K}\mathbb{E}\left[\frac{|\hat\pi_i(\bX_j,a)-\hat\pi_{i-1}(\bX_j,a)|}{(T-t)(T-t-1)}\right]\nonumber\\
  \le&\frac{(T-1)^{\min(\delta,1)}}{\log T}\sum_{j=2}^T\sum_{t=1}^{j-2}\sum_{i=t+1}^{j-1} \mathbb{E}\left[\frac{Q_i}{(T-t)(T-t-1)}\right]\nonumber\\
  \le&\frac{(T-1)^{\min(\delta,1)}}{\log T} \sup_{i\ge 1}\mathbb{E}\left[Q_ii^{1+\delta}\right]\sum_{j=2}^T\sum_{t=1}^{j-2}\sum_{i=t+1}^{j-1} \frac{i^{-1-\delta}}{(T-t)(T-t-1)}\nonumber\\
  \le&\frac{(T-1)^{\min(\delta,1)}}{\log T} \sup_{i\ge 1}\mathbb{E}\left[Q_ii^{1+\delta}\right]\sum_{t=1}^{T-2}\sum_{i=t+1}^{T-1}\sum_{j=i+1}^T \frac{i^{-1-\delta}}{(T-t)(T-t-1)}\nonumber\\
  \le&\frac{(T-1)^{\min(\delta,1)}}{\log T} \sup_{i\ge 1}\mathbb{E}\left[Q_ii^{1+\delta}\right]\sum_{t=1}^{T-2}\sum_{i=t+1}^{T-1}\frac{i^{-1-\delta}}{T-t}\nonumber\\
  \le&\frac{(T-1)^{\min(\delta,1)}}{\log T} \sup_{i\ge 1}\mathbb{E}\left[Q_ii^{1+\delta}\right]\cdot\frac{1}{\delta}\sum_{t=1}^{T-2}\frac{1}{(T-t)t^{\delta}}.\nonumber
\end{align}
Lemma~\ref{lemma:math_order} shows
$\sum_{t=1}^{T-2}\frac{1}{(T-t)t^{\delta}} = O\left(\frac{\log
    T}{T^{\min (\delta,1)}}\right)$, and therefore we have the desired
result,
\begin{align*}
  \lim_{T'\rightarrow\infty} \sup_{T\ge T'}\frac{(T-1)^{\min(\delta,1)}}{\log T}\sum_{j=2}^T\sum_{a=1}^K \mathbb{E}\left[\left|\tilde\pi_{T,j}(\bX_j,a)-\frac{\hat\pi_{j-1}(\bX_j,a)}{T-1}\right|\right] < \infty. 
\end{align*}
\end{proof}

\subsection{Lemma \ref{lemma:L2_approx}}
\begin{lemma}
  \label{lemma:L2_approx} Under
  Assumptions~\ref{ass:stationarity}--\ref{ass:fourth_moment},  we have,
  \begin{align*}
    \lim_{T\rightarrow\infty}(T-1)^{1+\eta}\sum_{j=2}^T\sum_{a=1}^K \mathbb{E}\left[\left(\tilde\pi_{T,j}(\bX_j,a)-\frac{\hat\pi_{j-1}(\bX_j,a)}{T-1}\right)^2\right] = 0,
  \end{align*}
  where $\tilde\pi_{T,j}$ is defined in Equation~\eqref{def:tilde}.
\end{lemma}
\begin{proof}
By Lemma~\ref{lemma:centering}, we have,
\begin{align*}
    &(T-1)^{1+\eta}\sum_{j=2}^T\mathbb{E}\left[\sum_{a=1}^K\left(\tilde \pi_{T,j}(\bX_j,a) - \frac{\hat\pi_{j-1}(\bX_j,a)}{T-1}\right)^2\right]\nonumber\\
    \le&(T-1)^{1+\eta}\sum_{j=2}^T\sum_{a=1}^K\mathbb{E}\left[\sum_{t_1=1}^{j-2} \frac{2|\hat\pi_{t_1}(\bX_j,a)-\hat\pi_{j-1}(\bX_j,a)|}{(T-t_1)(T-t_1-1)(T-j+1)}\right]\nonumber\\
    \le& 2(T-1)^{1+\eta}K\sum_{j=2}^T\sum_{t_1=1}^{j-2}\sup_{1\le a\le K}\mathbb{E}\left[\frac{|\hat\pi_{t_1}(\bX_j,a)-\hat\pi_{j-1}(\bX_j,a)|}{(T-t_1)(T-t_1-1)(T-j+1)}\right].
\end{align*}
Under Assumption~\ref{ass:clip_rate}, we have
$0\le \eta < \frac{\delta}{\delta+1}$. Therefore, there exists $\zeta$
such that
\begin{align*}
\frac{\eta}{\delta} < \zeta < 1 - \eta.
\end{align*}
Therefore, 
\begin{align}
    &2(T-1)^{1+\eta}K\sum_{j=2}^T\sum_{t_1=1}^{j-2}\sup_{1\le a\le K}\mathbb{E}\left[\frac{|\hat\pi_{t_1}(\bX_j,a)-\hat\pi_{j-1}(\bX_j,a)|}{(T-t_1)(T-t_1-1)(T-j+1)}\right]\label{011902}\\
    =&2(T-1)^{1+\eta}K\sum_{j=T^\zeta+2}^T\sum_{t_1=T^\zeta}^{j-2}\sup_{1\le a\le K}\mathbb{E}\left[\frac{|\hat\pi_{t_1}(\bX_j,a)-\hat\pi_{j-1}(\bX_j,a)|}{(T-t_1)(T-t_1-1)(T-j+1)}\right]\label{011912}\\
    &+2(T-1)^{1+\eta}K\sum_{j=2}^T\sum_{t_1=1}^{\min(j-2,T^{\zeta}-1)}\sup_{1\le a\le K}\mathbb{E}\left[\frac{|\hat\pi_{t_1}(\bX_j,a)-\hat\pi_{j-1}(\bX_j,a)|}{(T-t_1)(T-t_1-1)(T-j+1)}\right]\label{011922}
\end{align}
We will bound each of these two terms in turn.  The first term in
Equation~\eqref{011912} can be bounded with the following expression
\begin{align}
    &2(T-1)^{1+\eta}K\sum_{j=T^\zeta+2}^T\sum_{t_1=T^\zeta}^{j-2}\sup_{1\le a\le K}\mathbb{E}\left[\frac{|\hat\pi_{t_1}(\bX_j,a)-\hat\pi_{j-1}(\bX_j,a)|}{(T-t_1)(T-t_1-1)(T-j+1)}\right]\nonumber\\
    \le&2(T-1)^{1+\eta}K\sum_{j=T^\zeta+2}^T\sum_{t_1=T^\zeta}^{j-2}\sum_{i=t_1+1}^{j-1}\frac{\mathbb{E}[Q_i]}{(T-t_1)(T-t_1-1)(T-j+1)}\nonumber\\
    \le&2(T-1)^{1+\eta}K\sum_{t_1=T^\zeta}^{T-2}\sum_{i=t_1+1}^{T-1}\sum_{j=i+1}^T\frac{\mathbb{E}[Q_i]}{(T-t_1)(T-t_1-1)(T-j+1)}\nonumber\\
    \le&2(T-1)^{1+\eta}\cdot\log T \cdot K\sum_{t_1=T^\zeta}^{T-2}\sum_{i=t_1+1}^{T-1}\frac{\mathbb{E}[Q_i]}{(T-t_1)(T-t_1-1)}\nonumber\\
    \le&2(T-1)^{1+\eta}\cdot\log T \cdot K\cdot \sup_{i\ge 1} \mathbb{E}[Q_ii^{1+\delta}]\sum_{t_1=T^\zeta}^{T-2}\sum_{i=t_1+1}^{T-1}\frac{i^{-1-\delta}}{(T-t_1)(T-t_1-1)}\nonumber\\
    \le&2(T-1)^{1+\eta}\cdot\log T \cdot K\cdot \sup_{i\ge 1} \mathbb{E}[Q_ii^{1+\delta}]\sum_{t_1=T^\zeta}^{T-2}\frac{t_1^{-1-\delta}(T-1-t_1)}{(T-t_1)(T-t_1-1)}\nonumber\\
    \le&2(T-1)^{1+\eta}\cdot\log T \cdot K\cdot \sup_{i\ge 1} \mathbb{E}[Q_ii^{1+\delta}]\sum_{t_1=T^\zeta}^{T-2}\frac{1}{(T-t_1)t_1^{1+\delta}}\nonumber\\
    \le&2(T-1)^{\eta}\cdot\log T \cdot K\cdot \sup_{i\ge 1} \mathbb{E}[Q_ii^{1+\delta}]\left(\sum_{t_1=T^\zeta}^{T-2}\frac{1}{(T-t_1)t_1^{\delta}} + \sum_{t_1=T^\zeta}^{T-2}\frac{1}{t_1^{1+\delta}}\right).\nonumber
\end{align}
By Lemma~\ref{lemma:math_order}, we have,
\begin{align}
    \sum_{t_1=T^\zeta}^{T-2}\frac{1}{(T-t_1)t_1^{\delta}} = O\left(\frac{\log T}{T^{\min\{\delta,1\}}}\right),\ \ \sum_{t_1=T^\zeta}^{T-2}\frac{1}{t_1^{1+\delta}} = O\left(\frac{1}{T^{\zeta\delta}}\right)\nonumber
\end{align}
and $\zeta > \frac{\eta}{\delta}$, $\delta>\eta, \eta<1$.  Therefore,
we have the desired result,
\begin{align} 
(T-1)^\eta \log T \left(\sum_{t_1=T^\zeta}^{T-2}\frac{1}{(T-t_1)t_1^{\delta}} + \sum_{t_1=T^\zeta}^{T-2}\frac{1}{t_1^{1+\delta}}\right) = o(1),\nonumber
\end{align}
implying that the term in Equation~\eqref{011912} is $o(1)$.

Next, the term in Equation~\eqref{011922} can be bounded as,
\begin{align*}
    &2(T-1)^{1+\eta}K\sum_{j=2}^T\sum_{t_1=1}^{\min(j-2,T^{\zeta}-1)}\sup_{1\le a\le K}\mathbb{E}\left[\frac{|\hat\pi_{t_1}(\bX_j,a)-\hat\pi_{j-1}(\bX_j,a)|}{(T-t_1)(T-t_1-1)(T-j+1)}\right]\nonumber\\
    \le&4(T-1)^{1+\eta}K\sum_{j=2}^T\sum_{t_1=1}^{\min(j-2,T^{\zeta}-1)}\frac{1}{{(T-t_1)(T-t_1-1)(T-j+1)}}\nonumber\\
    \le&4(T-1)^{1+\eta}K\sum_{j=2}^T\frac{T^\zeta}{{(T-T^\zeta)(T-T^\zeta-1)(T-j+1)}}\nonumber\\
    \le&4(T-1)^{1+\eta}K\log T \cdot \frac{T^\zeta}{(T-T^\zeta-1)^2}
\end{align*}
Since $\zeta < 1-\eta$,
$$
(T-1)^{1+\eta} \log T \frac{T^\zeta}{(T-T^\zeta-1)^2} = o(1).
$$
Therefore, we have the desired result,
\begin{align*}
  \lim_{T\rightarrow\infty}(T-1)^{1+\eta}\sum_{j=2}^T\sum_{a=1}^K \mathbb{E}\left[\left(\tilde\pi_{T,j}(\bX_j,a)-\frac{\hat\pi_{j-1}(\bX_j,a)}{T-1}\right)^2\right] \stackrel{P}{\rightarrow} 0. 
\end{align*}

\end{proof}

\section{Bandit algorithms used in numerical experiments}
\label{app:bandit_algorithms}

In this appendix, we describe the linear contextual bandit algorithms
used in our numerical experiments.  They are $\epsilon$-greedy, UCB,
and Thompson Sampling algorithms.

\begin{algorithm}[H] \spacingset{1}
  \caption{The contextual $\epsilon$-greedy algorithm}\label{alg:epsilon}
 \KwIn{Exploration rates $\{c_t\}_{t=1}^\infty$. Initial policy $\pi_0$.}
  \KwOut{A bandit policy sequence $\{\hat\pi_t\}_{t=1}^T$ and
    corresponding bandit data $(\bX_t,A_t,R_t)$.}
%  Set $\hat\pi_0 = \pi_0$\; 
  \For{$t=1$ \KwTo $T$}{
    {Observe the context $\bX_t$}\; 
    {Sample an action $A_t$ from the policy $\hat\pi_{t-1}$}\;
    {Observe the reward $R_t$ under the context $\bX_t$ and action $A_t$}\;
    {Update the parameter estimation for $\bTheta$ as
      \begin{align}
        \widehat\bTheta_t = \left(\lambda I + \sum_{s=1}^t \Phi(\bX_s,
        A_s)\Phi(\bX_s,A_s)^\top\right)^{-1}\sum_{s=1}^t
        \Phi(\bX_s, A_s) R_s,\nonumber
      \end{align}}
    where $\lambda$ is a ridge regularization term\;
    {Update the policy as $$\hat\pi_t(\bx,a)=\begin{cases}
        \frac{c_t}{K-1}, & \text{if } a \not = \argmax_{\tilde a=1,2,...,K} \Phi(\bx,\tilde a)^\top\widehat\bTheta_t,\\
        1 - c_t, & \text{otherwise}.
      \end{cases};$$}\
  } 
\end{algorithm}

\begin{algorithm}[H] \spacingset{1}
  \caption{The contextual UCB algorithm}\label{alg:UCB}
  \KwIn{Exploration rates $\{c_t\}_{t=1}^\infty$. Initial policy
    $\pi_0$. Tuning parameter $\delta$ for the confidence set.}
   \KwOut{A bandit policy sequence $\{\hat\pi_t\}_{t=1}^T$ and
     corresponding bandit data $(\bX_t,A_t,R_t)$.}
 %  Set $\hat\pi_0 = \pi_0$\; 
   \For{$t=1$ \KwTo $T$}{
     {Observe the context $\bX_t$}\; 
     {Sample an action $A_t$ from the policy $\hat\pi_{t-1}$}\;
     {Observe the reward $R_t$ under the context $\bX_t$ and action $A_t$}\;
     {Update the parameter estimation for $\bTheta$ as \begin{align}
       \widehat\bTheta_t = \left(\lambda I + \sum_{s=1}^t \Phi(\bX_s,
                                                         A_s)\Phi(\bX_s,
                                                         A_s)^\top\right)^{-1}\sum_{s=1}^t
                                                         \Phi(\bX_s,
                                                         A_s) R_s,\nonumber
       \end{align}}
       where $\lambda$ is a ridge regularization term\;
     {Update the policy as $$\hat\pi_t(\bx,a) = \begin{cases}
      \frac{c_t}{K-1}, & \text{if } a \not = \argmax_{\tilde a=1,2,...,K, \widetilde\bTheta\in C_t}\Phi(\bx,\tilde a)^\top\widetilde\bTheta\\
      1-c_t, & \text{if } a = \arg\max_{\tilde a=1,2,...,K, \widetilde\bTheta\in C_t}\Phi(\bx,\tilde a)^\top\widetilde\bTheta
      \end{cases}$$}
    where $C_t$ is the $1-\delta$ confidence set whose definition is
    given by,
      \begin{align}
        C_t:= \left\{ \widetilde \bTheta \in \mathbb{R}^{dK} : \| \widehat{\bTheta}_t - \widetilde \bTheta \|_{\bar V_t} \leq R \sqrt{dK \log \left( \frac{1 + t L^2 / \lambda}{\delta} \right) + \lambda^{1/2} S} \right\}. \nonumber
        \end{align}
      \
     } 
\end{algorithm}

\begin{algorithm}[H] \spacingset{1}
  \caption{The contextual Thompson Sampling algorithm}\label{alg:TS}
  \KwIn{Exploration rates $\{c_t\}_{t=1}^\infty$. Initial policy $\pi_0$. A prior mean and covariance matrix of $\bTheta$, $\bmu_0,\boldsymbol\Sigma_0$.}
   \KwOut{A bandit policy sequence $\{\hat\pi_t\}_{t=1}^T$ and
     corresponding bandit data $(\bX_t,A_t,R_t)$.}
 %  Set $\hat\pi_0 = \pi_0$\; 
   \For{$t=1$ \KwTo $T$}{
     {Observes the context $\bX_t$}\; 
     {Sample an action $A_t$ from the policy $\hat\pi_{t-1}$}\;
     {Observe the reward $R_t$ under context $\bX_t$ and action $A_t$}\;
     {Update the posterior mean and covariance matrix for $\bTheta$ as \begin{align}
      \boldsymbol{\Sigma}_t =
                 \left(\boldsymbol\Sigma_0^{-1}+\sum_{s=1}^t\Phi(\bX_s,
                                                                         A_s)\Phi(\bX_s,
                                                                         A_s)^\top
                                                                         \right)^{-1},
                                                                         \quad
      \bmu_t = \boldsymbol\Sigma_t\left(\boldsymbol\Sigma_0^{-1}\bmu_0 + \sum_{s=1}^t
                                                                         \Phi(\bX_s, A_s) R_s\right).\nonumber
      \end{align} }
     Update the policy $$\hat\pi_{t}(\bx,a) = \begin{cases}
      \frac{c_t}{K-1}, \ \ \text{if}\ \  \mathbb{P}_{\bTheta\sim N(\bmu_t,\Sigma_t)}\left(\Phi(\bx,a)^\top\bTheta = \max_{\tilde a=1,2,...,K}\Phi(\bx,\tilde a)^\top\bTheta\right)<\frac{c_t}{K-1}\\
      1 - {c_t},\ \  \text{if}\ \  \mathbb{P}_{\bTheta\sim N(\bmu_t,\Sigma_t)}\left(\Phi(\bx,a)^\top\bTheta = \max_{\tilde a=1,2,...,K}\Phi(\bx,\tilde a)^\top\bTheta\right)>1-c_t\\
      \alpha_t\mathbb{P}_{\bTheta\sim N(\bmu_t,\Sigma_t)}\left(\Phi(\bx,a)^\top\bTheta = \max_{\tilde a=1,2,...,K}\Phi(\bx,\tilde a)^\top\bTheta\right),\ \  \text{otherwise}
     \end{cases}
     $$ 
     where $\alpha_t$ is a scaling parameter that makes the probability of choosing different arms sum to 1.
     } 
\end{algorithm}

\section{Proof of Theorem~\ref{thm:linear_contextual_bandit}}
\label{thm:linear_contextual_bandit_proof}

Define the following linear contextual bandit updates after the $t$th iteration,
\begin{align*}
\widehat\bTheta_t = \left(\lambda I + \sum_{s=1}^t
  \Phi(\bX_s,A_s)\Phi(\bX_s,A_s)^\top \right)^{-1}\left(\sum_{s=1}^t \Phi(\bX_s,A_s)R_s\right),
\end{align*}
where $\lambda$ is the ridge penalization parameter.
In the later part, we will denote 
\begin{align*}
\bar V_t = \left(\lambda I + \sum_{s=1}^t \Phi(\bX_s,A_s)\Phi(\bX_t,A_t)^\top\right).
\end{align*}

\begin{lemma}
  \label{lemma:infinite_sampling}
  Under Assumptions~\ref{ass:xbounded_bandit}--\ref{ass:clip_bandit},
  \begin{align*}
  &\mathbb{P}\left(\lambda_{\min}\left(\bar V_t\right) \le \lambda + \frac{1}{2}\lambda_{\min}\left(\mathbb{E}\left[\bX\bX^\top\right]\right)\sum_{s=1}^t c_{s-1} \right) \nonumber\\
    \le & d \cdot\exp\left(-\frac{\left(\sum_{s=1}^{t}c_{s-1}\right)^2\lambda_{\min}\left(\mathbb{E}\left[\bX\bX^\top\right]\right)^2}{256L^4t}\right) + d \cdot\exp\left(-\frac{\left(\sum_{s=1}^{t}c_{s-1}\right)^2\lambda_{\min}\left(\mathbb{E}\left[\bX\bX^\top\right]\right)^2}{512L^4\sum_{s=1}^t c_{s-1}^2}\right) \nonumber\\
    =& O\left(\exp\left(-t^{1-2\zeta}\right)\right)
  \end{align*}   
\end{lemma}
Proof is in Appendix~\ref{lemma:infinite_sampling_proof}.
\begin{corollary}
  \label{corollary:V-1}
  \begin{align*}
  &\mathbb{P}\left(\lambda_{\max}(\bar V_t^{-1}) \ge \frac{1}{\lambda + \frac{1}{2}\lambda_{\min}\left(\mathbb{E}\left[\bX\bX^\top\right]\right)\sum_{s=1}^t c_{s-1}}\right)\nonumber\\
  \le & d \cdot\exp\left(-\frac{\left(\sum_{s=1}^{t}c_{s-1}\right)^2\lambda_{\min}\left(\mathbb{E}\left[\bX\bX^\top\right]\right)^2}{256L^4t}\right) + d \cdot\exp\left(-\frac{\left(\sum_{s=1}^{t}c_{s-1}\right)^2\lambda_{\min}\left(\mathbb{E}\left[\bX\bX^\top\right]\right)^2}{512L^4\sum_{s=1}^t c_{s-1}^2}\right) \nonumber\\
  =&O\left(\exp\left(-t^{1-2\zeta}\right)\right)
  \end{align*}
\end{corollary}
\begin{proof}
  Notice $\bar V_t$ is a positive definite matrix and
  $\lambda_{\max}(\bar V_t^{-1}) = \frac{1}{\lambda_{\min}(\bar
    V_t)}$. Then, the application of
  Lemma~\ref{lemma:infinite_sampling} yields the desired result.
\end{proof}

We will now prove that the stability condition holds for each of the
three bandit algorithms above. In all of our proofs, we will use
$\pi^*(\bx,a)$ to denote the optimal policy under the context $\bx$
and action $a$.

\subsection{$\epsilon$-greedy algorithm}

\begin{proof} 
Notice for the $\epsilon$-greedy algorithm, 
\begin{align}
\hat\pi_t(\bX,a) = \begin{cases}
\frac{c_t}{K-1}, & \text{if } a \not = \argmax_{\tilde a=1,2,\ldots,K}\Phi(\bX,\tilde a)^\top\widehat\bTheta_t\\
1-c_t, & \text{if } a = \argmax_{\tilde a=1,2,\ldots,K}\Phi(\bX,\tilde a)^\top\widehat\bTheta_t
\end{cases}.
\end{align}
For any $\bx\in\mathcal{X}$, if $$
\argmax_{\tilde a=1,2,\ldots,K}\Phi(\bx,\tilde a)^\top\widehat\bTheta_t = \argmax_{\tilde a=1,2,\ldots,K}\Phi(\bx,\tilde a)^\top\widehat\bTheta_{t-1},
$$ 
then, we have
$\sup_{a=1,2,\ldots,K}|\hat\pi_t(\bx,a) - \hat\pi_{t-1}(\bx,a)|\le |c_t -
c_{t-1}|$. Therefore,
\begin{align}
  &\sup_{a=1,2,..,K}\mathbb{E}\left[|\hat\pi_t(\bX,a) - \hat\pi_{t-1}(\bX,a)|\right]\nonumber\\
  \le &\mathbb{E}\left[\sup_{a=1,2,..,K}|\hat\pi_t(\bX,a) - \hat\pi_{t-1}(\bX,a)|\right] \nonumber\\
  \le& \mathbb{P}\left(\argmax_{\tilde a=1,2,\ldots,K}\Phi(\bX,\tilde a)^\top\widehat\bTheta_t = \argmax_{\tilde a=1,2,\ldots,K}\Phi(\bX,\tilde a)^\top\widehat\bTheta_{t-1}\right)\times|c_t - c_{t-1}|\nonumber\\
  &+ \mathbb{P}\left(\argmax_{\tilde a=1,2,\ldots,K}\Phi(\bX,\tilde a)^\top\widehat\bTheta_t \not = \argmax_{\tilde a=1,2,\ldots,K}\Phi(\bX,\tilde a)^\top\widehat\bTheta_{t-1}\right)\nonumber\\
  \le& |c_t - c_{t-1}| + \mathbb{P}\left(\argmax_{\tilde a=1,2,\ldots,K}\Phi(\bX,\tilde a)^\top\widehat\bTheta_t \not = \argmax_{\tilde a=1,2,\ldots,K}\Phi(\bX,\tilde a)^\top\widehat\bTheta_{t-1}\right).\nonumber
\end{align} 
In addition, for any given $\bx \in \cX$ and $a\in\{1,2,\ldots,K\}$,
\begin{align*}
  \left|\Phi(\bx,a)^\top(\widehat\bTheta_t - \bTheta)\right| \le
  \frac{\Delta(\bx)}{3}\quad \text{and} \ \ &
                                              \left|\Phi(\bx,a)^\top(\widehat\bTheta_{t-1} - \bTheta)\right| \le \frac{\Delta(\bx)}{3}
\end{align*}
implies
$\argmax_{\tilde a=1,2,\ldots,K}\Phi(\bx,\tilde
a)^\top\widehat\bTheta_t = \argmax_{\tilde
  a=1,2,\ldots,K}\Phi(\bx,\tilde a)^\top\widehat\bTheta_{t-1}$. Recall
that $\Delta(\bx)$ is defined in Assumption~\ref{ass:uniform_gap} and
represents the gap of expected rewards between best arm and
second-best arm under context $\bx$. Intuitively, if both the
estimated expected reward under $\widehat\bTheta_t$ and
$\widehat\bTheta_{t-1}$ are within $\frac{\Delta(\bx)}{3}$ difference
of the actual expected reward, then the estimated best arm under
$\widehat\bTheta_t$ and $\widehat\bTheta_{t-1}$ should be the same.
Therefore,
\begin{align}
&\mathbb{P}\left(\argmax_{\tilde a=1,2,\ldots,K}\Phi(\bX,\tilde a)^\top\widehat\bTheta_t \not = \argmax_{\tilde a=1,2,\ldots,K}\Phi(\bX,\tilde a)^\top\widehat\bTheta_{t-1}\right)\nonumber\\
\le&\mathbb{P}\left(\exists a\in\{1,2,\ldots,K\}, \Bigr|\Phi(\bX,a)^\top(\widehat\bTheta_t - \bTheta)\Bigr| > \frac{\Delta(\bX)}{3}\ \ \text{or}\ \  \Bigr|\Phi(\bX,a)^\top(\widehat\bTheta_{t-1} - \bTheta)\Bigr| > \frac{\Delta(\bX)}{3}\right)\nonumber\\
\le &\sum_{a=1}^K \mathbb{P}\left( \Bigr|\Phi(\bX,a)^\top(\widehat\bTheta_t - \bTheta)\Bigr| \ge \frac{\Delta(\bX)}{3}\right) + \sum_{a=1}^K \mathbb{P}\left( \Bigr|\Phi(\bX,a)^\top(\widehat\bTheta_{t-1} - \bTheta)\Bigr| \ge \frac{\Delta(\bX)}{3}\right)\nonumber\\
\le &\sum_{a=1}^K \mathbb{P}\left( \Bigr|\Phi(\bX,a)^\top\bar V_t^{\frac{1}{2}}\bar V_t^{-\frac{1}{2}}(\widehat\bTheta_t - \bTheta)\Bigr| \ge \frac{\Delta(\bX)}{3}\right) + \sum_{a=1}^K \mathbb{P}\left( \Bigr|\Phi(\bX,a)^\top\bar V_{t-1}^{\frac{1}{2}}\bar V_{t-1}^{-\frac{1}{2}}(\widehat\bTheta_{t-1} - \bTheta)\Bigr| \ge \frac{\Delta(\bX)}{3}\right)\nonumber\\
\le &\sum_{a=1}^K \mathbb{P}\left( \Bigr|\Phi(\bX,a)^\top\bar V_t^{-1}\Phi(\bX,a)\Bigr|\cdot \Bigr|(\widehat\bTheta_t - \bTheta)^\top\bar V_t(\widehat\bTheta_t - \bTheta)\Bigr| \ge \frac{\Delta(\bX)^2}{9}\right)\nonumber\\
  &+ \sum_{a=1}^K \mathbb{P}\left( \Bigr|\Phi(\bX,a)^\top\bar V_{t-1}^{-1}\Phi(\bX,a)\Bigr|\cdot \Bigr|(\widehat\bTheta_{t-1} - \bTheta)^\top\bar V_{t-1}(\widehat\bTheta_{t-1} - \bTheta)\Bigr| \ge \frac{\Delta(\bX)^2}{9}\right).\nonumber
\end{align}
Notice for any $a$,
\begin{align}
  &\mathbb{P}\left( \Bigr|\Phi(\bX,a)^\top\bar V_t^{-1}\Phi(\bX,a)\Bigr|\cdot \Bigr|(\widehat\bTheta_t - \bTheta)^\top\bar V_t(\widehat\bTheta_t - \bTheta)\Bigr| \ge \frac{\Delta(\bX)^2}{9}\right)\nonumber\\
  \le& \mathbb{P}\left( \Bigr|\Phi(\bX,a)^\top\bar V_t^{-1}\Phi(\bX,a)\Bigr| \ge \frac{\Delta(\bX)}{3\sqrt{t}}\right) + \mathbb{P}\left(\Bigr|(\widehat\bTheta_t - \bTheta)^\top\bar V_t(\widehat\bTheta_t - \bTheta)\Bigr| \ge \frac{\sqrt{t}\Delta(\bX)}{3}\right)\label{01303}
\end{align}
We bound each of these two terms in turn.  For the first term, 
\begin{align}
\mathbb{P}\left( \Bigr|\Phi(\bX,a)^\top\bar V_t^{-1}\Phi(\bX,a)\Bigr| \ge \frac{\Delta(\bX)}{3\sqrt{t}}\right)
\le& \mathbb{P}\left( \norm{\Phi(\bX,a)}^2   \lambda_{\max}\left(\bar V_t^{-1}\right) \ge \frac{\Delta(\bX)}{3\sqrt{t}}\right)\nonumber\\
\le&\mathbb{P}\left(\lambda_{\max}(\bar V_t^{-1})\ge
     \frac{\Delta(\bX)}{3 L^2 \sqrt{t}} \right) \nonumber\\
  =  & O\left(\exp\left(-t^{1-2\zeta}\right)\right),\label{02011}
\end{align}
which follow from Corollary~\ref{corollary:V-1} and
Assumption~\ref{ass:clip_bandit}.  For the second term, according to
\cite{abbasi2011improved}, for any $\delta>0$,
\begin{align}
  &\mathbb{P}\left(\Bigr|(\widehat\bTheta_t - \bTheta)^\top\bar V_t(\widehat\bTheta_t - \bTheta)\Bigr| \ge U \sqrt{dK\log \left( \frac{t^2L/\lambda + 1}{\delta} \right) + \lambda^{1/2}S}
  \right)\le \delta \nonumber,
\end{align}
where $U$ is the parameter for the sub-Gaussian distribution of
rewards and $S$ is the bound for norm of the true parameter, i.e.,
$||\bTheta|| \le S$.  Therefore, let
$\delta = \exp(-t^{\frac{1}{2}})$,
\begin{align}
  &\mathbb{P}\left(\Bigr|(\widehat\bTheta_t - \bTheta)^\top\bar V_t(\widehat\bTheta_t - \bTheta)\Bigr| \ge U \sqrt{dt^{\frac{1}{2}}\log \left( {t^2L/\lambda + 1} \right) + \lambda^{1/2}}
  \right)\le \exp(-t^{\frac{1}{2}}). \nonumber
\end{align}
When $t$ is sufficiently large, $\frac{1}{2}\inf_{\bx\in\mathcal{X}}\Delta(\bx)\sqrt{t}\ge U \sqrt{dt^{\frac{1}{2}}\log \left( {t^2L/\lambda + 1} \right) + \lambda^{1/2}}$ and for any $\bx\in\mathcal{X}$,
\begin{align}
  \mathbb{P}\left(\Bigr|(\widehat\bTheta_t - \bTheta)^\top\bar V_t(\widehat\bTheta_t - \bTheta)\Bigr| \ge \frac{\sqrt{t}\Delta(\bX)}{2}\right) \le \exp(-t^{\frac{1}{2}}) = O\left(\exp\left(-\sqrt{t}\right)\right). \label{013022}
\end{align}
Plugging Equations~\eqref{02011}~and~\eqref{013022} into
Equation~\eqref{01303}, we have
\begin{align}
  &\mathbb{P}\left( \Bigr|\Phi(\bX,a)^\top\bar V_t^{-1}\Phi(\bX,a)\Bigr|\cdot \Bigr|(\widehat\bTheta_t - \bTheta)^\top\bar V_t(\widehat\bTheta_t - \bTheta)\Bigr| \ge \frac{\Delta(\bX)^2}{4}\right)=O\left(\exp\left(-t^{\min\{1-2\zeta,\ 0.5\}}\right)\right).\nonumber
\end{align}
Therefore,
$$
\sup_{a=1,2,..,K}\mathbb{E}\left[|\hat\pi_t(\bX,a) - \hat\pi_{t-1}(\bX,a)|\right] = O\left(\exp\left(-t^{\min\{1-2\zeta, 0.5\}}\right)\right) + |c_t - c_{t-1}|.
$$ 
Under Assumption~\ref{ass:clip_bandit},
$|c_t-c_{t-1}| = |t^{-\zeta} - (t-1)^{-\zeta}|$. When $\zeta = 0$,
this term is 0 and when $0<\zeta<\frac{1}{2}$,
$|c_t-c_{t-1}| = O\left(\frac{1}{t^{1+\zeta}}\right)$. Therefore, the
stability condition is satisfied.

\end{proof}

\subsection{Upper Confidence Bound}
For UCB algorithms, we denote the $(1-\delta)$ confidence set of $\hat\Theta_t$ as $C_t$:
\begin{align*}
C_t:= \left\{ \tilde \bTheta \in \mathbb{R}^{dK} : \| \hat{\bTheta}_t - \tilde \bTheta \|_{\bar V_t} \leq R \sqrt{dK \log \left( \frac{1 + t L^2 / \lambda}{\delta} \right) + \lambda^{1/2} S} \right\},
\end{align*}
and
\begin{align*}
\hat\pi_t(\bX,a) = \begin{cases}
\frac{c_t}{K-1}, & \text{if } a \not = \argmax_{\tilde a=1,2,\ldots,K, \widetilde\bTheta\in C_t}\Phi(\bX,\tilde a)^\top\widetilde\bTheta\\
1-c_t, & \text{if } a = \argmax_{\tilde a=1,2,\ldots,K, \widetilde\bTheta\in C_t}\Phi(\bX,\tilde a)^\top\widetilde\bTheta
\end{cases}.
\end{align*}
For any $\bx\in\mathcal{X}$, if $$
\argmax_{\tilde a=1,2,\ldots,K,\widetilde\bTheta\in C_t}\Phi(\bx,\tilde a)^\top\widetilde\bTheta = \argmax_{\tilde a=1,2,\ldots,K,\widetilde\bTheta\in C_{t-1}}\Phi(\bx,\tilde a)^\top\widetilde\bTheta,
$$ 
then $\sup_{a=1,2,\ldots,K}|\hat\pi_t(\bx,a) - \hat\pi_{t-1}(\bx,a)|\le |c_t - c_{t-1}|$. Therefore,
\begin{align}
  &\sup_{a=1,2,..,K}\mathbb{E}\left[|\hat\pi_t(\bX,a) - \hat\pi_{t-1}(\bX,a)|\right]\nonumber\\
  \le &\mathbb{E}\left[\sup_{a=1,2,..,K}|\hat\pi_t(\bX,a) - \hat\pi_{t-1}(\bX,a)|\right] \nonumber\\
  \le& \mathbb{P}\left(\argmax_{\tilde a=1,2,\ldots,K,\widetilde\bTheta\in C_t}\Phi(\bX,\tilde a)^\top\widetilde\bTheta = \argmax_{\tilde a=1,2,\ldots,K,\widetilde\bTheta\in C_{t-1}}\Phi(\bX,\tilde a)^\top\widetilde\bTheta\right)\times|c_t - c_{t-1}|\nonumber\\
  &+ \mathbb{P}\left(\argmax_{\tilde a=1,2,\ldots,K,\widetilde\bTheta\in C_t}\Phi(\bX,\tilde a)^\top\widetilde\bTheta \not = \argmax_{\tilde a=1,2,\ldots,K,\widetilde\bTheta\in C_{t-1}}\Phi(\bX,\tilde a)^\top\widetilde\bTheta\right)\nonumber\\
  \le& |c_t - c_{t-1}| + \mathbb{P}\left(\argmax_{\tilde a=1,2,\ldots,K,\widetilde\bTheta\in C_t}\Phi(\bX,\tilde a)^\top\widetilde\bTheta \not = \argmax_{\tilde a=1,2,\ldots,K,\widetilde\bTheta\in C_{t-1}}\Phi(\bX,\tilde a)^\top\widetilde\bTheta\right).\nonumber
\end{align}
For any given $\bx \in \cX$, $a\in\{1,2,\ldots,K\}$,
$\widetilde\bTheta_t\in C_t$, $\widetilde\bTheta_{t-1}\in C_{t-1}$, we
have,
$$
\Bigr|\Phi(\bx,a)^\top(\widetilde\bTheta_t - \bTheta)\Bigr| \le
\frac{\Delta(\bx)}{3},\quad \text{and} \ \
\Bigr|\Phi(\bx,a)^\top(\widetilde\bTheta_{t-1} - \bTheta)\Bigr|
\le \frac{\Delta(\bx)}{3},
$$
which implies
$\argmax_{\tilde a=1,2,\ldots,K,\widetilde\bTheta\in
  C_t}\Phi(\bX,\tilde a)^\top\widetilde\bTheta  = \argmax_{\tilde
  a=1,2,\ldots,K,\widetilde\bTheta\in C_{t-1}}\Phi(\bX,\tilde
a)^\top\widetilde\bTheta$.  Therefore,
\begin{align}
&\mathbb{P}\left(\arg\max_{\tilde a=1,2,\ldots,K,\widetilde\bTheta\in C_t}\Phi(\bX,\tilde a)^\top\widetilde\bTheta \not = \argmax_{\tilde a=1,2,\ldots,K,\widetilde\bTheta\in C_{t-1}}\Phi(\bX,\tilde a)^\top\widetilde\bTheta\right)\nonumber\\
 \le &\sum_{a=1}^K \mathbb{P}\left( \sup_{\widetilde\bTheta_t\in C_t}\Bigr|\Phi(\bX,a)^\top\bar V_t^{\frac{1}{3}}\bar V_t^{-\frac{1}{2}}(\widetilde\bTheta_t - \widehat \bTheta_t + \widehat \bTheta_t - \bTheta)\Bigr| \ge \frac{\Delta(\bX)}{3}\right) \nonumber\\
 &+ \sum_{a=1}^K \mathbb{P}\left( \sup_{\widetilde\bTheta_{t-1}\in C_{t-1}}\Bigr|\Phi(\bX,a)^\top\bar V_{t-1}^{\frac{1}{2}}\bar V_{t-1}^{-\frac{1}{2}}(\widetilde\bTheta_{t-1} - \widehat \bTheta_{t-1} + \widehat \bTheta_{t-1}- \bTheta)\Bigr| \ge \frac{\Delta(\bX)}{3}\right)\nonumber\\
 \le&\sum_{a=1}^K \mathbb{P}\left(\sup_{\widetilde\bTheta_t\in C_t} \Bigr|\Phi(\bX,a)^\top\bar V_t^{\frac{1}{2}}\bar V_t^{-\frac{1}{2}}(\widetilde\bTheta_t - \widehat \bTheta_t)\Bigr| \ge \frac{\Delta(\bX)}{6}\right) +\sum_{a=1}^K \mathbb{P}\left( \Bigr|\Phi(\bX,a)^\top\bar V_t^{\frac{1}{2}}\bar V_t^{-\frac{1}{2}}(\widehat \bTheta_t - \bTheta)\Bigr| \ge \frac{\Delta(\bX)}{6}\right) \nonumber\\
 &+ \sum_{a=1}^K \mathbb{P}\left( \sup_{\widetilde\bTheta_{t-1}\in C_{t-1}}\Bigr|\Phi(\bX,a)^\top\bar V_{t-1}^{\frac{1}{2}}\bar V_{t-1}^{-\frac{1}{2}}(\widetilde\bTheta_{t-1} - \widehat \bTheta_{t-1})\Bigr| \ge \frac{\Delta(\bX)}{6}\right) \nonumber\\
 &+\sum_{a=1}^K \mathbb{P}\left( \Bigr|\Phi(\bX,a)^\top\bar V_{t-1}^{\frac{1}{2}}\bar V_{t-1}^{-\frac{1}{2}}(\widehat \bTheta_{t-1} - \bTheta)\Bigr| \ge \frac{\Delta(\bX)}{6}\right). \nonumber
\end{align}
The previous results for the $\epsilon$-greedy algorithm have shown,
\begin{align}
&\sum_{a=1}^K \mathbb{P}\left(\sup_{\widetilde\bTheta_t\in C_t} \Bigr|\Phi(\bX,a)^\top\bar V_t^{\frac{1}{2}}\bar V_t^{-\frac{1}{2}}(\tilde\Theta_t - \widehat \bTheta_t)\Bigr| \ge \frac{\Delta(\bX)}{6}\right)+\sum_{a=1}^K \mathbb{P}\left( \Bigr|\Phi(\bX,a)^\top\bar V_{t-1}^{\frac{1}{2}}\bar V_{t-1}^{-\frac{1}{2}}(\widehat \bTheta_{t-1} - \bTheta)\Bigr| \ge \frac{\Delta(\bX)}{6}\right)\nonumber\\
=&O\left(\exp\left(-t^{\min\{1-2\zeta,\ 0.5\}}\right)\right).\nonumber
\end{align}
We can bound the first term by as follows,
\begin{align}
  &\mathbb{P}\left( \sup_{\widetilde\bTheta_t\in C_t}\Bigr|\Phi(\bX,a)^\top\bar V_t^{\frac{1}{2}}\bar V_t^{-\frac{1}{2}}(\widetilde\bTheta_t - \widehat \bTheta_t )\Bigr| \ge \frac{\Delta(\bX)}{6}\right) \nonumber\\
  \le& \mathbb{P}\left(\sup_{\widetilde\bTheta_t\in C_t} \Bigr|\Phi(\bX,a)^\top\bar V_t^{-1}\Phi(\bX,a)\Bigr|\cdot \Bigr|(\widetilde\bTheta_t - \widehat \bTheta_t)^\top\bar V_t(\tilde\Theta_t - \widehat \bTheta_t)\Bigr| \ge \frac{\Delta(\bX)^2}{36}\right)\nonumber\\
  \le& \mathbb{P}\left( \Bigr|\Phi(\bX,a)^\top\bar V_t^{-1}\Phi(\bX,a)\Bigr| \ge \frac{\Delta(\bX)}{6\sqrt{t}}\right) + \mathbb{P}\left(\sup_{\widetilde\bTheta_t\in C_t}\Bigr|(\hat\Theta_t - \widetilde\bTheta_t)^\top\bar V_t(\widehat\bTheta_t - \widetilde\bTheta_t)\Bigr| \ge \frac{\sqrt{t}\Delta(\bX)}{6}\right)\nonumber
\end{align}
Since by the definition of $\widetilde\bTheta_t$, 
\begin{align*}
  \sup_{\widetilde\bTheta_t\in C_t} \| \hat{\bTheta} - \tilde \bTheta \|_{\bar V_t} \leq U \sqrt{dK \log \left( \frac{1 + t L^2 / \lambda}{\delta} \right) + \lambda^{1/2} S},
\end{align*}
we can use the exact same proof as for the $\epsilon$-greedy algorithm to show that 
\begin{align*}
  & \mathbb{P}\left( \Bigr|\Phi(\bX,a)^\top\bar
  V_t^{-1}\Phi(\bX,a)\Bigr| \ge \frac{\Delta(\bX)}{6\sqrt{t}}\right) +
  \mathbb{P}\left(\Bigr|(\widehat\bTheta_t -
  \widetilde\bTheta_t)^\top\bar V_t(\widehat\bTheta_t -
  \widetilde\bTheta_t)\Bigr| \ge \frac{\sqrt{t}\Delta(\bX)}{6}\right)
  \\
  = & O\left(\exp\left(-t^{\min\{1-2\zeta,\ 0.5\}}\right)\right) \nonumber
\end{align*}
and 
\begin{align*}
  \mathbb{P}\left( \sup_{\widetilde\bTheta_t\in C_t}\Bigr|\Phi(\bX,a)^\top\bar V_t^{\frac{1}{2}}\bar V_t^{-\frac{1}{2}}(\widetilde\bTheta_t - \hat \bTheta_t )\Bigr| \ge \frac{\Delta(\bX)}{6}\right)= O\left(\exp\left(-t^{\min\{1-2\zeta,\ 0.5\}}\right)\right).
\end{align*}
Similarly,
\begin{align}
  &\mathbb{P}\left( \sup_{\widetilde\bTheta_{t-1}\in C_{t-1}}\Bigr|\Phi(\bX,a)^\top\bar V_{t-1}^{\frac{1}{2}}\bar V_{t-1}^{-\frac{1}{2}}(\widetilde\bTheta_{t-1} - \hat \bTheta_{t-1})\Bigr| \ge \frac{\Delta(\bX)}{6}\right) = O\left(\exp\left(-t^{\min\{1-2\zeta,\ 0.5\}}\right)\right).\nonumber
\end{align}
Thus, we have
$$\sup_{a=1,2,..,K}\mathbb{E}\left[|\hat\pi_t(\bX,a) -
  \hat\pi_{t-1}(\bX,a)|\right] = O\left(\exp\left(-t^{\min\{1-2\zeta,
      0.5\}}\right)\right) + |c_t-c_{t-1}|,$$ implying that the
stability condition is satisfied.

\subsection{Thompson Sampling}

Suppose the prior distribution of $\boldsymbol\bTheta$ is a Gaussian
distribution with mean $\bmu_0$ and covariance matrix
$\boldsymbol\Sigma_0$. For Thompson Sampling, after the $t$-th
iteration, the posterior distribution of $\boldsymbol\Theta$ is a
Gaussian distribution with mean $\bmu_t$ and covariance matrix
$\boldsymbol\Sigma_t$, where
\begin{align*}
\bmu_t &= \boldsymbol\Sigma_t\left(\boldsymbol\Sigma_0^{-1}\bmu_0 +
         \sum_{s=1}^t \Phi(\bX_s, A_s) R_s\right),\nonumber\\
\boldsymbol\Sigma_t &=
                      \left(\boldsymbol\Sigma_0^{-1}+\sum_{s=1}^t\Phi(\bX_s,
                      A_s)\Phi(\bX_s, A_s)^\top\right)^{-1}.
\end{align*} 
For simplicity, assume the prior $\bmu_0=0$ and
$\boldsymbol\Sigma_0^{-1} = \bV_0 = \lambda I$. Then, 
\begin{align*}
\bmu_t &= \widehat\Theta_t = \left(\lambda I + \sum_{s=1}^t \Phi(\bX_s,
         A_s)\Phi(\bX_s, A_s)^\top\right)^{-1}\sum_{s=1}^t
         \Phi(\bX_s, A_s) R_s,\nonumber\\
\boldsymbol\Sigma_t &= \left(\lambda I + \sum_{s=1}^t \Phi(\bX_s, A_s)\Phi(\bX_s, A_s)^\top\right)^{-1}.
\end{align*}
Therefore,  
\begin{align}
  \hat\pi_t(\bx,a) = \begin{cases}
    \frac{c_t}{K-1}, & \text{if } \mathbb{P}_{\bgamma_t}\left(\bgamma_t^\top\Phi(\bx,a) = \max_{1\le a\le K} \bgamma_t^\top\Phi(\bx,a)\right) < \frac{c_t}{K-1}\\
    \alpha_t\mathbb{P}_{\bgamma_t}\left(\bgamma_t^\top\Phi(\bx,a) = \max_{1\le a\le K} \bgamma_t^\top\Phi(\bx,a)\right) & \text{if } \frac{c_t}{K-1}\le\mathbb{P}_{\bgamma_t}\left(\bgamma_t^\top\Phi(\bx,a) = \max_{1\le a\le K} \bgamma_t^\top\Phi(\bx,a)\right)\le 1-c_t\\
    1-c_t, & \text{if } \mathbb{P}_{\bgamma_t}\left(\bgamma_t^\top\Phi(\bx,a) = \max_{1\le a\le K} \bgamma_t^\top\Phi(\bx,a)\right) > 1-c_t
  \end{cases} \nonumber
\end{align} 
where $\bgamma_t\sim N(\bmu_t,\Sigma_t)$.  Thus, for any
$a\not = a^*(\bx)$ where $a^*(\bx)$ is defined in
Assumption~\ref{ass:uniform_gap} and is the optimal arm under context
$\bx$, we have,
\begin{align*}
  & \hat\pi_t(\bx,a) > \frac{c_t}{K-1} \nonumber\\
\Longleftrightarrow \ &  \mathbb{P}_{\bgamma_t}\left(\bgamma_t^\top\Phi(\bx,a) = \max_{1\le a\le K} \bgamma_t^\top\Phi(\bx,a)\right) > \frac{c_t}{K-1} \nonumber\\
\Longleftrightarrow \ & \mathbb{P}_{\bgamma_t}\left(\bgamma_t^\top\Phi(\bx,a)> \bgamma_t^\top\Phi(\bx, a^*(\bx))\right) > \frac{c_t}{K-1} \nonumber\\
\Longleftrightarrow \ &   \mathbb{P}_{\bgamma_t}\left((\bgamma_t^\top - \bTheta^\top)\Phi(\bx,a)> (\bgamma_t^\top - \bTheta^\top)\Phi(\bx, a^*(\bx)) + \bTheta^\top(\Phi(\bx,a^*(\bx))-\Phi(\bx,a))\right) > \frac{c_t}{K-1} \nonumber\\
\Longleftrightarrow \ &  \mathbb{P}_{\bgamma_t}\left((\bgamma_t^\top - \bTheta^\top)\Phi(\bx,a)> (\bgamma_t^\top - \bTheta^\top)\Phi(\bx, a^*(\bx)) + \Delta(\bx)\right) > \frac{c_t}{K-1} \nonumber\\
\Longleftrightarrow \ &  \mathbb{P}_{\bgamma_t}\left((\bgamma_t^\top - \bTheta^\top)(\Phi(\bx,a)-\Phi(\bx, a^*(\bx)))> \inf_{\bx\in\mathcal{X}}\Delta(\bx)\right) > \frac{c_t}{K-1}. \nonumber
\end{align*}
On the other hand, by Markov inequality, we have,
\begin{align}
  &\mathbb{P}_{\bgamma_t}\left((\bgamma_t^\top - \bTheta^\top)(\Phi(\bx,a)-\Phi(\bx, a^*(\bx)))> \inf_{\bx\in\mathcal{X}}\Delta(\bx)\right)\nonumber\\
\le&\frac{\mathbb{E}_{\bgamma_t}\left[\left\{(\bgamma_t^\top - \bTheta^\top)(\Phi(\bx,a)-\Phi(\bx, a^*(\bx)))\right\}^2\right]}{|\inf_{\bx\in\mathcal{X}}\Delta(\bx)|^2}\nonumber\\
\le&\frac{\left\{(\mu_t^\top - \bTheta^\top)(\Phi(\bx,a)-\Phi(\bx, a^*(\bx)))\right\}^2+(\Phi(\bx,a)-\Phi(\bx, a^*(\bx)))^\top\boldsymbol\Sigma_t(\Phi(\bx,a)-\Phi(\bx, a^*(\bx)))}{|\inf_{\bx\in\mathcal{X}}\Delta(\bx)|^2}\nonumber\\
\le&\frac{L^2||\mu_t - \bTheta||^2+L^2\lambda_{\max}(\boldsymbol\Sigma_t)}{|\inf_{\bx\in\mathcal{X}}\Delta(\bx)|^2}\nonumber
\end{align}
Therefore, we have,
\begin{align}
  \hat\pi_t(\bx,a) > \frac{c_t}{K-1} \Rightarrow  \frac{L^2||\mu_t - \bTheta||^2+L^2\lambda_{\max}(\Sigma_t)}{|\inf_{\bx\in\mathcal{X}}\Delta(\bx)|^2} &> \frac{c_t}{K-1}. \nonumber
\end{align}
Utilizing this result, we obtain,
\begin{align*}
  &\sup_{a=1,2,..,K}\mathbb{E}\left[|\hat\pi_t(\bX,a) - \hat\pi_{t-1}(\bX,a)|\right]\nonumber\\
  \le &\mathbb{E}\left[\sup_{a=1,2,..,K}|\hat\pi_t(\bX,a) - \hat\pi_{t-1}(\bX,a)|\right] \nonumber\\
  \le& |c_t-c_{t-1}|\mathbb{P}\left(\forall a\not = a^*(\bX),\hat\pi_t(\bX,a) = \frac{c_t}{K-1},\hat\pi_{t-1}(\bX,a)=\frac{c_{t-1}}{K-1} \right)\nonumber\\
  &+\mathbb{P}\left(\exists a \not = a^*(\bX),\hat\pi_t(\bX,a) > \frac{c_t}{K-1}\ \ \text{or}\ \ \hat\pi_{t-1}(\bX,a)>\frac{c_{t-1}}{K-1} \right)\nonumber\\
  \le&|c_t-c_{t-1}| + \mathbb{P}\left(\exists a \not = a^*(\bX),\hat\pi_t(\bX,a) > \frac{c_t}{K-1}\ \ \text{or}\ \ \hat\pi_{t-1}(\bX,a)>\frac{c_{t-1}}{K-1} \right)\nonumber\\
  \le&|c_t-c_{t-1}| + K \mathbb{P}\left( \frac{L^2||\mu_t - \bTheta||^2+L^2\lambda_{\max}(\boldsymbol\Sigma_t)}{|\inf_{\bx\in\mathcal{X}}\Delta(\bx)|^2} > \frac{c_t}{K-1}\right) + K \mathbb{P}\left( \frac{L^2||\mu_{t-1} - \bTheta||^2+L^2\lambda_{\max}(\boldsymbol\Sigma_{t-1})}{|\inf_{\bx\in\mathcal{X}}\Delta(\bx)|^2} > \frac{c_{t-1}}{K-1}\right)
\end{align*}
Notice the definition of $\boldsymbol\mu_t$ is the same as that of
$\widehat\bTheta_t$ in the $\epsilon$-greedy algorithm.  Similarly,
the definition of $\boldsymbol\Sigma_t$ is the same as $\bar V_t^{-1}$
in the $\epsilon$-greedy algorithm.  Therefore, using the results in the $\epsilon$-greedy algorithm:
\begin{align}
  \mathbb{P}\left(\Bigr|(\widehat\bTheta_t - \bTheta)^\top\bar V_t(\widehat\bTheta_t - \bTheta)\Bigr| \ge U \sqrt{dK\log \left( \frac{t^2L/\lambda + 1}{\delta} \right) + \lambda^{1/2}S}
  \right)&\le \delta \nonumber\\
\Longrightarrow \  \mathbb{P}\left(\Bigr|(\widehat\bTheta_t - \bTheta)^\top(\widehat\bTheta_t - \bTheta)\Bigr|\cdot\lambda_{\min}(\bar V_t) \ge U \sqrt{dK\log \left( \frac{t^2L/\lambda + 1}{\delta} \right) + \lambda^{1/2}S}
  \right)&\le \delta \label{0202temp1}
\end{align}
and 
\begin{align*}
  \mathbb{P}\left(\lambda_{\max}(\bar V_t^{-1}) \ge \frac{1}{\lambda + \frac{1}{2}\lambda_{\min}\left(\mathbb{E}\left[\bX\bX^\top\right]\right)\sum_{s=1}^t c_{s-1}}\right) = O\left(\exp\left(-t^{1-2\zeta}\right)\right),
\end{align*}
we can bound the probability $\mathbb{P}\left( \frac{L^2||\mu_t - \bTheta||^2+L^2\lambda_{\max}(\boldsymbol\Sigma_t)}{|\inf_{\bx\in\mathcal{X}}\Delta(\bx)|^2} > \frac{c_t}{K-1}\right)$ as below:
\begin{align} 
&~\mathbb{P}\left( \frac{L^2||\mu_t - \bTheta||^2+L^2\lambda_{\max}(\boldsymbol\Sigma_t)}{|\inf_{\bx\in\mathcal{X}}\Delta(\bx)|^2} > \frac{c_t}{K-1}\right)\nonumber\\
=&~\mathbb{P}\left( \frac{L^2||\widehat\bTheta_t - \bTheta||^2+L^2\lambda_{\max}(\bar V_t^{-1})}{|\inf_{\bx\in\mathcal{X}}\Delta(\bx)|^2} > \frac{c_t}{K-1}\right)\nonumber\\
\le&~\mathbb{P}\left( \frac{L^2||\widehat\bTheta_t - \bTheta||^2+L^2\lambda_{\max}(\bar V_t^{-1})}{|\inf_{\bx\in\mathcal{X}}\Delta(\bx)|^2} > \frac{c_t}{(K-1)}, \lambda_{\max}(\bar V_t^{-1}) < \frac{1}{\lambda + \frac{1}{2}\lambda_{\min}\left(\mathbb{E}\left[\bX\bX^\top\right]\right)\sum_{s=1}^t c_{s-1}}\right)\nonumber\\
&~+\mathbb{P}\left(\lambda_{\max}(\bar V_t^{-1}) \ge \frac{1}{\lambda + \frac{1}{2}\lambda_{\min}\left(\mathbb{E}\left[\bX\bX^\top\right]\right)\sum_{s=1}^t c_{s-1}}\right)\nonumber\\
\le&~\mathbb{P}\left( \frac{L^2||\widehat\bTheta_t - \bTheta||^2\lambda_{\min}(\bar V_t)+L^2}{|\inf_{\bx\in\mathcal{X}}\Delta(\bx)|^2} > \frac{c_t\left(\lambda + \frac{1}{2}\lambda_{\min}\left(\mathbb{E}\left[\bX\bX^\top\right]\right)\sum_{s=1}^t c_{s-1}\right)}{(K-1)}\right)+O\left(\exp\left(-t^{1-2\zeta}\right)\right).\nonumber
\end{align}
Using Equation~\eqref{0202temp1}, we can show that
\begin{align*}
  \mathbb{P}\left(\Bigr|(\widehat\bTheta_t -
  \bTheta)^\top(\widehat\bTheta_t -
  \bTheta)\Bigr|\cdot\lambda_{\min}(\bar V_t) \ge U t^{2-4\zeta}\sqrt{dK\log\left(t^2L/\lambda +1\right)+\lambda^{\frac{1}{2}}S}
  \right)&= O\left(\exp\left(-t^{4-8\zeta}\right)\right)
\end{align*}
Since $c_t\left(\lambda + \frac{1}{2}\lambda_{\min}\left(\mathbb{E}\left[\bX\bX^\top\right]\right)\sum_{s=1}^t c_{s-1}\right) \sim t^{1-2\zeta}$ and $2-4\zeta > 1-2\zeta$, we have
\begin{align*}
  \mathbb{P}\left( \frac{L^2||\mu_t - \bTheta||^2+L^2\lambda_{\max}(\boldsymbol\Sigma_t)}{|\inf_{\bx\in\mathcal{X}}\Delta(\bx)|^2} > \frac{c_t}{K-1}\right) = O\left(\exp\left(-t^{4-8\zeta}\right)\right) + O\left(\exp\left(-t^{1-2\zeta}\right)\right).
\end{align*}
Thus, the stability condition holds.

\subsection{Proof of Lemma~\ref{lemma:infinite_sampling}}
\label{lemma:infinite_sampling_proof}

\begin{proof}
  By definition, we have,
  \begin{align*}
  \bar V_t &= \lambda I + \sum_{s=1}^t \Phi(\bX_s,
             A_s)\Phi(\bX_s, A_s)^\top \nonumber\\
  &= \begin{bmatrix}
    \lambda I + \sum_{s=1}^t \bX_s\bX_s^\top I(A_s = 1) &\bf{0} &\cdots &\bf{0} \\
    \bf{0}& \lambda I +\sum_{s=1}^t \bX_s\bX_s^\top I(A_s = 2) & \cdots &\bf{0} \\
    \vdots & \vdots & \ddots & \vdots\\
    \bf{0} & \bf{0} & \cdots & \lambda I +\sum_{s=1}^t \bX_s\bX_s^\top I(A_s = K)
  \end{bmatrix}.
  \end{align*}
For any $a\in\{1,2,\ldots,K\}$, 
\begin{align}
&\mathbb{P}\left(\lambda_{\min}\left(\lambda I + \sum_{s=1}^t \bX_s\bX_s^\top I(A_s = a)\right) \le \lambda + \frac{1}{2}\lambda_{\min}\left( \mathbb{E}\left[\bX\bX^\top\right]\right)\sum_{s=1}^t c_{s-1}\right) \nonumber\\
= \ &\mathbb{P}\left(\lambda_{\min}\left(\sum_{s=1}^t \bX_s\bX_s^\top I(A_s = a)\right) \le \frac{1}{2}\lambda_{\min}\left( \mathbb{E}\left[\bX\bX^\top\right]\right)\sum_{s=1}^t c_{s-1}\right)\nonumber\\
= \ &\mathbb{P}\left(\inf_{\balpha\in\mathbb{R}^d}\balpha^\top\left(\sum_{s=1}^t \bX_s\bX_s^\top I(A_s = a)\right)\balpha \le \frac{1}{2}\sum_{s=1}^t c_{s-1}\inf_{\bbeta\in\mathbb{R}^d}\bbeta^\top \mathbb{E}\left[\bX\bX^\top\right]\bbeta\right) \nonumber\\
= \ &\mathbb{P}\left(\inf_{\balpha\in\mathbb{R}^d}\balpha^\top\left(\sum_{s=1}^t \bX_s\bX_s^\top I(A_s = a)\right)\balpha - \frac{1}{2}\sum_{s=1}^t c_{s-1}\inf_{\bbeta\in\mathbb{R}^d}\bbeta^\top \mathbb{E}\left[\bX\bX^\top\right]\bbeta\le 0\right). \nonumber
\end{align}
Note that for any $\balpha\in\mathbb{R}^d$, we have,
\begin{align*}
  \frac{1}{2}\sum_{s=1}^t c_{s-1}\inf_{\bbeta\in\mathbb{R}^d}\bbeta^\top \mathbb{E}\left[\bX\bX^\top\right]\bbeta \le \frac{1}{2}\sum_{s=1}^t c_{s-1}\balpha^\top \mathbb{E}\left[\bX\bX^\top\right]\balpha.
\end{align*}
Therefore,
\begin{align*}
  &\mathbb{P}\left(\inf_{\balpha\in\mathbb{R}^d}\balpha^\top\left(\sum_{s=1}^t \bX_s\bX_s^\top I(A_s = a)\right)\balpha - \frac{1}{2}\sum_{s=1}^t c_{s-1}\inf_{\bbeta\in\mathbb{R}^d}\bbeta^\top \mathbb{E}\left[\bX\bX^\top\right]\bbeta\le 0\right) \nonumber\\
  \le&\mathbb{P}\left(\inf_{\balpha\in\mathbb{R}^d}\balpha^\top\left(\sum_{s=1}^t \bX_s\bX_s^\top I(A_s = a) - \frac{1}{2} \mathbb{E}\left[\bX\bX^\top\right]\sum_{s=1}^t c_{s-1}\right)\balpha\le 0\right).
\end{align*}
Furthermore,
\begin{align}
  &\sum_{s=1}^t \bX_s\bX_s^\top I(A_s = a) - \frac{1}{2} \mathbb{E}\left[\bX\bX^\top\right]\sum_{s=1}^t c_{s-1} \nonumber\\
  =&\sum_{s=1}^t \bX_s\bX_s^\top\left(I(A_s = a)-\hat\pi_{s-1}(\bX_s,a)\right) + \sum_{s=1}^t \bX_s\bX_s^\top\left(\hat\pi_{s-1}(\bX_s,a) - c_{s-1}\right)\nonumber\\
  &+ \sum_{s=1}^t \left(\bX_s \bX_s^\top - \mathbb{E}\left[\bX\bX^\top\right]\right)c_{s-1} + \frac{1}{2} \mathbb{E}\left[\bX\bX^\top\right]\sum_{s=1}^t c_{s-1}\nonumber.
\end{align}
Since $\hat\pi_{s-1}(\bX_s,a)\ge c_{s-1}$ for all $\bX_{s},s,a$, and $\bX_s\bX_s^\top$ is a semi-positive definite matrix, $\forall \balpha\in\mathbb{R}^d$,
$$
\balpha^\top\left(\sum_{s=1}^t \bX_s\bX_s^\top\left(\hat\pi_{s-1}(\bX_s,a) - c_{s-1}\right)\right)\balpha \ge 0.
$$
Therefore,
\begin{align}
&\mathbb{P}\left(\inf_{\balpha\in\mathbb{R}^d}\balpha^\top\left(\sum_{s=1}^t \bX_s\bX_s^\top I(A_s = a) - \frac{1}{2} \mathbb{E}\left[\bX\bX^\top\right]\sum_{s=1}^t c_{s-1}\right)\balpha\le 0\right)\nonumber\\
\le &\mathbb{P}\left(\inf_{\balpha\in\mathbb{R}^d}\balpha^\top\left(\sum_{s=1}^t \bX_s\bX_s^\top\left(I(A_s = a)-\hat\pi_{s-1}(\bX_s,a)\right) + \sum_{s=1}^t \left(\bX_s \bX_s^\top - \mathbb{E}\left[\bX\bX^\top\right]\right)c_{s-1}\right.\right.\nonumber\\
 &\quad \left.\left.+\frac{1}{2} \mathbb{E}\left[\bX\bX^\top\right]\sum_{s=1}^t c_{s-1}\right)\balpha \le 0\right)\nonumber\\
\le & \mathbb{P}\left(\inf_{\balpha\in\mathbb{R}^d}\balpha^\top\left(\sum_{s=1}^t \bX_s\bX_s^\top\left(I(A_s = a)-\hat\pi_{s-1}(\bX_s,a)\right)+ \frac{1}{4}\mathbb{E}\left[\bX\bX^\top\right]\sum_{s=1}^t c_{s-1}\right)\balpha\le 0\right)\nonumber\\
&+ \mathbb{P}\left(\inf_{\balpha\in\mathbb{R}^d}\balpha^\top\left(\sum_{s=1}^t \left(\bX_s \bX_s^\top - \mathbb{E}\left[\bX\bX^\top\right]\right)c_{s-1} + \frac{1}{4}\mathbb{E}\left[\bX\bX^\top\right]\sum_{s=1}^t c_{s-1}\right)\balpha\le 0\right)\nonumber\\
\le&\mathbb{P}\left(\lambda_{\min}\left(\sum_{s=1}^t \bX_s\bX_s^\top\left(I(A_s = a)-\hat\pi_{s-1}(\bX_s,a)\right)\right) \le -\frac{\sum_{s=1}^t c_{s-1}}{4}\lambda_{\min}\left(\mathbb{E}\left[\bX\bX^\top\right]\right)\right)\label{01301}\\
&+\mathbb{P}\left(\lambda_{\min}\left(\sum_{s=1}^t \left(\bX_s \bX_s^\top - \mathbb{E}\left[\bX\bX^\top\right]\right)c_{s-1}\right) \le -\frac{\sum_{s=1}^t c_{s-1}}{4}\lambda_{\min}\left(\mathbb{E}\left[\bX\bX^\top\right]\right)\right)\label{01302}
\end{align}

We now bound these two terms.  For the term in Equation~\eqref{01301},
notice $\sum_{s=1}^t \bX_s\bX_s^\top(I(A_s=a)-\hat\pi_{t-1}(\bX_s,a))$
is a matrix martingale with respect to the filtration
$\mathcal{F}_t =
\sigma(\bX_1,A_1,R_1,\ldots,\bX_t,A_t,R_t)$. Furthermore, for all $t$,
$\norm{(\bX_s\bX_s^\top(I(A_s=a)-\hat\pi_{t-1}(\bX_s,a)))^2}$ is
bounded by $2L^4$. Therefore, by the matrix Azuma-Hoeffding inequality,
we have
\begin{align}
  &\mathbb{P}\left(\lambda_{\min}\left(\sum_{s=1}^t \bX_s\bX_s^\top\left(I(A_s = a)-\hat\pi_{s-1}(\bX_s,a)\right)\right) \le -\frac{\sum_{s=1}^t c_{s-1}}{4}\lambda_{\min}\left(\mathbb{E}\left[\bX\bX^\top\right]\right)\right)\nonumber\\
  =&\mathbb{P}\left(\lambda_{\max}\left(\sum_{s=1}^t \bX_s\bX_s^\top\left(-I(A_s = a)+\hat\pi_{s-1}(\bX_s,a)\right)\right) \ge \frac{\sum_{s=1}^t c_{s-1}}{4}\lambda_{\min}\left(\mathbb{E}\left[\bX\bX^\top\right]\right)\right)\nonumber\\
  \le& d \cdot
       \exp\left(-\frac{\frac{1}{16}\left(\sum_{s=1}^{t}c_{s-1}\right)^2\lambda_{\min}\left(\mathbb{E}\left[\bX\bX^\top\right]\right)^2}{8\times
       2L^4t}\right) \nonumber \\
  = & d \cdot \exp\left(-\frac{\left(\sum_{s=1}^{t}c_{s-1}\right)^2\lambda_{\min}\left(\mathbb{E}\left[\bX\bX^\top\right]\right)^2}{256L^4t}\right).\nonumber
\end{align}

For the term in Equation~\eqref{01302}, notice
$\sum_{s=1}^t (\bX_s\bX_s^\top -
\mathbb{E}\left[\bX\bX^\top\right])c_{s-1}$ is the sum of $t$
independent random variables with mean $\bf{0}$, and
$\norm{(\bX_s\bX_s^\top -
  \mathbb{E}\left[\bX\bX^\top\right])^2c_{s-1}^2}$ is bounded by
$4c_{s-1}^2L^4$. Therefore, we can apply the Azuma-Hoeffding
inequality as done above,
\begin{align}
  &\mathbb{P}\left(\lambda_{\min}\left(\sum_{s=1}^t \left(\bX_s \bX_s^\top - \mathbb{E}\left[\bX\bX^\top\right]\right)c_{s-1}\right) \le -\frac{\sum_{s=1}^t c_{s-1}}{4}\lambda_{\min}\left(\mathbb{E}\left[\bX\bX^\top\right]\right)\right)\nonumber\\
  =&\mathbb{P}\left(\lambda_{\max}\left(\sum_{s=1}^t \left(-\bX_s \bX_s^\top + \mathbb{E}\left[\bX\bX^\top\right]\right)c_{s-1}\right) \ge \frac{\sum_{s=1}^t c_{s-1}}{4}\lambda_{\min}\left(\mathbb{E}\left[\bX\bX^\top\right]\right)\right)\nonumber\\
  \le& d \cdot
       \exp\left(-\frac{\frac{1}{16}\left(\sum_{s=1}^{t}c_{s-1}\right)^2\lambda_{\min}\left(\mathbb{E}\left[\bX\bX^\top\right]\right)^2}{8\sum_{s=1}^t
       4c_{s-1}^2L^4}\right) \nonumber \\
  = & d \cdot \exp\left(-\frac{\left(\sum_{s=1}^{t}c_{s-1}\right)^2\lambda_{\min}\left(\mathbb{E}\left[\bX\bX^\top\right]\right)^2}{512L^4\sum_{s=1}^t c_{s-1}^2}\right)\nonumber
\end{align}

Therefore, we have
\begin{align}
&\mathbb{P}\left(\lambda_{\min}\left(\bar V_t\right) \le \lambda + \frac{1}{2}\lambda_{\min}\left(\mathbb{E}\left[\bX\bX^\top\right]\right)\sum_{s=1}^t c_{s-1}\right) \nonumber\\
\le & d\cdot \exp\left(-\frac{\left(\sum_{s=1}^{t}c_{s-1}\right)^2\lambda_{\min}\left(\mathbb{E}\left[\bX\bX^\top\right]\right)^2}{256L^4t}\right) + d\cdot \exp\left(-\frac{\left(\sum_{s=1}^{t}c_{s-1}\right)^2\lambda_{\min}\left(\mathbb{E}\left[\bX\bX^\top\right]\right)^2}{512L^4\sum_{s=1}^t c_{s-1}^2}\right) \nonumber.
\end{align}
By Assumption~\ref{ass:clip_bandit}, $c_t\le Const\cdot t^{-\zeta}$
where $0\le \zeta<\frac{1}{2}$. Therefore, we have the desired result,
$$
\mathbb{P}\left(\lambda_{\min}\left(\bar V_t\right) \le \lambda + \frac{1}{2}\lambda_{\min}\left(\mathbb{E}\left[\bX\bX^\top\right]\right)\sum_{s=1}^t c_{s-1} \right) = O\left(\exp\left(-t^{1-2\zeta}\right)\right).
$$
\end{proof}

\section{Additional results with synthetic data experiments}
\label{app:additional_simulation}

In this Appendix, we present additional results from our synthetic
data experiments.
We compare the performance of the Cram with the Sample Splitting method with \cite{zhan2021off}'s off-policy evaluation. For a clearer comparison on the effect of different simulation parameters, we use IPW estimator for both off-policy evaluation and Cram to remove the effect of simulation parameters on the augmentation.
We note that for sample splitting
evaluation, the performance of the different adaptive weights of
\cite{zhan2021off} is close to each other. This is because, in these
experimental settings, we use the last 20\% of bandit at the end of a
bandit sequence for evaluation, where the learned policy is relatively stable and different weights in \cite{zhan2021off} are similar. Therefore, we only present the result with their AWAIPW-NS weight. 

\subsection{Results for different length of Bandit rounds}

\begin{figure}[t!]
  \centering
  \begin{subfigure}[b]{0.49\textwidth}
      \centering
      \includegraphics[width=\textwidth]{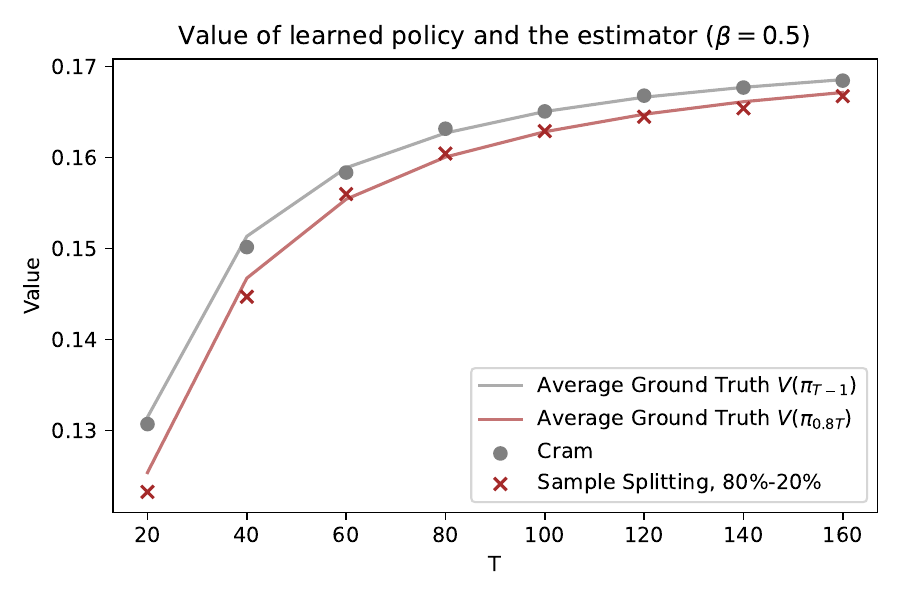}
      \caption{Value, $\beta=0.5$}
      \label{fig:TS-Bias-T}
  \end{subfigure}
  \begin{subfigure}[b]{0.49\textwidth}
      \centering
      \includegraphics[width=\textwidth]{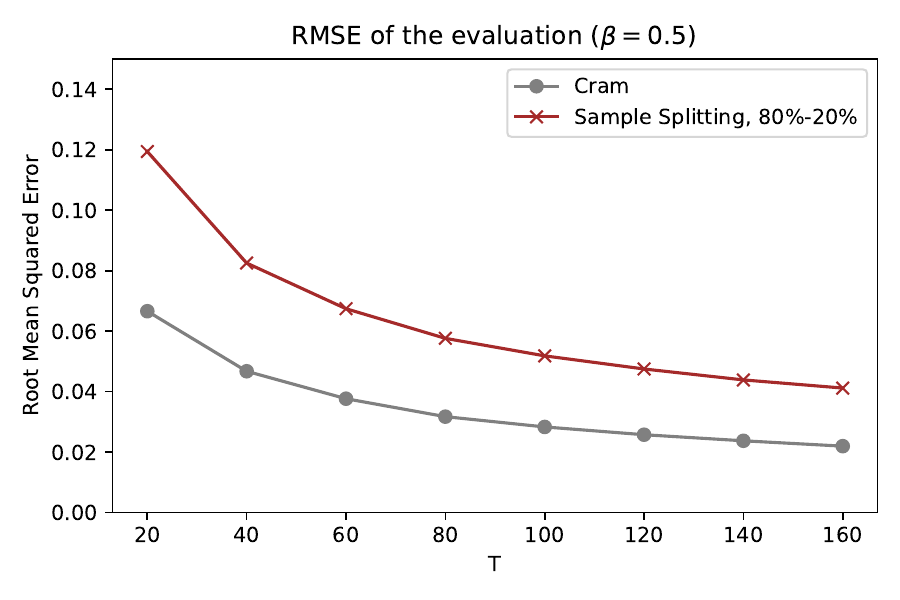}
      \caption{Root MSE, $\beta=0.5$}
      \label{fig:TS-MSE-T}
  \end{subfigure}
  \begin{subfigure}[b]{0.49\textwidth}
      \centering
      \includegraphics[width=\textwidth]{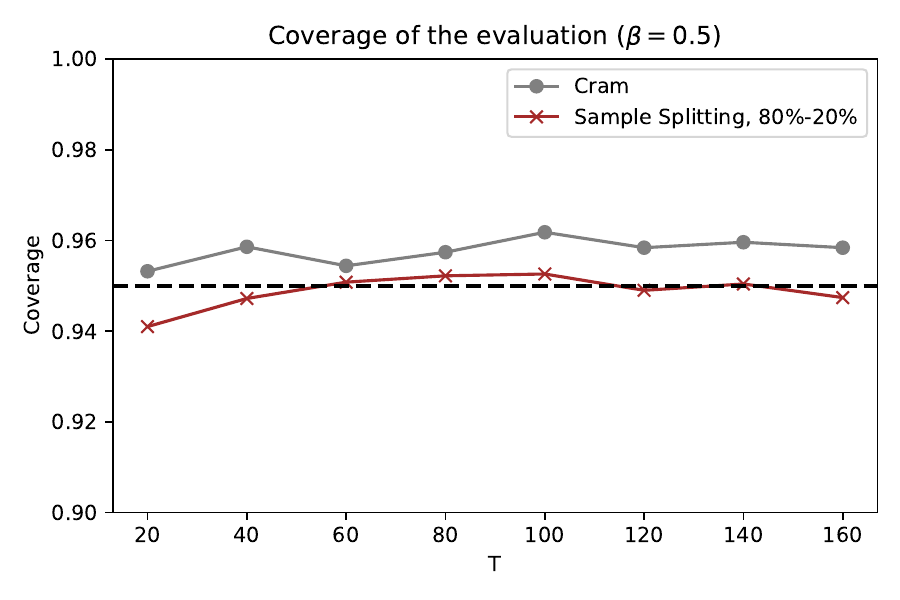}
      \caption{Coverage of 95\% CI, $\beta=0.5$}
      \label{fig:TS-Coverage-T}
  \end{subfigure}
  \begin{subfigure}[b]{0.49\textwidth}
    \centering
    \includegraphics[width=\textwidth]{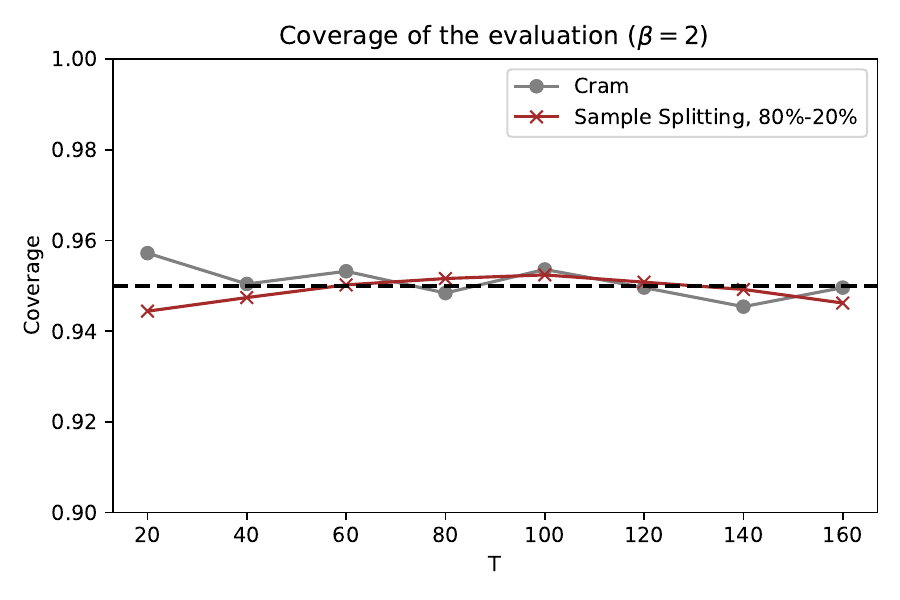}
    \caption{Coverage of 95\% CI, $\beta=2$}
    \label{fig:TS-Coverage-T2}
\end{subfigure}
\caption{Evaluation performance as a function of bandit sequence
  length $T$. The data are collected with Thompson Sampling
  Algorithm. We compare cram evaluation v.s. 80--20\% Sample
  Splitting.}
  \label{fig:TS-T}
\end{figure}

In Figure~\ref{fig:TS-T}, we show the performance of the cram
evaluation and the 80--20\% sample splitting evaluation as a function
of bandit length $T$.  Since cram evaluates the policy $\hat\pi_{T-1}$
rather than the policy trained on the training set, i.e.,
$\hat\pi_{0.8T}$, as shown in Figure~\ref{fig:TS-Bias-T}, the ground
truth of the crammed policy value is always greater than the ground
truth of the trained on sample splitting.  For both cram and sample
splitting methods, the evaluation is unbiased and the evaluation RMSE
decrease as $T$ increases. For different $T$, the cram evaluation
consistently have 40\% improvement in RMSE compared to the 80--20\%
sample splitting evaluation. 

For empirical coverage of the 95\% confidence interval (CI), both cram
and sample splitting methods have close to nominal level coverage. The
cram's coverage is slightly above the nominal level, which is mainly
due to the over-estimation of the cram variance. The consistency
results in Theorem~\ref{thm:variance_estimator} relies on the policy
stability. In finite sample, if the signal is low and the policy is
not stabilizing fast enough, the estimator~${\hat v_{Tj}^2}$ in
Definition~\ref{def:crammed_variance} may slightly overestimate the
variance as we are approximately using all data after the $j$th
iteration as if they are collected using policy $\pi_{j-1}$.  As
expected, in Figure~\ref{fig:TS-Coverage-T2}, when the signal is
stronger, the empirical coverage of the cram CI becomes close to the
nominal level.

\subsection{Results for different signal strengths}
 
Next, we fix the $T$ to 100 and summarize the performance of the cram
and sample splitting evaluation under different signal strengths
$\beta$. Figure~\ref{fig:TS-Beta} shows the RMSE (left axis) and the
coverage of the 95\% confidence interval (right axis) for Thompson
Sampling, UCB, and $\epsilon$-greedy bandit algorithms. For
readability, we only show the sample splitting evaluation results with
NS weights as different weights yield similar results in our
simulation. 

As shown in Figure~\ref{fig:TS-Beta}, for different bandit algorithms
and different signal strengths, the cram method consistently
outperforms sample splitting in terms of evaluation RMSE. The cram
method performs especially well for a greater signal strength, for
which the bandit algorithm stabilizes faster. In all setups, the 95\%
confidence interval based on cramming has a valid empirical coverage
for, though we find slight over-coverage due to the over-estimation of
variance.  This over-coverage is partially alleviated when the signal
strength is greater.
 
\begin{figure}[t]
  \centering
  \begin{subfigure}[b]{0.49\textwidth}
      \centering
      \includegraphics[width=\textwidth]{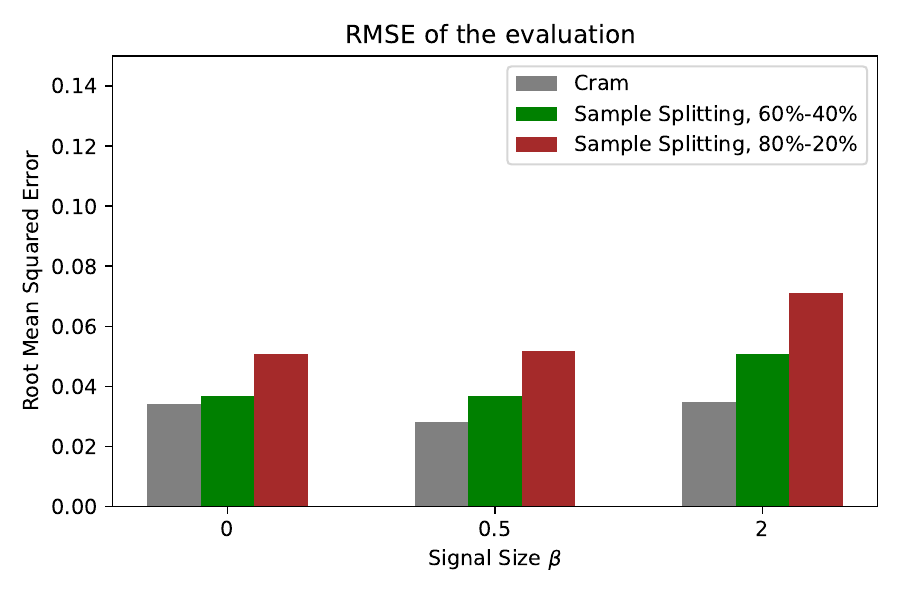}
      \caption{Thompson Sampling RMSE, $T=100$}
      \label{fig:TS-Beta-1}
  \end{subfigure}
  \begin{subfigure}[b]{0.49\textwidth}
      \centering
      \includegraphics[width=\textwidth]{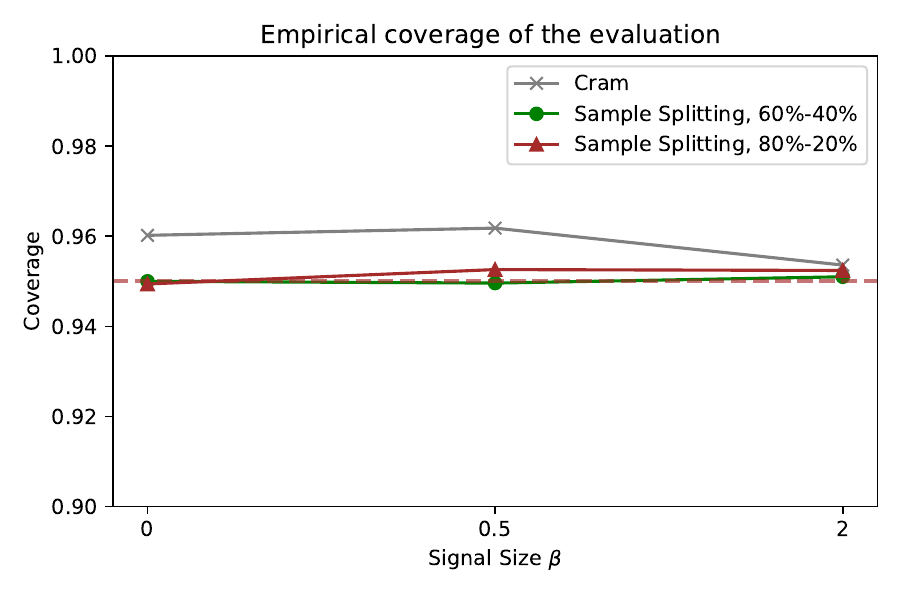}
      \caption{Thompson Sampling coverage, $T=100$}
      \label{fig:TS-Beta-2}
  \end{subfigure}
  \begin{subfigure}[b]{0.49\textwidth}
      \centering
      \includegraphics[width=\textwidth]{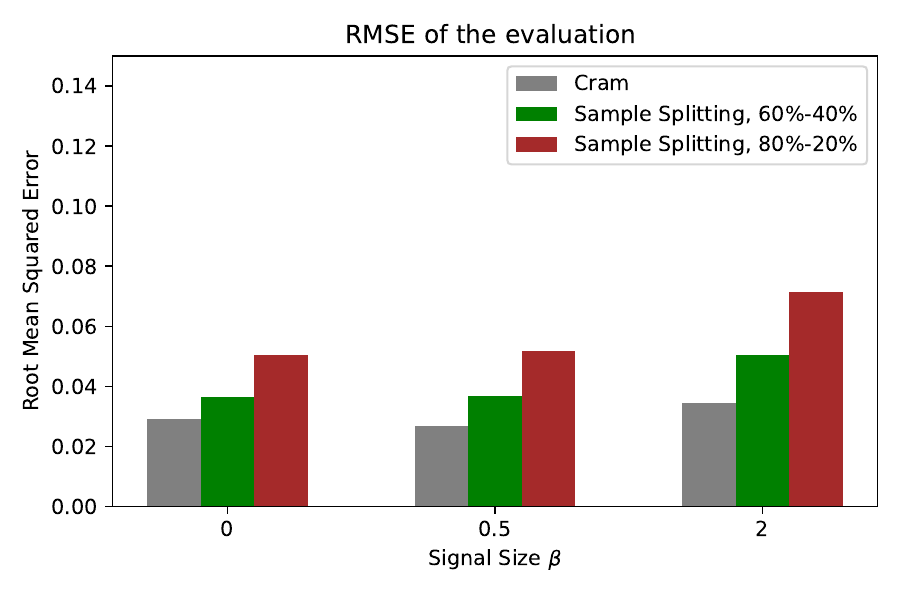}
      \caption{LinUCB RMSE, $T=100$}
      \label{fig:UCB-Beta-1}
  \end{subfigure}
  \begin{subfigure}[b]{0.49\textwidth}
      \centering
      \includegraphics[width=\textwidth]{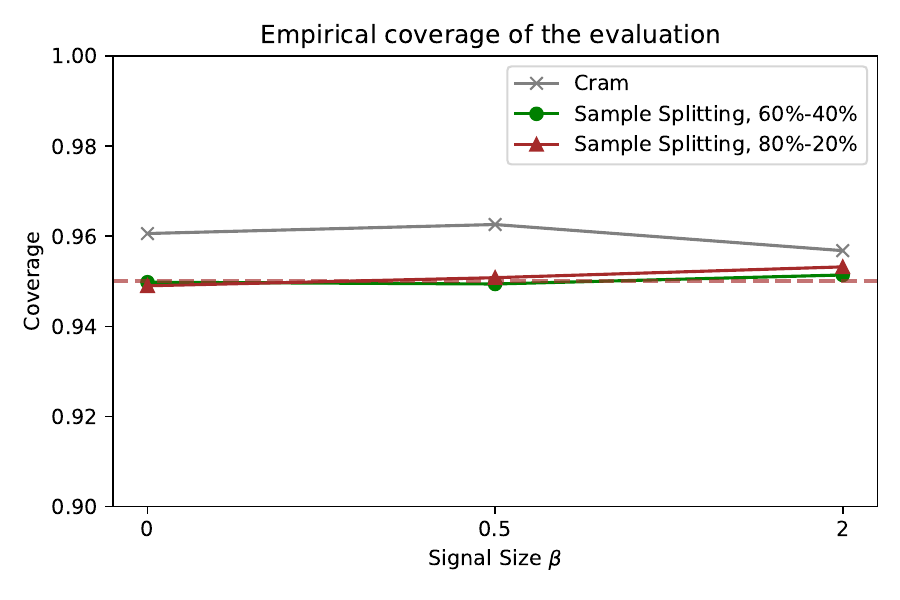}
      \caption{LinUCB coverage, $T=100$}
      \label{fig:UCB-Beta-2}
  \end{subfigure}
  \begin{subfigure}[b]{0.49\textwidth}
      \centering
      \includegraphics[width=\textwidth]{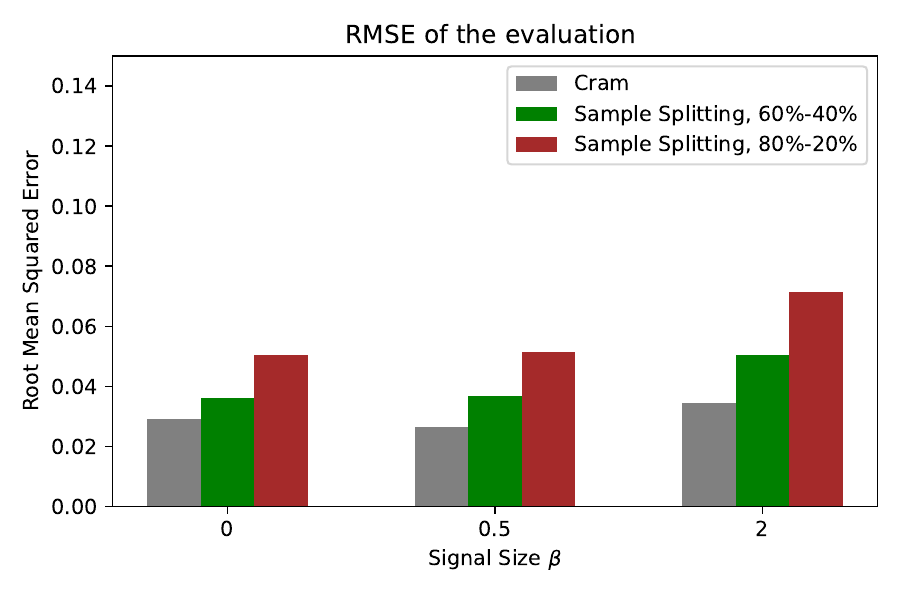}
      \caption{$\epsilon$-greedy RMSE, $T=100$}
      \label{fig:Epsilon-Beta-1}
  \end{subfigure}
  \begin{subfigure}[b]{0.49\textwidth}
      \centering
      \includegraphics[width=\textwidth]{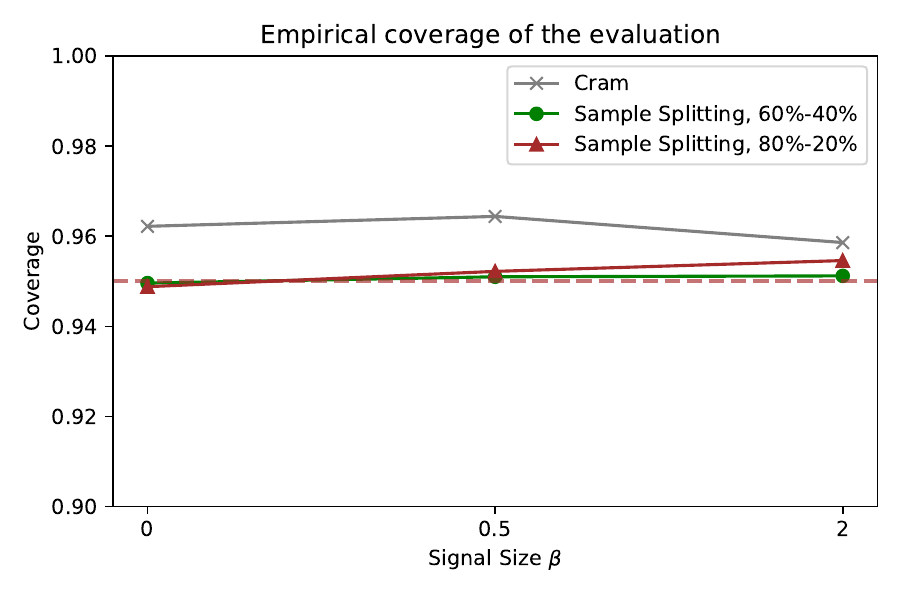}
      \caption{$\epsilon$-greedy coverage, $T=100$}
      \label{fig:Epsilon-Beta-2}
  \end{subfigure}
  \caption{Performance of the evaluation for different bandit
  algorithm and  signal size. The bar plot (left y-axis) shows the
  RMSE, and the line plot (right y-axis) shows the coverage of the
  95\% confidence interval. The decay rate $\eta$ is set to 0.}
\label{fig:TS-Beta}
\end{figure}

\subsection{Results for different clipping decay rates}

In Figure~\ref{fig:Decay}, we examine the performance of the cram
evaluation under different clipping decay rates with the three bandit
algorithms. We find that the cram method still outperforms 80--20\%
sampling splitting, though the 95\% CI empirical coverage is slightly
more conservative than sample splitting. The cram method works better
when the clipping rate decays slower, reducing the instances of
extreme propensity scores.
 
\begin{figure}[t]
  \centering
  \begin{subfigure}[b]{0.49\textwidth}
      \centering
      \includegraphics[width=\textwidth]{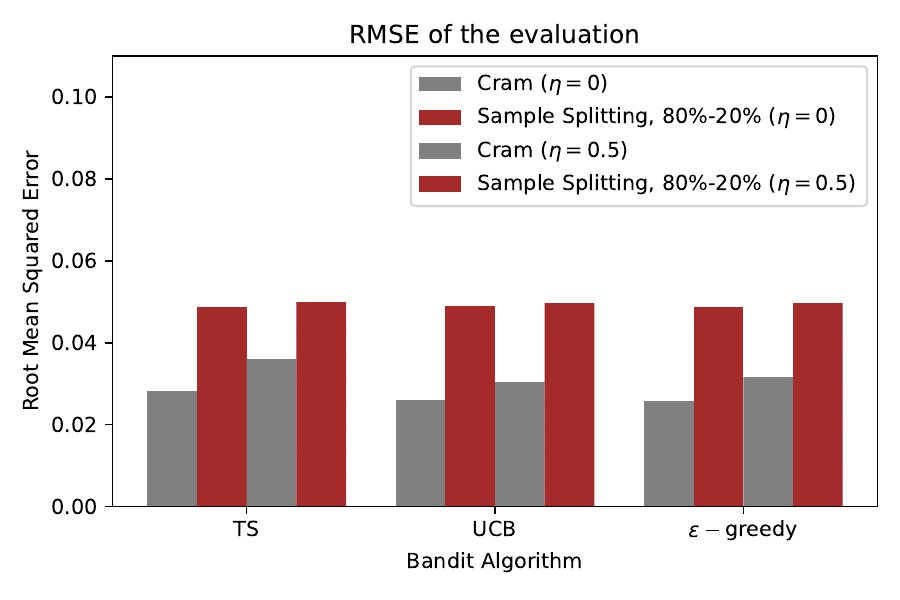}
      \caption{Root MSE, $T=100,\beta=0.5$}
      \label{fig:Decay1}
  \end{subfigure}
  \begin{subfigure}[b]{0.49\textwidth}
      \centering
      \includegraphics[width=\textwidth]{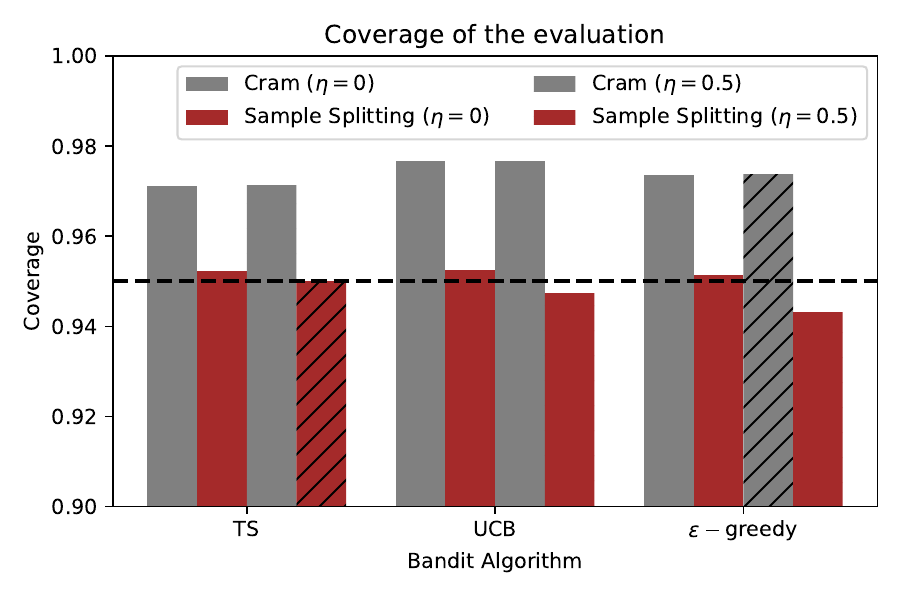}
      \caption{Coverage, $T=100, \beta=0.5$}
      \label{fig:Decay2}
  \end{subfigure}
  \caption{Performance of the evaluation for different bandit
    algorithms and clipping decaying rates. The left plot shows the
    RMSE, and the right plot shows the coverage of the 95\% confidence
    interval.} 
  \label{fig:Decay}
\end{figure}

\end{document}